\documentclass{article}
\usepackage{fedhv_arxiv_preprint,times}

\usepackage{microtype}
\microtypesetup{expansion=false}
\usepackage{graphicx}
\usepackage{booktabs}
\usepackage{bm}
\usepackage{tabularx}
\usepackage{float}
\usepackage{xcolor}
\usepackage{array}
\usepackage{hyperref}
\hypersetup{
  hidelinks,
  pdftitle={FedHV: Low-Overhead Hypervolume Weighting for Federated Multi-Objective Optimization},
  pdfauthor={Amirardalan Dehghanpour; Seyed Mohammad Azimi-Abarghouyi; Christopher G. Brinton}
}
\usepackage{url}
\usepackage{subcaption}
\usepackage[section]{placeins}
\usepackage{dblfloatfix}
\usepackage{float}
\usepackage{tikz}
\usetikzlibrary{arrows.meta,positioning}
\definecolor{fedhvblue}{HTML}{EAF2F8}
\definecolor{fedhvgreen}{HTML}{E9F7F2}
\definecolor{fedhvline}{HTML}{315D73}
\newcommand{\thetav}{\bm{\theta}}

\floatstyle{ruled}
\newfloat{algorithm}{tbp}{loa}
\floatname{algorithm}{Algorithm}
\newcounter{algline}
\newenvironment{algorithmic}[1][]{%
  \setcounter{algline}{0}%
  \begin{list}{\scriptsize\arabic{algline}:}{\usecounter{algline}
  \setlength{\leftmargin}{2.5em}\setlength{\labelsep}{0.6em}
  \setlength{\itemsep}{1.0pt}\setlength{\parsep}{0pt}
  \setlength{\topsep}{2pt}}}{\end{list}}
\newcommand{\STATE}{\item}
\newcommand{\FOR}[1]{\item \textbf{for} #1 \textbf{do}}
\newcommand{\FORALL}[1]{\item \textbf{for all} #1}
\newcommand{\ENDFOR}{\item \textbf{end for}}

\usepackage{amsmath}
\usepackage{amssymb}
\usepackage{mathtools}
\usepackage{amsthm}

\usepackage[capitalize,noabbrev]{cleveref}

\def\wv{\bm{w}}

\newcommand{\sg}[1]{\operatorname{sg}\!\left(#1\right)}

\theoremstyle{plain}
\newtheorem{theorem}{Theorem}[section]
\newtheorem{proposition}[theorem]{Proposition}
\newtheorem{lemma}[theorem]{Lemma}
\newtheorem{corollary}[theorem]{Corollary}
\theoremstyle{definition}
\newtheorem{definition}[theorem]{Definition}
\newtheorem{assumption}[theorem]{Assumption}
\theoremstyle{remark}
\newtheorem{remark}[theorem]{Remark}

\title{FedHV: Low-Overhead Hypervolume Weighting for Federated Multi-Objective Optimization}

\author{
\textbf{Amirardalan Dehghanpour}%
\thanks{Corresponding author:
\texttt{ardalandehghan0@gmail.com}}\\
\textnormal{Independent Researcher}\\[1.2ex]
\textbf{Seyed Mohammad Azimi-Abarghouyi}\\
\textnormal{Chalmers University of Technology}\\
\textnormal{\texttt{azimimo@chalmers.se}}\\[1.2ex]
\textbf{Christopher G. Brinton}\\
\textnormal{Purdue University}\\
\textnormal{\texttt{cgb@purdue.edu}}
}

\begin{document}
\maketitle
\fancyhead[L]{\fedhvpreprintstatus}

\begin{abstract}
Task-wise federated multi-objective optimization (FedMOO) trains a shared model for competing prediction objectives under heterogeneous data, partial participation, and communication constraints. Existing methods commonly derive task weights from gradient or update geometry. This requires task-specific information or iterative server-side optimization. We introduce \textbf{FedHV}, which maps reference-relative objective slacks to closed-form inverse-slack weights. Each client optimizes one weighted loss and returns objective estimates with its model update. The protocol adds exactly $2m$ auxiliary scalars per participating client, yielding $\Theta(d+m)$ total per-client communication, compared with the $\Theta(md)$ task-specific communication of FSMGDA, and requires no additional synchronization stage. We analyze the resulting one-round-delayed weights under client heterogeneity, multi-step local updates, partial participation, and finite-sample objective reports. Under a fixed-horizon positive-slack reference condition, with the prescribed horizon-dependent step size and vanishing report error, FedHV achieves an $\mathcal{O}(T^{-1/2})$ rate for the average squared log-hypervolume gradient norm; persistent report error determines the resulting stationarity neighborhood. The same bound controls the squared Pareto-stationarity residual. Across six Dirichlet-partitioned non-IID settings from four vision benchmark families and three training seeds, FedHV exceeds FSMGDA and FedCMOO in mean accuracy in five settings. Among these methods and uniform scalarization, it achieves the highest worst-task accuracy in four settings and improves the difficult CIFAR-10 objective in both \textsc{CIFAR10--MNIST} settings.
\end{abstract}

\section{Introduction}
\label{sec:introduction}

Federated learning (FL) trains models from decentralized data without sharing clients' raw examples. In many applications, however, one shared model must serve several prediction tasks whose objectives may conflict. Optimizing their average can favor easier tasks and leave a difficult objective with poor performance. \emph{Task-wise federated multi-objective optimization} (FedMOO) instead seeks a shared model that balances multiple global task objectives \citep{yang2023fmol,askin2025fedcmoo}. This setting differs from client-wise robust or fair FL, where the objectives represent clients or demographic groups rather than prediction tasks \citep{mohri2019agnostic,li2020fair,hu2022federated}.

Adapting task trade-offs is difficult in FL because the server observes information from only a subset of heterogeneous clients, while multiple local updates introduce client drift. FSMGDA constructs a federated multi-gradient direction from objective-specific client updates, leading to $\Theta(md)$ per-client communication for $m$ tasks and a $d$-dimensional model \citep{yang2023fmol}. FedCMOO reduces this communication through randomized low-rank estimation, but it still estimates task-update geometry and solves an iterative projected weight problem at the server \citep{askin2025fedcmoo}. Objective values provide a lower-dimensional alternative: clients can report $m$ scalar losses, allowing the server to adapt task weights without constructing or communicating task-specific $d$-dimensional updates.

We propose \textbf{FedHV} (\emph{Federated Hypervolume}), which forms task weights directly from these objective reports. Given global task losses $f_i(\theta)$ and fixed reference losses $r_i$, FedHV uses
\[
w_i(\theta)\propto\bigl(r_i-f_i(\theta)\bigr)^{-1},
\]
so that tasks closer to their references receive greater weight. This inverse-slack rule follows from centralized single-solution hypervolume maximization \citep{miranda2016single}. FedHV obtains the reference vector from a short uniform-weight calibration phase, adds a positive margin, and freezes it before adaptive training.

At each round, selected clients evaluate their task losses at the common pre-update model and optimize one scalar loss using weights computed from the preceding round's reports. The new reports accompany the model updates and determine the next-round weights. FedHV therefore adds exactly $2m$ auxiliary scalars per participating client without an intermediate synchronization stage. The one-round delay is also central to the analysis: the implemented weights depend on noisy reports at the preceding global model rather than the current objectives. Our proof controls this discrepancy through the report error and the inter-round objective change, together with stochastic-gradient noise, partial-participation heterogeneity, and multi-step local drift.

\paragraph{Contributions.}
\begin{itemize}
    \item \textbf{Protocol.} We introduce a task-wise FedMOO method that computes closed-form inverse-slack weights from low-dimensional objective reports. FedHV uses one weighted local trajectory, adds $2m$ auxiliary scalars per participating client, and requires neither task-update geometry nor an additional synchronization stage.
    
    \item \textbf{Analysis.} We establish convergence for the one-round-delayed protocol under partial participation, client heterogeneity, stochastic multi-step local updates, and finite-sample objective reports. With vanishing report error and the prescribed step-size schedule, the average squared log-hypervolume gradient norm decreases as $\mathcal{O}(T^{-1/2})$. We also relate this quantity to the Pareto-stationarity residual and characterize reference-relative worst-task slack control.
    
    \item \textbf{Evaluation.} Across six Dirichlet-partitioned non-IID settings, FedHV exceeds FSMGDA and FedCMOO in mean accuracy in five settings and achieves the highest worst-task accuracy among all four evaluated methods in four. Ablations examine the effects of objective-report budget, client participation, and local computation.
\end{itemize}

\section{Related Work}
\label{sec:related-work}

\paragraph{Pareto-oriented multi-objective optimization.}
Multi-objective optimization seeks solutions for which no objective can be improved without degrading another. MGDA computes a minimum-norm convex combination of task gradients and uses the resulting direction to approach Pareto-stationary points \citep{desideri2012mgda,sener2018multi}. Pareto MTL extends this perspective by targeting several regions of the Pareto front rather than producing only one compromise solution \citep{lin2019pareto}. These methods explicitly define their updates through task-gradient geometry and Pareto-stationarity conditions.

\paragraph{Gradient balancing in multi-task learning.}
A related class of multi-task learning methods modifies task gradients to improve shared-model training without necessarily constructing an MGDA direction. GradNorm balances gradient magnitudes and relative training rates \citep{chen2018gradnorm}; PCGrad removes components of mutually conflicting gradients \citep{yu2020pcgrad}; and CAGrad balances average descent against gradient conflict \citep{liu2021cagrad}. Although these methods address negative transfer among tasks, their update rules differ from methods explicitly designed around the Pareto-stationarity residual. They also require separate task-gradient information and were developed primarily for centralized optimization \citep{hu2023scalarization}.

\paragraph{Federated multi-objective optimization.}
FMOL introduces FMGDA and FSMGDA for nonconvex federated multi-objective problems \citep{yang2023fmol}. FSMGDA supports multiple local steps but constructs its federated direction from objective-specific client updates, resulting in communication that scales with both the number of tasks and the model dimension. FedCMOO reduces this cost using randomized low-rank estimation of task-update geometry and then optimizes the task weights through projected server-side iterations \citep{askin2025fedcmoo}. FedHV does not estimate update geometry. It forms weights from an $m$-dimensional objective report and uses a single weighted local trajectory, replacing task-specific model information with objective values while retaining the usual model-update exchange.

\paragraph{Hypervolume-based optimization.}
Hypervolume is a reference-dependent and Pareto-compliant measure of joint objective quality \citep{zitzler2007hypervolume,auger2009theory,guerreiro2021hypervolume}. It has been optimized for a single neural-network solution and for learned Pareto sets or parameterized solution families \citep{miranda2016single,deist2021hypervolume,zhang2023hypervolume}. For a single solution, \citet{miranda2016single} derive the log-hypervolume objective and its inverse-slack weighting rule. FedHV uses this centralized identity as its scalarization rule; its contribution lies in the federated protocol and analysis. In particular, the weights are estimated from stochastic finite-sample client reports, held fixed during multiple local steps, and applied one round after the corresponding objectives are observed. Our analysis accounts jointly for this delay, partial participation, client heterogeneity, and local-update drift.

\section{FedHV}
\label{sec:method}

\subsection{Task-wise federated multi-objective optimization}
\label{subsec:fedmoo-setup}

We consider $N$ equal-size clients and $m$ prediction objectives.
Client $k$ has local risk $f_i^{(k)}(\theta)$ for task $i$, and the
corresponding global risk is
$f_i(\theta)=N^{-1}\sum_{k=1}^{N}f_i^{(k)}(\theta)$. Task-wise FedMOO
seeks a shared model $\theta\in\mathbb{R}^d$ that balances
\begin{equation}
\bm F(\theta)
=
\bigl(f_1(\theta),\ldots,f_m(\theta)\bigr)^\top .
\label{eq:moo}
\end{equation}
Let
$\Delta_m:=\{\lambda\in\mathbb{R}_{\geq0}^m:
\sum_i\lambda_i=1\}$. Under the standard definition
\citep{desideri2012mgda,sener2018multi}, $\theta^\star$ is
\emph{Pareto-stationary} if some $\lambda\in\Delta_m$ satisfies
\begin{equation}
\sum_{i=1}^{m}\lambda_i\nabla f_i(\theta^\star)
=
\mathbf{0}.
\label{eq:pareto-stationarity}
\end{equation}

\subsection{Log-hypervolume weighting}
\label{subsec:hv-weights}

Let $\bm r=(r_1,\ldots,r_m)^\top$ be a fixed vector of reference losses
and define the reference-relative slack
$s_i(\theta):=r_i-f_i(\theta)$. On the domain where all slacks are
positive, FedHV minimizes the negative log-hypervolume
\begin{equation}
\Phi(\theta)
=
-\sum_{i=1}^{m}\log\bigl(r_i-f_i(\theta)\bigr).
\label{eq:loghv}
\end{equation}
The ideal weights computed from the true global objectives are
\begin{equation}
w_i^\star(\theta)
=
\frac{\bigl(r_i-f_i(\theta)\bigr)^{-1}}
{\sum_{j=1}^{m}\bigl(r_j-f_j(\theta)\bigr)^{-1}},
\qquad i=1,\ldots,m.
\label{eq:inverse-slack-weights}
\end{equation}
Writing
$c(\theta):=\sum_j(r_j-f_j(\theta))^{-1}$ gives
\begin{equation}
\nabla\Phi(\theta)
=
c(\theta)
\sum_{i=1}^{m}
w_i^\star(\theta)\nabla f_i(\theta).
\label{eq:weighted-gradient-identity}
\end{equation}
Tasks closer to their references therefore receive greater weight, and
normalization changes only the positive scale of the resulting
gradient. The inverse-slack identity is due to
\citet{miranda2016single}; FedHV uses it within a federated protocol.

The reference is constructed before adaptive training using
uniform-weight calibration. Each two-task benchmark family uses one
$20$-round calibration run. For \textsc{CIFAR10--MNIST}, the reference
obtained under the $\alpha=0.30$ partition is reused for both reported
settings. The \textsc{CelebA5} reference is constructed from three
independent $20$-round runs, corresponding to $60$ sequential-round
equivalents, and is shared across both \textsc{CelebA5} settings and all
adaptive-training seeds. Adaptive models are reinitialized rather than
continued from the calibration runs. Calibration is therefore a
one-time preprocessing cost and does not add synchronization to the
adaptive rounds. Appendix~\ref{app:reference-selection} gives the
construction, margins, and resulting references.

\subsection{Federated protocol}
\label{subsec:fedhv-protocol}

At round $t$, the server holds the global model $\theta^t$ and weight
vector $\wv^t$. The initial weights are
$\wv^0=\frac{1}{m}\mathbf{1}_m$; subsequent weights are computed from
the preceding aggregate objective report. The server samples
$S_t\subseteq\{1,\ldots,N\}$ and broadcasts $(\theta^t,\wv^t)$.

Before local optimization, each selected client $k\in S_t$ estimates
$\widehat{\bm F}_k^{\,t}$ at the common pre-update model $\theta^t$.
The report may use $B_{\mathrm{rep}}$ minibatches or the client's
complete local dataset. Client $k$ then performs $K$ local SGD steps on
the scalarized loss
\begin{equation}
L_k^t(\theta;B)
=
\sum_{i=1}^{m}w_i^t\,\ell_i(\theta;B).
\label{eq:local-weighted-loss}
\end{equation}
The server computes $\wv^t$ once before local training, and each client
holds it fixed throughout all $K$ steps. Consequently, clients neither
recompute the weights nor differentiate through the weight map during
local SGD. This design permits one local training trajectory and avoids
the task-specific gradient calculations that would be required to
update the weights within the round.

Let
$\Delta\theta_k^t:=\theta_{k,K}^t-\theta^t$ denote the resulting model
update. Client $k$ returns
$(\Delta\theta_k^t,\widehat{\bm F}_k^{\,t})$ in one uplink, and the
server averages both quantities:
\begin{equation}
\theta^{t+1}
=
\theta^t
+
\eta_g\frac{1}{|S_t|}
\sum_{k\in S_t}\Delta\theta_k^t,
\qquad
\widehat{\bm F}^{\,t}
=
\frac{1}{|S_t|}
\sum_{k\in S_t}\widehat{\bm F}_k^{\,t}.
\label{eq:fedhv-aggregation}
\end{equation}

The server forms the next weight vector by applying a numerical floor
to the estimated slacks:
\begin{equation}
\begin{aligned}
\widehat{s}_i^t
&:=
\max\left\{
r_i-\widehat f_i^{\,t},
\rho_{\mathrm{clip}}
\right\},
\qquad \rho_{\mathrm{clip}}>0,\\
w_i^{t+1}
&:=
\frac{(\widehat{s}_i^t)^{-1}}
{\sum_{j=1}^{m}(\widehat{s}_j^t)^{-1}}.
\end{aligned}
\label{eq:fedhv_weights}
\end{equation}
We denote this stabilized mapping by
$W_{\rho_{\mathrm{clip}}}(\widehat{\bm F};\bm r)$. The floor is a
numerical failsafe for noisy or anomalous reports: it keeps the
implemented weights finite but neither asserts nor enforces positivity
of the true slacks $r_i-f_i(\theta)$. The convergence analysis instead
requires the true objectives to remain inside a positive-slack domain.
\begingroup
\setlength{\textfloatsep}{6pt plus 0pt minus 0pt}
\begin{figure}[t]
\centering
\captionsetup{font=small,skip=0.5pt}

\begin{minipage}[t]{0.56\linewidth}
\vspace{0pt}
\centering
\begin{tikzpicture}[
  scale=0.86,
  transform shape,
  node distance=1.0mm,
  every node/.style={font=\scriptsize},
  box/.style={
    draw=fedhvline,
    line width=0.6pt,
    rounded corners=2pt,
    fill=fedhvblue,
    text width=4.45cm,
    align=center,
    inner xsep=3pt,
    inner ysep=2pt
  },
  clients/.style={box, fill=fedhvgreen},
  flow/.style={
    -{Latex[length=1.6mm,width=1.1mm]},
    line width=0.6pt,
    draw=fedhvline
  }
]

\node[box] (server) {
  \textbf{1. Server broadcast}\\
  Send $\theta^t$ and $\wv^t$, where
  $\wv^0=\frac{1}{m}\mathbf{1}_m$ and
  $\wv^t=W_{\rho_{\mathrm{clip}}}
  (\widehat{\bm F}^{\,t-1};\bm r)$ for $t\geq1$
};

\node[clients, below=of server] (client) {
  \textbf{2. Selected clients in parallel}\\
  Evaluate $\widehat{\bm F}_k^{\,t}$ at $\theta^t$;
  run $K$ SGD steps on $\sum_i w_i^t\ell_i$; return
  $(\Delta\theta_k^t,\widehat{\bm F}_k^{\,t})$
};

\node[box, below=of client] (aggregate) {
  \textbf{3. Server aggregation}\\
  Average the model updates and reports to obtain
  $\theta^{t+1}$ and $\widehat{\bm F}^{\,t}$
};

\node[box, below=of aggregate] (weight) {
  \textbf{4. Next-round weights}\\
  Form $\widehat{\bm s}^{\,t}$ and set
  $w_i^{t+1}\propto(\widehat s_i^t)^{-1}$
};

\draw[flow] (server) -- (client);
\draw[flow] (client) -- (aggregate);
\draw[flow] (aggregate) -- (weight);
\draw[flow,dashed]
  (weight.west) -- ++(-0.30,0) |- (server.west)
  node[pos=0.25,left,align=right,font=\tiny]
  {next\\round};

\end{tikzpicture}
\end{minipage}
\hfill
\begin{minipage}[t]{0.41\linewidth}
\vspace{0pt}
\centering
\scriptsize
\textbf{Per-round method comparison}\par\vspace{1pt}

\renewcommand{\arraystretch}{1.05}
\setlength{\tabcolsep}{0.6pt}

\begin{tabularx}{\linewidth}{@{}
  >{\raggedright\arraybackslash}X
  >{\centering\arraybackslash}p{0.46in}
  >{\centering\arraybackslash}p{0.50in}
  >{\centering\arraybackslash}p{0.44in}@{}}
\toprule
\textbf{Property}
& \textbf{FSMGDA}
& \textbf{FedCMOO}
& \textbf{FedHV} \\
\midrule

Task-specific geometry
& $\checkmark$
& $\checkmark$
& \textbf{No} \\

Server-side weight optimization
& MinNorm
& Iter. proj.
& \textbf{Closed form} \\

Single local scalar loss
& No
& $\checkmark$
& $\mathbf{\checkmark}$ \\

Per-client communication
& $\Theta(md)$
& $\Theta(d)$
& $\mathbf{\Theta(d{+}m)}$ \\

\bottomrule
\end{tabularx}
\end{minipage}

\caption{\textbf{One FedHV communication round.}
Reports from round $t-1$ determine the weights for round $t$ and
accompany the ordinary model exchange. Right: dominant per-round
operations and per-client communication; Appendix~\ref{app:complexity}
gives exact counts.}
\label{fig:fedhv-round}
\end{figure}
\endgroup
Reports collected in round $t$ determine $\wv^{t+1}$ rather than the
weights used by the same client cohort. This one-round ordering fixes
$\wv^t$ before the current clients and minibatches are sampled, avoids
an intermediate server--client exchange, and allows the weights to
remain constant during local optimization. The weights accompany the
model downlink and the reports accompany the model-update uplink,
adding exactly $2m$ auxiliary scalars per participating client without
an additional synchronization stage. Appendix~\ref{app:fedhv-algorithm}
provides complete pseudocode.

\section{Theoretical Guarantees}
\label{sec:theory}

\subsection{Assumptions and main result}

We analyze the one-round-delayed protocol over a fixed horizon
$T\geq1$. Appendix~\ref{app:theory-assumptions} defines the probability
space and states the assumptions formally.

\paragraph{Standing assumptions.}
Clients have equal aggregation weights, and each round samples $M$ of
the $N$ clients uniformly without replacement, with participation rate
$p=M/N$. Local objectives are nonnegative and smooth. Minibatch
gradients are conditionally unbiased with variance bounded by
$\sigma_i^2$, and client gradient dissimilarity is bounded by
$\zeta_i^2$ along the analyzed trajectory. Multiple local steps
introduce the corresponding client-drift term.

\paragraph{Reference-domain condition.}
The fixed reference must satisfy
\begin{equation}
r_i-f_i(\theta)\geq\rho>0,
\qquad i\in[m],
\label{eq:main-safe-reference}
\end{equation}
at every global, local, and auxiliary point used by the fixed-horizon
proof. This condition keeps the log-hypervolume objective and its
gradient well defined. The numerical floor
$0<\rho_{\mathrm{clip}}\leq\rho$ acts only on estimated slacks and does
not establish Equation~\eqref{eq:main-safe-reference} for the true
objectives.

Appendix~\ref{app:reference-certificate} gives a conditional
high-probability certificate for preserving the safe domain. Its
intuition is as follows. Suppose the initial model has a verified
positive true-slack margin and the weighted stochastic local gradients
are almost surely bounded. A round-excursion condition then limits the
movement of every local, global, and auxiliary iterate, keeping an
entire round inside a smaller positive-slack region whenever the round
begins in the safe region. A horizon-dependent budget controls the
cumulative change in $\Phi$ and implies that the global iterates remain
safe with probability at least $1-\delta$. These conditions explain
how a safe initialization and bounded updates can preserve the
reference margin; they are sufficient conditions and are not certified
by the empirical calibration procedure alone.

Define the aggregate stochastic-gradient variance and the
finite-population participation factor by
\begin{equation}
\Sigma^2
:=
\sum_{i=1}^{m}\sigma_i^2,
\qquad
\chi_N(M)
:=
\frac{N-M}{M(N-1)}
=
\frac{1-p}{p(N-1)}.
\label{eq:main-sigma-h}
\end{equation}
The partial-participation heterogeneity term is
\[
\mathcal H_N(p;\zeta)
:=
\chi_N(M)\sum_{i=1}^{m}\zeta_i^2,
\]
which vanishes under full participation.

Define the server-report error by
\begin{equation}
\Xi_{\mathrm{loss},T}
:=
\max_{0\leq t<T}
\mathbb E\!\left[
\left\|
\widehat{\bm F}^{\,t}-\bm F(\theta^t)
\right\|_\infty^2
\right].
\label{eq:main-xi-loss}
\end{equation}
For the finite-population reporting procedure used by FedHV,
Appendix~\ref{app:theory-report} establishes
\begin{equation}
\Xi_{\mathrm{loss},T}
\leq
\sum_{i=1}^{m}
\left[
\chi_N(M)\omega_{i,T}^2+\tau_{i,T}^2
\right],
\label{eq:main-report-bound}
\end{equation}
where $\omega_{i,T}^2$ measures between-client variation in objective
$i$, and $\tau_{i,T}^2$ measures finite-sample error in the local
objective reports. This decomposition gives two direct ways to tighten
the report term. Increasing the report budget reduces
$\tau_{i,T}^2$, and evaluating each selected client's complete local
dataset makes this component zero. Increasing $M$ reduces the
client-subsampling coefficient $\chi_N(M)$, which becomes zero under
full participation. Complete local evaluation therefore removes
within-client reporting noise but does not remove the error caused by
sampling only a subset of heterogeneous clients.

For $t\geq1$, FedHV uses
\[
\wv^t
=
W_{\rho_{\mathrm{clip}}}
\bigl(\widehat{\bm F}^{\,t-1};\bm r\bigr),
\]
whereas the ideal current weight is
\[
\wv^{\star,t}
=
W_{\rho_{\mathrm{clip}}}
\bigl(\bm F(\theta^t);\bm r\bigr).
\]
The Lipschitz property of the stabilized inverse-slack map, proved in
Appendix~\ref{app:theory-weight-map}, gives
\begin{align}
\|\wv^t-\wv^{\star,t}\|_1
\leq
\frac{2}{\rho_{\mathrm{clip}}}
\Big(
&\|\widehat{\bm F}^{\,t-1}
-\bm F(\theta^{t-1})\|_\infty
\nonumber\\
&+
\|\bm F(\theta^{t-1})
-\bm F(\theta^t)\|_\infty
\Big).
\label{eq:main-stale-decomposition}
\end{align}
The first term is the preceding report error, while the second measures
the inter-round change in the global objectives. This decomposition
makes the effect of one-round staleness explicit rather than treating
the implemented weights as if they were evaluated at $\theta^t$.

\begin{theorem}[FedHV stationarity]
\label{thm:fedhv}
Under Assumptions~\ref{as:uniform-aggregation},
\ref{as:smooth}, \ref{as:hetero}, \ref{as:slack},
\ref{as:steps}, and~\ref{as:objective-reports}, let
$\eta_{\mathrm{eff}}:=\eta_g\eta_\ell K$. There exist finite constants
$A_1,\ldots,A_7$, independent of
$T$, $N$, $p$, $\eta_\ell$, $\eta_g$, and $K$, such that
\begin{align}
\frac{1}{T}
\sum_{t=0}^{T-1}
\mathbb E\|\nabla\Phi(\theta^t)\|^2
\leq{}&
\frac{A_1\Delta_\Phi}{\eta_{\mathrm{eff}}T}
+
A_2\eta_{\mathrm{eff}}\Sigma^2
+
A_3\eta_{\mathrm{eff}}\mathcal H_N(p;\zeta)
\nonumber\\
&+
A_4\eta_\ell^2(K-1)(2K-1)
+
A_5\Xi_{\mathrm{loss},T}
\nonumber\\
&+
A_6\eta_{\mathrm{eff}}^2\mathcal V
+
\frac{A_7}{T},
\label{eq:main-fedhv-bound}
\end{align}
where
\[
\Delta_\Phi
:=
\Phi(\theta^0)-\Phi_{\mathrm{lb}},
\qquad
\Phi_{\mathrm{lb}}
:=
-\sum_{i=1}^{m}\log r_i.
\]
Here, $\Phi_{\mathrm{lb}}$ is a lower bound on $\Phi$, not necessarily
its global minimum, and $\mathcal V$ bounds the inter-round displacement
that enters the stale-weight error. Appendix
\ref{app:theory-convergence} specifies $\mathcal V$ and the constants
$A_1,\ldots,A_7$.
\end{theorem}

The bound separates the initial objective gap, stochastic-gradient
variance, heterogeneity under partial participation, multi-step local
drift, report error, and inter-round weight staleness. The
$A_7/T$ term accounts for the initial uniform weight vector. In
particular, the nonvanishing neighborhood is not an unspecified error
floor: its reporting component can be reduced through a larger report
budget or greater client participation, subject to the corresponding
communication and evaluation costs.

\begin{corollary}[Rate and stationarity neighborhood]
\label{cor:fedhv-rate-main}
Under the conditions of Theorem~\ref{thm:fedhv}, fix $\eta_g>0$ and,
for a target horizon $T$, choose
\[
\eta_\ell
=
\Theta\!\left(
\frac{1}{\eta_gK\sqrt{T}}
\right)
\]
within the stability range of Assumption~\ref{as:steps}. Then
\begin{equation}
\frac{1}{T}
\sum_{t=0}^{T-1}
\mathbb E\|\nabla\Phi(\theta^t)\|^2
\leq
\mathcal O(T^{-1/2})
+
A_5\Xi_{\mathrm{loss},T}.
\label{eq:main-fedhv-rate}
\end{equation}
Thus, vanishing report error yields an $\mathcal O(T^{-1/2})$ ergodic
rate, whereas persistent report error determines a stationarity
neighborhood. The schedule depends on the target horizon and differs
from the fixed step sizes used in the experiments. Appendix
\ref{app:theory-convergence} gives the explicit rate and best-iterate
statement.
\end{corollary}

The proof first applies
Equation~\eqref{eq:main-stale-decomposition} to separate report error
from inter-round objective change. A local-SGD
bias--variance--drift decomposition then controls the expected update
and its second moment. Applying smoothness of $\Phi$ to the resulting
one-round descent inequality and summing over $t$ yields
Theorem~\ref{thm:fedhv}; Appendix~\ref{app:theory-proofs} gives the
complete argument.

\subsection{Relation to Pareto stationarity}

Define the Pareto-stationarity residual
\[
\mathcal P(\theta)
:=
\min_{\lambda\in\Delta_m}
\left\|
\sum_{i=1}^{m}
\lambda_i\nabla f_i(\theta)
\right\|.
\]

\begin{proposition}[Quantitative Pareto relation]
\label{prop:pareto-main}
Let
\[
c(\theta)
:=
\sum_{i=1}^{m}
\frac{1}{r_i-f_i(\theta)},
\qquad
c_{\min}
:=
\sum_{i=1}^{m}
\frac{1}{r_i}.
\]
On the safe region,
\begin{equation}
\mathcal P(\theta)
\leq
\frac{\|\nabla\Phi(\theta)\|}{c(\theta)}
\leq
\frac{\|\nabla\Phi(\theta)\|}{c_{\min}}.
\label{eq:main-pareto-bound}
\end{equation}
\end{proposition}

The normalized inverse-slack weights belong to $\Delta_m$ and are
therefore a feasible choice in the definition of
$\mathcal P(\theta)$. Nonnegative losses also imply
$c(\theta)\geq c_{\min}$. Consequently,
$\nabla\Phi(\theta)=\mathbf{0}$ implies Pareto stationarity, and
Theorem~\ref{thm:fedhv} controls the corresponding squared Pareto
residual. FedHV obtains this relation while each client differentiates
only the scalar loss in Equation~\eqref{eq:local-weighted-loss}; the
algorithm does not construct per-task gradient geometry, test
pairwise gradient conflicts, or solve a minimum-norm problem.
Appendix~\ref{app:theory-pareto} gives the exact and averaged forms.

\subsection{Hypervolume-induced worst-task slack control}

\begin{proposition}[Minimum-slack control]
\label{prop:slack-floor-main}
Let $s_i(\theta):=r_i-f_i(\theta)>0$, suppose
$s_j(\theta)\leq S_j$ for all $j\in[m]$, and define
\[
R_{\max}(\bm S)
:=
\max_{i\in[m]}
\prod_{j\neq i}S_j.
\]
Then
\begin{equation}
\min_{i\in[m]}s_i(\theta)
\geq
\frac{e^{-\Phi(\theta)}}{R_{\max}(\bm S)}.
\label{eq:main-slack-floor}
\end{equation}
For nonnegative losses, one may take $S_j=r_j$.
\end{proposition}

Hence, if $\Phi(\theta^t)\leq B$, every task satisfies
\[
f_i(\theta^t)
\leq
r_i-\frac{e^{-B}}{R_{\max}(\bm S)}.
\]
This is a reference-relative loss-space guarantee and does not by
itself imply a bound on worst-task test accuracy. Appendix
\ref{app:barrier-control} gives the complete derivation.

\section{Experiments}
\label{sec:experiments}

\subsection{Experimental Setup}
\label{sec:experimental-setup}

\paragraph{Benchmarks.}
We use six non-IID settings from the four benchmark families adopted by
FedCMOO~\citep{askin2025fedcmoo}: \textsc{MultiMNIST},
\textsc{MNIST--FMNIST}, \textsc{CIFAR10--MNIST}, and
\textsc{CelebA5} \citep{lecun1998gradient,xiao2017fashion,
krizhevsky2009learning,liu2015deep}. The settings vary Dirichlet
concentration, participation, and, for \textsc{CelebA5}, whether client
allocation uses all attribute labels or only the first. Appendix
\ref{app:experimental-details} gives the complete constructions and
partitions.

\paragraph{Training protocol and baselines.}
Within each setting, methods share the architecture, client partition,
participation schedule, round budget, and number of local steps. We
compare FedHV with FSMGDA~\citep{yang2023fmol},
FedCMOO~\citep{askin2025fedcmoo}, and Uniform, which fixes $w_i=1/m$
in the same single-trajectory update as FedHV. We report mean and sample
standard deviation over seeds $\{0,42,2026\}$ on one fixed partition.
Appendix~\ref{app:experimental-details} gives the architectures,
optimization settings, and other configurations used
by the implementations.

\paragraph{References and evaluation metrics.}
FedHV freezes a calibration-derived training reference
$\bm r^{\mathrm{train}}$; every method uses
$\bm r^{\mathrm{eval}}=3\mathbf{1}_m$ for evaluation hypervolume. We
report $\mathrm{MeanAcc}=\frac{1}{m}\sum_i a_i$ and
$\mathrm{WorstAcc}=\min_i a_i$; Appendix~\ref{app:loss-hv-results}
reports mean loss and $\mathrm{HV}=\prod_i(r_i^{\mathrm{eval}}-\ell_i)$.

Across the reported standard runs, all retained server-aggregated objective reports have positive raw reference slack and the numerical slack floor is never activated; Appendix~\ref{app:reference-selection} reports the minimum observed margins. This is a report-level diagnostic rather than a certificate of the true-slack condition in \eqref{eq:main-safe-reference}.

\begin{table*}[!b]
\centering
\caption{Final MeanAcc and WorstAcc (\%), reported as mean $\pm$
sample standard deviation over three training runs on one fixed client
partition. Bold denotes the best result and underlining denotes the
second-best result among the four methods. Rankings are computed from
the unrounded means.}
\label{tab:main-results}
\scriptsize
\setlength{\tabcolsep}{2.0pt}
\resizebox{\textwidth}{!}{%
\begin{tabular}{lcccccccc}
\toprule
& \multicolumn{2}{c}{Uniform}
& \multicolumn{2}{c}{FSMGDA}
& \multicolumn{2}{c}{FedCMOO}
& \multicolumn{2}{c}{FedHV} \\
\cmidrule(lr){2-3}
\cmidrule(lr){4-5}
\cmidrule(lr){6-7}
\cmidrule(lr){8-9}
Setting
& Mean $\uparrow$
& Worst $\uparrow$
& Mean $\uparrow$
& Worst $\uparrow$
& Mean $\uparrow$
& Worst $\uparrow$
& Mean $\uparrow$
& Worst $\uparrow$ \\
\midrule
\textsc{MultiMNIST}
& \underline{$90.17\pm0.18$}
& $88.14\pm0.34$
& $88.78\pm1.38$
& $86.15\pm2.53$
& $89.89\pm0.16$
& \underline{$88.79\pm0.57$}
& \textbf{\boldmath$90.72\pm0.65$}
& \textbf{\boldmath$88.94\pm0.61$} \\
\textsc{MNIST--FMNIST}
& \underline{$86.41\pm0.51$}
& $76.98\pm0.77$
& $84.84\pm0.87$
& $75.58\pm1.05$
& \textbf{\boldmath$86.69\pm0.18$}
& \textbf{\boldmath$78.58\pm0.66$}
& $86.35\pm0.45$
& \underline{$77.03\pm0.70$} \\
\textsc{CelebA5}-A
& \textbf{\boldmath$88.08\pm0.11$}
& \underline{$86.30\pm0.19$}
& $86.98\pm0.06$
& $85.77\pm0.08$
& $86.60\pm0.31$
& $85.00\pm0.29$
& \underline{$88.06\pm0.21$}
& \textbf{\boldmath$86.44\pm0.16$} \\
\textsc{CelebA5}-B
& \textbf{\boldmath$87.94\pm0.07$}
& \textbf{\boldmath$86.32\pm0.05$}
& $85.88\pm0.07$
& $84.70\pm0.05$
& $84.68\pm1.94$
& $82.11\pm3.55$
& \underline{$87.88\pm0.05$}
& \underline{$86.32\pm0.11$} \\
\textsc{CIFAR10--MNIST}-A
& $72.72\pm0.11$
& $51.33\pm0.18$
& \underline{$72.76\pm0.22$}
& \underline{$51.96\pm0.55$}
& $71.54\pm0.48$
& $47.98\pm1.15$
& \textbf{\boldmath$73.18\pm0.14$}
& \textbf{\boldmath$53.03\pm0.35$} \\
\textsc{CIFAR10--MNIST}-B
& \underline{$75.33\pm0.20$}
& \underline{$56.24\pm0.32$}
& $73.10\pm0.66$
& $51.85\pm1.00$
& $74.88\pm0.51$
& $54.23\pm0.88$
& \textbf{\boldmath$75.46\pm0.15$}
& \textbf{\boldmath$57.55\pm0.20$} \\
\bottomrule
\end{tabular}
}
\end{table*}

\subsection{Comparison with Federated Multi-Objective Baselines}
\label{sec:main-comparison}

Table~\ref{tab:main-results} shows that FedHV exceeds FSMGDA and
FedCMOO in MeanAcc and WorstAcc in five of six settings; FedCMOO leads
these methods on \textsc{MNIST--FMNIST}. Among all four methods, FedHV
has the highest MeanAcc in three settings and WorstAcc in four, whereas
Uniform has the highest MeanAcc on both \textsc{CelebA5} settings and
WorstAcc on \textsc{CelebA5}-B.
The largest and most consistent separation occurs on the mixed-difficulty
\textsc{CIFAR10--MNIST} benchmark. Using the reported means, FedHV raises
WorstAcc from the next-best $51.96\%$ to $53.03\%$ in setting A and from
$56.24\%$ to $57.55\%$ in setting B, gains of $1.07$ and $1.31$
percentage points. The corresponding MeanAcc margins are smaller
($0.42$ and $0.13$ points), indicating that the main benefit is improved
performance on the weaker task rather than a uniform increase across both
objectives.

\begin{figure*}[t]
    \centering
    \includegraphics[width=0.98\textwidth]
    {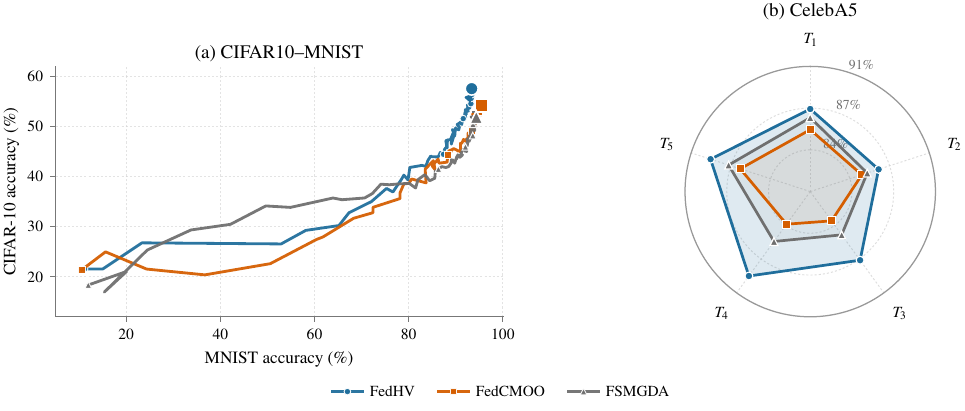}
    \caption{\textbf{Task-level accuracy, averaged over three runs.}
    \textbf{(a)} Evaluation trajectories on \textsc{CIFAR10--MNIST}-B
    ($\alpha=0.30$); the terminal cross shows one sample standard
    deviation per task. \textbf{(b)} Mean accuracy for the five
    \textsc{CelebA5}-B objectives ($\alpha=0.30$); the radial axis spans
    $80$--$91\%$. Appendix~\ref{app:per-task-results} gives numerical
    values and uncertainty.}
    \label{fig:main-task-profiles}
\end{figure*}

Figure~\ref{fig:main-task-profiles} clarifies these trade-offs. On
\textsc{CIFAR10--MNIST}-B, FedHV attains higher final CIFAR-10 accuracy
than FSMGDA and FedCMOO but lower MNIST accuracy than FedCMOO; its
aggregate gain therefore comes from the harder objective, not uniform
improvement. On \textsc{CelebA5}-B, FedHV remains comparatively
balanced across the five attribute groups. Complete per-task results,
including the alternate partition, are in
Appendix~\ref{app:per-task-results}.

FedHV obtains these results while each client trains on one scalar loss and
communicates only an $m$-dimensional objective report in addition to its model
update. Table~\ref{tab:main-results} and Figure~\ref{fig:main-task-profiles}
show that objective-value weighting can match or exceed the
geometry-based methods without task-specific model updates or a server-side
weight solver.

Accuracy and loss-space metrics do not rank the methods identically.
FedHV has the lowest MeanLoss and largest evaluation HV on
\textsc{MultiMNIST}; FedCMOO leads on \textsc{MNIST--FMNIST}; and
Uniform leads on both \textsc{CelebA5} settings. On
\textsc{CIFAR10--MNIST}, FSMGDA leads these metrics at $\alpha=0.10$
and FedCMOO at $\alpha=0.30$, although FedHV has the highest mean and
worst-task accuracy in both settings. These results indicate a
loss--accuracy trade-off rather than dominance across all criteria;
Appendix~\ref{app:loss-hv-results} gives the complete values. Beyond task accuracy, FedHV  reduces the observed computational and communication overhead of multi-objective optimization. Empirical profiling on \textsc{CIFAR10--MNIST} (detailed in Appendix~\ref{app:complexity}) confirms that FedHV reduces end-to-end training time by $24.7\%$ compared to FedCMOO and $32.9\%$ compared to FSMGDA, while exhibiting lower observed across-seed runtime variability ($\pm5.69$s vs.\ $\pm33.31$s). Furthermore, by replacing task-update geometry with low-dimensional scalar reports, FedHV eliminates asymmetric uplink bottlenecks, effectively cutting the required client uplink payload in half compared to geometry-based baselines. Crucially, these end-to-end efficiency gains fully account for FedHV's one-time preprocessing calibration cost.

\subsection{Ablation Studies}
\label{sec:ablations}

On \textsc{CIFAR10--MNIST}-B, we vary one factor from the default $(B_{\mathrm{rep}},M,K)=(\mathrm{all},10,10)$: $B_{\mathrm{rep}}\in\{1,2,5,\mathrm{all}\}$, $M\in\{5,10,20,30\}$, or $K\in\{1,5,10,20\}$.

\begin{figure*}[t]
    \centering
    \includegraphics[width=0.98\textwidth]{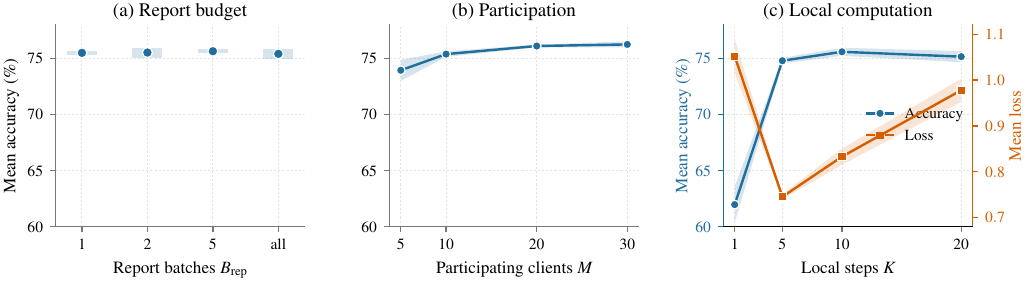}
    \caption{\textbf{FedHV ablations on \textsc{CIFAR10--MNIST}-B.} Mean $\pm$ std.\ dev.\ over three runs for \textbf{(a)} report budget $B_{\mathrm{rep}}$, \textbf{(b)} participating clients $M$, and \textbf{(c)} local steps $K$. Panel (c) includes MeanLoss on the right axis.}
    \label{fig:main-ablations}
\end{figure*}

Sparse reporting suffices: final mean accuracy varies by $\le0.23\%$ across report budgets, confirming that FedHV's low overhead does not require full local dataset evaluation. Increasing participation $M$ improves accuracy but raises per-round cost. For local computation, $K=1$ makes little progress, while increasing $K$ beyond $10$ degrades MeanLoss without accuracy gains, exhibiting a classic local-progress versus drift trade-off. Detailed numerical values are provided in Appendix~\ref{app:ablation-results}.
\section{Discussion}
\label{sec:discussion}

FedHV achieves adaptive task weighting without task-update geometry, iterative server-side solvers, or extra synchronization. Empirically, it excels under imbalanced task difficulty---improving worst-task accuracy on \textsc{CIFAR10--MNIST}---and selects a robust, distinct trade-off across evaluation metrics. Furthermore, the observed ablation trends are qualitatively consistent with the error, participation, and drift terms characterized in Theorem~\ref{thm:fedhv}. Theoretically, our convergence analysis relies on a fixed reference maintaining strictly positive true slacks; Appendix~\ref{app:reference-certificate} details sufficient conditions for this reference safety, leaving the broader impacts of calibration margins and objective scaling for future exploration.
\section*{AI Use Statement}

Generative AI tools were used to assist several stages of the research and writing workflow. Specifically, they were used for literature retrieval and the discovery of related work; brainstorming and refining research ideas; assisting with the implementation of FedHV and the evaluated baseline methods; and supporting the design and implementation of the experimental pipeline, including non-IID client-data partitioning. They were additionally used to review selected sections of the manuscript and suggest revisions to improve their organization, clarity, and readability. All suggested revisions were evaluated and finalized by the authors.

For the theoretical component, generative AI tools assisted in exploring proof strategies, drafting and refining proofs of lemmas and theoretical claims, and improving the presentation and wording of certain assumptions.

All cited works were manually inspected and verified by the authors. All AI-assisted code was reviewed and tested, and the baseline implementations were checked against their original methodological descriptions. The authors independently verified the mathematical assumptions, derivations, lemmas, and proof steps. All figures were generated from the authors' experimental outputs, and their underlying numerical values were checked for correctness. Generative AI was not used to generate synthetic datasets or to fabricate or alter experimental results. The authors reviewed all AI-assisted material and take full responsibility for the final content of the paper.
\section*{Ethics Statement}
This work uses established public benchmark datasets and does not collect new human-subject data. CelebA contains facial images and appearance-related attributes and may reflect demographic and representational biases in its original collection and annotations. Consequently, the reported results should not be interpreted as evidence of demographic fairness or deployment readiness. Our federated experiments simulate decentralized training and do not provide formal privacy guarantees; in particular, objective reports may reveal aggregate task-level information. The proposed method is intended as an optimization technique rather than a complete privacy or fairness mechanism.

\section*{Reproducibility Statement}
Appendix~\ref{app:full-theory} presents the formal assumptions, auxiliary results, conditional reference-safety certificate, explicit convergence analysis, and complete proofs of the theoretical claims. Appendix~\ref{app:experimental-details} documents the benchmark and objective construction, model architectures, federated data partitioning, training and evaluation protocol, baseline implementations, method-specific learning rates, FedHV hyperparameters, reference-calibration procedure, and random seeds. Appendix~\ref{app:complete-results} reports the complete loss, hypervolume, and task-level results; learning curves; ablation studies; weight trajectories; additional diagnostics; and uncertainty across training seeds. Appendix~\ref{app:complexity} provides analytical communication and computational complexity comparisons together with empirical wall-clock and communication measurements. Finally, Appendix~\ref{app:fedhv-algorithm} gives the complete round-level pseudocode for FedHV. The client-partition seed, training seeds, optimization settings, participation and local-update configurations, objective-report settings, and frozen training references used in the experiments are explicitly reported.

\bibliography{example_paper}
\bibliographystyle{iclr2027_conference}

\newpage
\appendix
\onecolumn

\section{Full Theoretical Development}\label{app:full-theory}

{
\subsection{Probability Space and Round Filtration}\label{app:theory-filtration}
\paragraph{Objective and round-wise notation.}
We analyze FedHV under the log-hypervolume objective introduced in
Section~\ref{sec:method}. Throughout this section, we denote the slack of
objective \(i\) by
\[
s_i(\theta)
:=
r_i-f_i(\theta),
\]
and write
\[
\bm F(\theta)
:=
\bigl(f_1(\theta),\ldots,f_m(\theta)\bigr)^\top .
\]
Throughout the theoretical section, \(\bm F(\theta)\) denotes this
global objective vector, while \(f_i(\theta)\) denotes its scalar
components. Likewise, \(\bm r\) is the fixed reference vector and
\(r_i\) is its \(i\)-th component.
The reference vector \(\bm r\) is assumed to be chosen safely so that all
true slacks remain positive along the optimization trajectory.

For every round \(t\), define the ideal current hypervolume weights
\[
w_i^{\star,t}
:=
w_i^\star(\theta^t)
=
\frac{
\bigl(r_i-f_i(\theta^t)\bigr)^{-1}
}{
\sum_{j=1}^m
\bigl(r_j-f_j(\theta^t)\bigr)^{-1}
},
\qquad i=1,\ldots,m.
\]
We use \(\wv^{\star,t}\) to denote the corresponding vector of ideal
weights. These weights are used only as an analytical reference.

The actual weight vector used by FedHV in communication round \(t\)
is denoted by \(\wv^t\). Following Algorithm~\ref{alg:fedhv},
\[
\wv^0=\frac{1}{m}\mathbf{1}_m,
\]
and, for \(t\geq1\),
\[
\wv^t
=
W_{\rho_{\mathrm{clip}}}\!\left(\widehat{\bm F}^{\,t-1};\bm r\right),
\]
where \(W_{\rho_{\mathrm{clip}}}\) denotes the stabilized inverse-slack normalization
defined in \eqref{eq:fedhv_weights}.
Hence, \(\wv^t\) is computed from objective estimates collected in
round \(t-1\), whereas \(\wv^{\star,t}\) is the ideal hypervolume
weight associated with the current model \(\theta^t\).
}
{
\paragraph{Round history and client sampling.}
Let \(\mathcal{F}_t\) denote the sigma-field generated by all randomness
revealed before the participating client set \(S_t\) is sampled in
round \(t\). In particular, the current global model \(\theta^t\), all
objective estimates collected through round \(t-1\), and the stale
weight vector \(\wv^t\) are \(\mathcal{F}_t\)-measurable.

Conditioned on \(\mathcal{F}_t\), the participating set \(S_t\) is
sampled uniformly without replacement from the \(N\) clients, with
\[
|S_t|=M=pN,
\]
and is independent of the current-round minibatch randomness.
Consequently, because \(\wv^t\) is fixed conditional on
\(\mathcal{F}_t\), for any collection of client vectors
\(\{a_k(\wv^t)\}_{k=1}^N\) that is
\(\mathcal{F}_t\)-measurable,
\[
\mathbb{E}
\left[
\left.
\frac{1}{M}
\sum_{k\in S_t}
a_k(\wv^t)
\right|
\mathcal{F}_t
\right]
=
\frac{1}{N}
\sum_{k=1}^N
a_k(\wv^t).
\]
Let \(\mathcal U_t^{\mathrm{rep}}\) denote all randomness used to form
the objective reports after \(S_t\) is selected and before local SGD
starts. For local step \(s\), define the within-round filtration
\[
\mathcal{G}_{t,s}
:=
\sigma\left(
\mathcal{F}_t,S_t,\mathcal U_t^{\mathrm{rep}},
\{B_{k,u}^t:k\in S_t,\;0\leq u<s\}
\right),
\]
where \(B_{k,u}^t\) denotes the minibatch used at local step \(u\) in
round \(t\). The current local iterates
\(\theta_{k,s}^t\) and the weight vector \(\wv^t\) are
\(\mathcal{G}_{t,s}\)-measurable.
Conditionally on \(\mathcal G_{t,s}\), the current minibatch
\(B_{k,s}^t\) is sampled afresh with replacement from the fixed
empirical distribution on \(\mathcal D_k\). In particular, it is
independent of previous local minibatches and report randomness
conditional on the fixed client dataset. The next-round history
\(\mathcal F_{t+1}\) contains the current client sample,
\(\mathcal U_t^{\mathrm{rep}}\), all minibatches, client updates, and
reported objective estimates. No independence between objective-report
errors and client sampling is required by the objective-estimation
bound.
}
\subsection{Formal Assumptions}\label{app:theory-assumptions}
For the convergence result, fix a deterministic target horizon
\(T\geq1\). All trajectory-wise assumptions and all convergence
statements below concern the communication rounds
\(t=0,\ldots,T-1\).

\begin{assumption}[Uniform client averaging]
\label{as:uniform-aggregation}
For the theoretical result, all clients have equal aggregation weight,
that is, \(n_k=n/N\) and
\[
f_i(\theta)
=
\frac{1}{N}\sum_{k=1}^{N}f_i^{(k)}(\theta).
\]
The server uses the simple average of the selected-client updates and
objective reports, and \(S_t\) is sampled uniformly without replacement
with \(|S_t|=M=pN\). Thus, conditional on \(\mathcal F_t\), every
selected-client average of an \(\mathcal F_t\)-measurable client vector
is an unbiased estimator of its population average. The theorem does
not cover unequal client sizes or data-weighted aggregation.
\end{assumption}

\begin{assumption}[Local smoothness and stochastic gradients]
\label{as:smooth}
For every client \(k\in[N]\) and objective \(i\in[m]\), the local
objective \(f_i^{(k)}\) is nonnegative and \(L_i\)-smooth; that is, for all
\(\theta,\theta'\in\mathbb{R}^d\),
\begin{equation}
\label{eq:local_smoothness}
\left\|
\nabla f_i^{(k)}(\theta)
-
\nabla f_i^{(k)}(\theta')
\right\|
\le
L_i
\left\|
\theta-\theta'
\right\|.
\end{equation}
Consequently, the global objective
\[
f_i(\theta)
=
\frac{1}{N}
\sum_{k=1}^{N}
f_i^{(k)}(\theta)
\]
is also \(L_i\)-smooth.

For every round \(t\), local step \(s\), participating client
\(k\in S_t\), and objective \(i\), let \(B_{k,s}^t\) be the minibatch
sampled and used at that step. Specifically, \(B_{k,s}^t\) contains
\(b\) indices drawn independently and uniformly with replacement from
\(\{1,\ldots,n_k\}\), and these draws are fresh conditional on
\(\mathcal G_{t,s}\). The same minibatch is used for all objectives,
which allows their stochastic gradients to be correlated. If
\(g_i^{(k)}(\theta;B)\) denotes the average of the per-example
gradients over \(B\), then for the
\(\mathcal G_{t,s}\)-measurable local iterate
\(\theta_{k,s}^t\), conditional unbiasedness follows from the sampling
rule:
\begin{equation}
\label{eq:stochastic_gradient_unbiased}
\mathbb{E}
\left[
g_i^{(k)}\bigl(\theta_{k,s}^t;B_{k,s}^t\bigr)
\mid
\mathcal{G}_{t,s}
\right]
=
\nabla f_i^{(k)}\bigl(\theta_{k,s}^t\bigr),
\end{equation}
This identity is exact for the stated sampling protocol. We assume
only the usual uniform conditional variance bound: there
exist finite constants \(\sigma_i^2\) such that
\begin{equation}
\label{eq:stochastic_gradient_variance}
\mathbb{E}
\left[
\left\|
g_i^{(k)}\bigl(\theta_{k,s}^t;B_{k,s}^t\bigr)
-
\nabla f_i^{(k)}\bigl(\theta_{k,s}^t\bigr)
\right\|^2
\mid
\mathcal{G}_{t,s}
\right]
\le
\sigma_i^2.
\end{equation}
The task-gradient noises produced from the same minibatch may be
arbitrarily correlated across objectives; no independence across
objectives is assumed. The fresh-sampling rule is stronger than the
martingale-difference property needed in the proofs; the proofs do not
condition on future iterates and do not require independence of all
noise variables across local steps.
\end{assumption}

\begin{assumption}[Trajectory-wise client heterogeneity]\label{as:hetero}
For each objective $i\in[m]$, define the population gradient-dissimilarity
function
\[
\mathsf H_i(\theta)
:=
\frac{1}{N}\sum_{k=1}^N
\left\|
\nabla f_i^{(k)}(\theta)-\nabla f_i(\theta)
\right\|^2.
\]
There exists a finite constant $\zeta_i^2$ such that, almost surely,
for every round $t=0,\ldots,T-1$,
\[
\mathsf H_i(\theta^t)\leq\zeta_i^2,
\qquad
\mathsf H_i(v_{t,s})\leq\zeta_i^2,
\quad s=0,\ldots,K,
\]
where $v_{t,s}$ is the virtual trajectory in
\eqref{eq:virtual-trajectory}. This is a trajectory-wise condition on
the fixed horizon: it is imposed only at the global and virtual
iterates that the proof actually evaluates, not uniformly over every
\(\theta\in\mathcal D\). The actual local iterates and all connecting
segments are covered by the fixed-horizon domain condition in
Assumption~\ref{as:slack}. Their deviation from the virtual trajectory
is controlled by Lemma~\ref{lem:grad2-derived}, so no separate
global-in-$\mathcal D$ heterogeneity bound is required at arbitrary
local points.
\end{assumption}

\begin{assumption}[Safe reference on the fixed analysis horizon]
\label{as:slack}
Let \(\mathcal D\subseteq\mathbb R^d\) be a convex analysis region
containing the points that appear in the convergence proof. Define the
safe subset
\[
\mathcal D_\rho
:=
\left\{
\theta\in\mathcal D:
\min_{i\in[m]}\bigl(r_i-f_i(\theta)\bigr)\geq\rho
\right\},
\qquad \rho>0.
\]
The initialization is assumed safe:
\[
\theta^0\in\mathcal D_\rho.
\]
For a communication round \(t\), let \(\mathcal B_t\) be the union of
the line segments used by the proof,
\[
\mathcal B_t
:=[\theta^t,\theta^{t+1}]
\;\cup\!\!\bigcup_{k\in S_t}\bigcup_{s=0}^{K-1}
[\theta_{k,s}^t,v_{t,s}]
\;\cup\!\!\bigcup_{s=0}^{K-1}[v_{t,s},\theta^t],
\]
and let
\[
\mathcal R_t
:=\{\theta^t,\theta^{t+1}\}
\cup\{\theta_{k,s}^t:k\in S_t,\ 0\leq s\leq K\}
\cup\{v_{t,s}:0\leq s\leq K\}
\cup\mathcal B_t.
\]
Here \([a,b]\) denotes the line segment joining \(a\) and \(b\); the
actual local and virtual iterates are defined in the convergence proof.
For every communication round \(t=0,\ldots,T-1\), assume, almost
surely,
\[
\mathcal R_t\subseteq\mathcal D_\rho.
\]
Equivalently, every global iterate, local iterate, virtual iterate,
and connecting segment used by the proof remains in the domain on the
entire deterministic horizon. This is a domain condition for the
logarithm, not an independent bounded-gradient assumption.
The same fixed reference vector
\(\bm r\) is used in every round. The implementation uses a separate
numerical clipping floor \(0<\rho_{\mathrm{clip}}\le\rho\). When the server uses
the round-\(t\) estimated global objective vector
\(\widehat{\bm F}^{\,t}\), the estimated slacks used to compute the
next task weights are stabilized according to
\begin{equation}
\label{eq:estimated_clipped_slack}
\widehat s_i^{\,t}
:=
\max
\left\{
r_i-\widehat f_i^{\,t},
\rho_{\mathrm{clip}}
\right\}.
\end{equation}
The task losses and their estimates are nonnegative:
\[
f_i(\theta)\ge 0,
\qquad
\widehat f_i^{\,t}\ge 0,
\]
and the fixed reference satisfies
\[
r_i\ge \rho>0,
\qquad
\forall i\in[m].
\]
\end{assumption}

\begin{remark}[Safe reference versus slack stabilization]
\label{rem:safe-reference}
The set \(\mathcal D_\rho\) makes explicit the domain on which the true
slacks
\(r_i-f_i(\theta)\) are bounded below by \(\rho\). This condition ensures
that the log-hypervolume objective
\[
\Phi(\theta)
=
-\sum_{i=1}^{m}
\log\!\bigl(r_i-f_i(\theta)\bigr)
\]
is well defined and smooth on every segment used on the fixed horizon.

The clipping operation in
\eqref{eq:estimated_clipped_slack} has a different role: the numerical
floor \(\rho_{\mathrm{clip}}\) stabilizes task weights computed from estimated
losses and prevents their inverse-slack factors from becoming
numerically unbounded. The safe margin \(\rho\), by contrast, is a
property of the \emph{true} slacks required by the analysis. Because
\(0<\rho_{\mathrm{clip}}\le\rho\), clipping is inactive for the true global
objective vector on \(\mathcal D_\rho\). Numerical clipping does not
enforce the safe-reference condition for the true objective slacks. Accordingly, the convergence analysis is a
fixed-horizon result under Assumption~\ref{as:slack}. The practical
fixed-reference selection used in the experiments is an empirical
calibration rule and does not by itself verify this almost-sure domain
condition; it is described in Appendix~\ref{app:reference-selection}.
Appendix~\ref{app:reference-certificate} gives separate sufficient
conditions that convert a verified initial margin into a
high-probability trajectory certificate.
\end{remark}

\subsection{Basic Properties}\label{app:theory-basic}
\begin{lemma}[Self-bounded global gradients]
\label{lem:self-bounded-gradients}
Under Assumptions~\ref{as:uniform-aggregation} and~\ref{as:smooth},
whenever \(f_i(\theta)\le r_i\),
\[
\|\nabla f_i(\theta)\|^2
\leq
2L_i f_i(\theta)
\leq
2L_i r_i
=:
G_i^2.
\]
Consequently, with \(G_{\max}:=\max_i G_i\), every global
scalarized gradient with weights in \(\Delta_m\) satisfies
\[
\left\|\sum_{i=1}^m w_i\nabla f_i(\theta)\right\|
\leq G_{\max}.
\]
This bound is derived from nonnegativity and smoothness; it is not an
independent bounded-gradient assumption. Its \(\rho\)-free form avoids
circular constants when Appendix~\ref{app:reference-certificate}
derives a safe margin rather than assuming one.
\end{lemma}

{
\paragraph{Weighted stochastic-gradient noise.}
At local step \(s\) of communication round \(t\), participating client
\(k\in S_t\) uses the stale weight vector \(\wv^t\), which is
\(\mathcal{F}_t\)-measurable, and forms the weighted stochastic gradient
\[
g_{k,s}^t
:=
\sum_{i=1}^m
w_i^t\,
g_i^{(k)}
\bigl(\theta_{k,s}^t;B_{k,s}^t\bigr).
\]
The corresponding exact local weighted gradient is
\[
\nabla
L_{\wv^t}^{(k)}
\bigl(\theta_{k,s}^t\bigr)
=
\sum_{i=1}^m
w_i^t
\nabla
f_i^{(k)}
\bigl(\theta_{k,s}^t\bigr).
\]
In the implementation this is the gradient of
\[
\sum_{i=1}^m
\sg{w_i^t}\,\ell_i(\theta;B_{k,s}^t),
\]
where \(\sg{\cdot}\) denotes stop-gradient. Thus the weights are
held fixed during all \(K\) local steps and no derivative through the
weight-estimation map is taken.
We define the weighted stochastic-gradient noise by
\[
\xi_{k,s}^t
:=
g_{k,s}^t
-
\nabla
L_{\wv^t}^{(k)}
\bigl(\theta_{k,s}^t\bigr).
\]
Writing
\(\varepsilon_{i,k,s}^t
:=g_i^{(k)}(\theta_{k,s}^t;B_{k,s}^t)
-\nabla f_i^{(k)}(\theta_{k,s}^t)\), we have
\(\xi_{k,s}^t=\sum_i w_i^t\varepsilon_{i,k,s}^t\). Since
\(\wv^t\in\Delta_m\), Jensen's inequality gives
\[
\mathbb E\left[
\left.\|\xi_{k,s}^t\|^2\right|\mathcal G_{t,s}
\right]
\leq
\sum_{i=1}^m w_i^t\sigma_i^2
\leq
\Sigma^2,
\]
without requiring the task noises generated by a shared minibatch to
be uncorrelated. For the average over the participating clients, let
\(\kappa_M\) denote any valid bound satisfying
\[
\mathbb E\left[
\left.
\left\|
\frac1M\sum_{k\in S_t}\xi_{k,s}^t
\right\|^2
\right|\mathcal G_{t,s}
\right]
\leq
\kappa_M\Sigma^2.
\]
Jensen's inequality always gives \(\kappa_M=1\), which is the value
used in the main theorem. Under the additional cross-client condition
(CCO-grad) stated in Remark~\ref{rem:optional-cross-client}, the usual
federated variance reduction gives \(\kappa_M=1/M\).
}

{
\begin{lemma}[Martingale-difference weighted noise]
\label{lem:martingale-noise}
Under Assumption~\ref{as:smooth}, the weighted noise defined above
satisfies, for every \(k\in S_t\) and local step \(s\):
\[
\mathbb{E}
\left[
\xi_{k,s}^t
\mid
\mathcal{G}_{t,s}
\right]
=
\mathbf{0},
\]
and
\[
\mathbb{E}
\left[
\|\xi_{k,s}^t\|^2
\mid
\mathcal{G}_{t,s}
\right]
\leq \Sigma^2,
\qquad
\Sigma^2:=\sum_{i=1}^m\sigma_i^2.
\]
This is a consequence of the fresh conditional minibatch sampling
rule and is not an additional stochastic-gradient assumption. No
independence across objectives, clients, or local steps is required
for the martingale-difference statement.
The client-averaged bound with \(\kappa_M=1\) follows directly from
Jensen's inequality. Under (CCO-grad) in
Remark~\ref{rem:optional-cross-client}, one may sharpen this bound to
\(\kappa_M=1/M\). This optional condition is the only place where a
\(1/M\) stochastic-variance factor is used.
\end{lemma}
}

\begin{lemma}[Smoothness of the log-hypervolume objective]
\label{lem:smooth-phi}
Under Assumptions~\ref{as:smooth} and~\ref{as:slack}, together with
Lemma~\ref{lem:self-bounded-gradients}, the
log-hypervolume objective is \(L_\Phi\)-smooth on every line segment
contained in \(\mathcal D_\rho\); specifically, for any
\(\theta,\theta'\in\mathcal D\) with
\([\theta,\theta']\subseteq\mathcal D_\rho\),
\[
 \|\nabla\Phi(\theta)-\nabla\Phi(\theta')\|
 \leq L_\Phi\|\theta-\theta'\|,
\qquad
L_\Phi
\le
\sum_{i=1}^m \frac{L_i}{\rho}
+
\sum_{i=1}^m \frac{G_i^2}{\rho^2}.
\]
The proof is given in Appendix~\ref{app:proof-smooth-phi}.
\end{lemma}

\begin{assumption}[Stepsizes]
\label{as:steps}
Let
\[
L_w:=\sum_{i=1}^{m}L_i,
\qquad
c^t
:=
\sum_{i=1}^{m}
\bigl(r_i-f_i(\theta^t)\bigr)^{-1},
\]
and define
\[
c_{\min}
:=
\sum_{i=1}^{m}\frac{1}{r_i},
\qquad
c_{\max}
:=
\frac{m}{\rho}.
\]
Let \(\eta_\ell>0\) and \(\eta_g>0\). The local and effective
stepsizes satisfy
\begin{equation}
\label{eq:stepsize_conditions}
\eta_\ell K
\le
\frac{c_1}{L_w},
\qquad
\eta_{\mathrm{eff}}
:=
\eta_g\eta_\ell K
\le
\frac{
c_2c_{\min}^{\,2}
}{
c_{\max}L_\Phi
},
\end{equation}
where, throughout the theorem and its proof, we fix the numerical
values
\[
c_1:=\frac14,
\qquad
c_2:=\frac1{48}.
\]
Thus there is no unspecified ``sufficiently small'' constant in the
statement. The value of \(c_1\) gives
\(\eta_\ell L_ws\leq1/4\) for every \(s\leq K\), which is the
condition used in Lemma~\ref{lem:grad2-derived}. Since the numerical
coefficient \(C_0\) in Lemma~\ref{lem:update-second-moment} is
\(C_0=6\), the value \(c_2=1/(8C_0)=1/48\) leaves the explicit
absorption margin used in the one-round descent proof. Define for later
reference
\[
\bar\eta_{\mathrm{eff}}
:=\frac{c_2c_{\min}^{\,2}}{c_{\max}L_\Phi}.
\]
\end{assumption}

\begin{lemma}[Log-hypervolume stationarity implies Pareto stationarity]
\label{lem:hv-pareto}
If $\nabla\Phi(\theta^\star)=0$, then there exists
$\lambda\in\Delta_m$ such that
\[
\sum_{i=1}^m
\lambda_i\nabla f_i(\theta^\star)
=
0 .
\]
Therefore, every stationary point of the log-hypervolume objective is
Pareto-stationary.
The proof is given in
Appendix~\ref{app:proof-hv-pareto}.
\end{lemma}

\subsection{Barrier Control of the Minimum Task Slack}
\label{app:barrier-control}

The log-hypervolume objective is also a logarithmic barrier for the
feasible box \(\{\theta:f_i(\theta)<r_i,\ i\in[m]\}\).

\begin{lemma}[Barrier level set and slack floor]
\label{lem:barrier-level-set}
Let \(s_i(\theta):=r_i-f_i(\theta)>0\), and suppose
\(s_j(\theta)\le S_j\) for constants \(S_j>0\). Define
\[
R_{\max}(\bm S):=\max_{i\in[m]}\prod_{j\ne i}S_j,
\qquad
\psi(\theta):=-\log\min_i s_i(\theta).
\]
Then
\begin{equation}
\label{eq:appendix-slack-floor}
\min_i s_i(\theta)
\ge \frac{e^{-\Phi(\theta)}}{R_{\max}(\bm S)},
\end{equation}
and
\begin{equation}
\label{eq:barrier-sandwich}
\psi(\theta)-\log R_{\max}(\bm S)
\le \Phi(\theta)
\le m\psi(\theta).
\end{equation}
Because every objective is nonnegative, \(S_j=r_j\) is always a
valid choice.
\end{lemma}

\begin{proof}
For any \(i\), the identity
\(\sum_j\log s_j(\theta)=-\Phi(\theta)\) gives
\[
\log s_i(\theta)
=-\Phi(\theta)-\sum_{j\ne i}\log s_j(\theta)
\ge-\Phi(\theta)-\log R_{\max}(\bm S).
\]
Exponentiating and minimizing over \(i\) proves
\eqref{eq:appendix-slack-floor}. For the right inequality in
\eqref{eq:barrier-sandwich}, each \(-\log s_i\le\psi\). For the left
inequality, choose \(i^\star\in\arg\max_i(-\log s_i)\) and use
\[
\Phi(\theta)
=\psi(\theta)+\sum_{j\ne i^\star}-\log s_j(\theta)
\ge\psi(\theta)-\log R_{\max}(\bm S).
\]
\end{proof}

Taking \(\Phi(\theta)\le B\) in
\eqref{eq:appendix-slack-floor} proves
Proposition~\ref{prop:slack-floor-main}: every task obeys
\(f_i(\theta)\le r_i-e^{-B}/R_{\max}(\bm S)\). The guarantee concerns
reference-relative objective losses. It does not by itself convert to
a worst-task test-accuracy guarantee.

\subsection{Conditional Reference-Safety Certificate}
\label{app:reference-certificate}

The base convergence theorem conditions on
Assumption~\ref{as:slack}. The following separate result shows how an
initial true-slack margin can certify that domain condition with high
probability under additional bounded-update requirements. It does not
claim that the empirical calibration rule alone verifies them.

\begin{assumption}[Almost-surely bounded local updates]
\label{as:bounded-stochastic-gradient}
There is \(G_{\mathrm{sg}}\ge G_{\max}\) such that the weighted local
stochastic gradient satisfies
\(\|g_{k,s}^t\|\le G_{\mathrm{sg}}\) almost surely for every
\(t,k,s\).
\end{assumption}

\begin{lemma}[Round-excursion buffer]
\label{lem:round-excursion-buffer}
Let
\[
\Delta_{\max}:=\eta_\ell K G_{\mathrm{sg}}
\max\{1,\eta_g\},
\]
and suppose
\begin{equation}
\label{eq:round-buffer-condition}
G_{\max}\Delta_{\max}
+\frac{L_{\max}}{2}\Delta_{\max}^2
\le \frac{\rho}{2},
\qquad L_{\max}:=\max_i L_i.
\end{equation}
If \(\theta^t\in\mathcal D_\rho\), then every global, local, and
virtual iterate and every connecting segment in \(\mathcal R_t\) lies
in \(\mathcal D_{\rho/2}\) almost surely.
\end{lemma}

\begin{proof}
For any \(\theta'\) with
\(\|\theta'-\theta^t\|\le\Delta_{\max}\), smoothness and
Lemma~\ref{lem:self-bounded-gradients} give
\[
f_i(\theta')
\le f_i(\theta^t)+G_{\max}\Delta_{\max}
+\frac{L_{\max}}2\Delta_{\max}^2
\le f_i(\theta^t)+\frac\rho2.
\]
Thus the closed ball of radius \(\Delta_{\max}\) around \(\theta^t\)
is contained in \(\mathcal D_{\rho/2}\). Assumption
\ref{as:bounded-stochastic-gradient} gives
\(\|\theta_{k,s}^t-\theta^t\|\le\eta_\ell K G_{\mathrm{sg}}\) and
\(\|\theta^{t+1}-\theta^t\|
\le\eta_g\eta_\ell K G_{\mathrm{sg}}\). The same induction for the
virtual trajectory uses
\(\|\nabla L_{\wv^t}(v_{t,s})\|\le G_{\max}\). All these points lie
in the ball, and convexity of the ball contains the required segments.
\end{proof}

\begin{lemma}[Maximal Azuma inequality]
\label{lem:maximal-azuma}
If \((Y_t)_{t\ge0}\) is a supermartingale with \(Y_0=0\) and
\(|Y_{t+1}-Y_t|\le c\) almost surely, then
\[
\mathbb P\!\left(\max_{0\le t\le T}Y_t\ge\lambda\right)
\le \exp\!\left(-\frac{\lambda^2}{2Tc^2}\right).
\]
\end{lemma}

\begin{proof}
Hoeffding's lemma makes
\(\exp(\gamma Y_t-\gamma^2c^2t/2)\) a nonnegative
supermartingale. Ville's maximal inequality therefore bounds the
displayed probability by
\(\exp(-\gamma\lambda+\gamma^2c^2T/2)\); minimizing at
\(\gamma=\lambda/(c^2T)\) proves the result.
\end{proof}

\begin{theorem}[Conditional self-certification of reference safety]
\label{thm:reference-safety-certificate}
Assume \(s_i(\theta^0)>0\),
Assumption~\ref{as:bounded-stochastic-gradient}, and
\eqref{eq:round-buffer-condition}. Let
\(R_r:=\max_i\prod_{j\ne i}r_j\),
\(\eta_{\mathrm{eff}}:=\eta_g\eta_\ell K\), and
\[
L_\Phi(\varrho)
:=\sum_i\frac{L_i}{\varrho}
+\sum_i\frac{G_i^2}{\varrho^2}.
\]
For any \(\mathcal F_t\)-measurable weights \(\wv^t\in\Delta_m\), set
\begin{align*}
\mu(\rho)
&:=\frac{mG_{\max}}{\rho}
\left(\eta_{\mathrm{eff}}G_{\max}
+\frac12\eta_g\eta_\ell^2K^2L_wG_{\mathrm{sg}}\right)
+\frac12L_\Phi(\rho/2)\eta_{\mathrm{eff}}^2G_{\mathrm{sg}}^2,\\
D(\rho)
&:=\frac{2mG_{\max}}{\rho}
\eta_{\mathrm{eff}}G_{\mathrm{sg}}
+\frac12L_\Phi(\rho/2)\eta_{\mathrm{eff}}^2G_{\mathrm{sg}}^2.
\end{align*}
If, for \(\delta\in(0,1)\),
\begin{equation}
\label{eq:reference-safety-budget}
\Phi(\theta^0)+T\mu(\rho)
+\bigl(D(\rho)+\mu(\rho)\bigr)
\sqrt{2T\log(1/\delta)}
\le \log\frac{1}{\rho R_r},
\end{equation}
then, with probability at least \(1-\delta\),
\(\theta^t\in\mathcal D_\rho\) for every \(t\le T\) and
\(\mathcal R_t\subseteq\mathcal D_{\rho/2}\) for every \(t<T\).
Hence the domain condition of Assumption~\ref{as:slack} holds at level
\(\rho/2\) on this event.
\end{theorem}

\begin{proof}
Let \(B:=\log(1/(\rho R_r))\). By
Lemma~\ref{lem:barrier-level-set}, \(\{\Phi\le B\}\subseteq
\mathcal D_\rho\). Define the stopping time
\(\tau:=\min\{t\ge0:\Phi(\theta^t)>B\}\), with
\(\min\varnothing=\infty\). Condition
\eqref{eq:reference-safety-budget} implies \(\tau\ge1\).

On \(\{t<\tau\}\), Lemma~\ref{lem:round-excursion-buffer} places the
whole round in \(\mathcal D_{\rho/2}\). Smoothness of \(\Phi\) on
that set, the identity
\[
\mathbb E[\Delta^t\mid\mathcal F_t]
=-\eta_\ell K\nabla L_{\wv^t}(\theta^t)+b_t,
\qquad
\Delta^t:=\frac1M\sum_{k\in S_t}\Delta\theta_k^t,
\]
and local smoothness give
\[
\|b_t\|
\le\frac12L_wG_{\mathrm{sg}}\eta_\ell^2K^2,
\qquad
\|\Delta^t\|\le\eta_\ell K G_{\mathrm{sg}}.
\]
Together with
\(\|\nabla\Phi(\theta^t)\|\le mG_{\max}/\rho\), these bounds yield
\begin{align*}
\mathbb E[\Phi(\theta^{t+1})-\Phi(\theta^t)\mid\mathcal F_t]
&\le\mu(\rho),\\
|\Phi(\theta^{t+1})-\Phi(\theta^t)|
&\le D(\rho)
\end{align*}
on \(\{t<\tau\}\). The second inequality uses the sharper bound
\(\|\nabla\Phi\|\le2mG_{\max}/\rho\) along the safe segment.

Now set
\(X_t:=\Phi(\theta^{t\wedge\tau})\) and
\(Y_t:=X_t-X_0-(t\wedge\tau)\mu(\rho)\). Then \((Y_t)\) is a
supermartingale with increments bounded by
\(D(\rho)+\mu(\rho)\). Lemma~\ref{lem:maximal-azuma} and
\eqref{eq:reference-safety-budget} imply, with probability at least
\(1-\delta\),
\[
\Phi(\theta^{t\wedge\tau})
\le\Phi(\theta^0)+T\mu(\rho)
+\bigl(D(\rho)+\mu(\rho)\bigr)
\sqrt{2T\log(1/\delta)}
\le B
\]
simultaneously for all \(t\le T\). If \(\tau\le T\), substituting
\(t=\tau\) contradicts the definition of \(\tau\). Thus
\(\tau>T\), which proves the claimed safety statements.
\end{proof}

\begin{remark}[Scope of the certificate]
The theorem starts from a true-slack-safe initialization and requires
almost-surely bounded local stochastic gradients plus the explicit
buffer and budget inequalities. These conditions can in principle be
checked after selecting a calibration margin, but they are not verified
numerically for the present experiments. The result therefore narrows
the theory--practice gap without claiming that the reported calibration
heuristic automatically certifies Assumption~\ref{as:slack}. Moreover,
Theorem~\ref{thm:fedhv} retains its original expectation statement;
the high-probability safety event is not used as an implicit
conditioning step in that theorem.
\end{remark}

\paragraph{Finite-population participation factor and heterogeneity.}
Using the number of participating clients
\[
M:=|S_t|=pN,
\]
define
\[
\chi_N(M)
:=
\frac{N-M}{M(N-1)}
=
\frac{1-p}{p(N-1)}.
\]
The gradient-side partial-participation term is
\[
\mathcal{H}_N(p;\zeta)
:=
\chi_N(M)\sum_{i=1}^m \zeta_i^2 .
\]
In particular,
\[
\mathcal{H}_N(p;\zeta)
\le
\frac{1-p}{p}
\sum_{i=1}^m \zeta_i^2 .
\]
Here $\zeta_i^2$ is the trajectory-wise envelope from
Assumption~\ref{as:hetero} on the fixed analysis horizon; it is not a
supremum over all points in \(\mathcal D\).

{
\subsection{Weight-Map Stability}\label{app:theory-weight-map}
\noindent\textbf{Notation (HV weight map and stale weights).}
For any nonnegative objective vector
\(z=(z_1,\ldots,z_m)^\top\), define the stabilized inverse-slack
weight map \(W_{\rho_{\mathrm{clip}}}(\cdot;\bm r)\) componentwise by
\[
[W_{\rho_{\mathrm{clip}}}(z;\bm r)]_i
:=
\frac{
\bigl(\max\{r_i-z_i,\rho_{\mathrm{clip}}\}\bigr)^{-1}
}{
\sum_{j=1}^m
\bigl(\max\{r_j-z_j,\rho_{\mathrm{clip}}\}\bigr)^{-1}
},
\qquad i=1,\ldots,m.
\]
Under Assumption~\ref{as:slack} and
\(0<\rho_{\mathrm{clip}}\le\rho\), clipping is inactive for the true global
objective vector \(\bm F(\theta)\). Hence, the ideal hypervolume
weight vector can be written as
\[
\wv^\star(\theta)
=
W_{\rho_{\mathrm{clip}}}(\bm F(\theta);\bm r).
\]
In particular,
\[
\wv^{\star,t}
:=
\wv^\star(\theta^t)
=
W_{\rho_{\mathrm{clip}}}(\bm F(\theta^t);\bm r).
\]

The actual FedHV weight used in round \(t\) is
\[
\wv^0=\frac{1}{m}\mathbf{1}_m,
\qquad
\wv^t
=
W_{\rho_{\mathrm{clip}}}(\widehat{\bm F}^{\,t-1};\bm r),
\quad t\geq1.
\]
Thus, \(\wv^t\) is measurable with respect to the round history
\(\mathcal F_t\) and is fixed before \(S_t\) is sampled.

We write \(\|\cdot\|_\infty\) for the sup norm on
\(\mathbb{R}^m\) and \(\|\cdot\|_1\) for the \(\ell_1\) norm.

Under Assumptions~\ref{as:slack} and~\ref{as:steps} and the
nonnegativity of the losses,
\[
c^t\in[c_{\min},c_{\max}].
\]
}

{
\begin{lemma}[Sharp calculus of the stabilized HV weight map]
\label{lem:weight-stab}
For \(a\in(0,\infty)^m\), define
\(V_i(a):=a_i^{-1}/\sum_j a_j^{-1}\) and
\(a_{\min}:=\min_i a_i\). Then
\begin{equation}
\label{eq:weight-derivative-bounds}
\sum_{i,j}\left|\frac{\partial V_i}{\partial a_j}\right|
\le\frac{2}{a_{\min}},
\qquad
\sum_{i,j,k}\left|\frac{\partial^2V_i}
{\partial a_j\partial a_k}\right|
\le\frac{8}{a_{\min}^2}.
\end{equation}
Consequently, for all \(z,z'\in\mathbb R^m\),
\begin{equation}
\label{eq:weight-map-global-sharp}
\|W_{\rho_{\mathrm{clip}}}(z;\bm r)
-W_{\rho_{\mathrm{clip}}}(z';\bm r)\|_1
\le\frac{2}{\rho_{\mathrm{clip}}}\|z-z'\|_\infty.
\end{equation}
Thus the notation used below may take
\(C(\rho_{\mathrm{clip}},\bm r):=2/\rho_{\mathrm{clip}}\), which is
dimension-free and independent of \(\bm r\). If the whole segment
\([z,z']\) has unclipped slacks at least
\(\widetilde\rho\ge\rho_{\mathrm{clip}}\), the sharper local constant
\(2/\widetilde\rho\) holds. If
\(\widetilde\rho>\rho_{\mathrm{clip}}\), then, with
\(J_W(z'):=\partial_zW_{\rho_{\mathrm{clip}}}(z';\bm r)\),
\begin{equation}
\label{eq:weight-map-second-order}
\|W(z)-W(z')-J_W(z')(z-z')\|_1
\le\frac{4}{\widetilde\rho^2}\|z-z'\|_\infty^2,
\qquad
\|J_W(z')\|_{\infty\to1}\le\frac{2}{\widetilde\rho},
\end{equation}
where \(W=W_{\rho_{\mathrm{clip}}}(\cdot;\bm r)\).
The proof is given in Appendix~\ref{app:proof-weight-stab}.
\end{lemma}
}

{
\subsection{Objective-Report Analysis}\label{app:theory-report}
\begin{assumption}[Objective-report sampling]
\label{as:objective-reports}
During round \(t\), after \(S_t\) is selected and at the common
pre-update model \(\theta^t\), each participating client \(k\) draws a
report index set
\[
\mathcal B_{k,t}^{\mathrm{rep}}
\subseteq\{1,\ldots,n_k\},
\qquad
\left|\mathcal B_{k,t}^{\mathrm{rep}}\right|
=b_{k,t}^{\mathrm{rep}},
\]
uniformly without replacement. Its report for objective \(i\) is the
minibatch average
\begin{equation}
\label{eq:objective-report-mechanism}
\widehat f_{i,k}^{\,t}
:=
\frac{1}{b_{k,t}^{\mathrm{rep}}}
\sum_{j\in\mathcal B_{k,t}^{\mathrm{rep}}}
\ell_i(\theta^t;x_{k,j}),
\qquad k\in S_t.
\end{equation}
Let \(q_t\) denote the configured number of report batches and let
\(b\) be the local minibatch size. The implementation uses
\[
b_{k,t}^{\mathrm{rep}}
=
\begin{cases}
n_k, & q_t=0,\\
\min\{q_t b,n_k\}, & q_t\geq1.
\end{cases}
\]
The notation matches the experimental budget as
\[
q_t=B_{\rm rep}\quad\text{for finite }B_{\rm rep},
\qquad
q_t=0\quad\Longleftrightarrow\quad B_{\rm rep}=\mathrm{all}.
\]
Thus \(q_t=0\) evaluates the complete client dataset, while
\(q_t\geq1\) takes the first \(q_t\) batches of a fresh random
permutation, which is equivalent to a uniformly sampled subset of
\(b_{k,t}^{\mathrm{rep}}\) examples. The report loader uses
\(\texttt{drop\_last=False}\).

Define the report error
\[
\nu_{i,k}^{\,t}
:=
\widehat f_{i,k}^{\,t}
-f_i^{(k)}(\theta^t),
\]
and the within-client finite-population dispersion
\[
v_{i,k}(\theta)
:=
\frac{1}{n_k}
\sum_{j=1}^{n_k}
\left(
\ell_i(\theta;x_{k,j})-f_i^{(k)}(\theta)
\right)^2.
\]
For \(n_k>1\), set
\[
\gamma_{k,t}
:=
\frac{n_k-b_{k,t}^{\mathrm{rep}}}
{b_{k,t}^{\mathrm{rep}}(n_k-1)},
\]
and set \(\gamma_{k,t}=0\) when \(n_k=1\). For the fixed target
horizon \(T\), assume the finite-second-moment bounds
\begin{equation}
\label{eq:report-dispersion}
\tau_{i,T}^2
:=
\max_{0\le t<T}
\mathbb E\left[
\frac1N\sum_{k=1}^N
\gamma_{k,t}v_{i,k}(\theta^t)
\right]
<\infty,
\end{equation}
and
\[
\omega_{i,T}^2
:=
\max_{0\le t<T}
\mathbb E\left[
\frac1N\sum_{k=1}^N
\bigl(f_i^{(k)}(\theta^t)-f_i(\theta^t)\bigr)^2
\right]
<\infty.
\]
For fixed \(n_k\), \(\gamma_{k,t}\) decreases as the report sample size
\(b_{k,t}^{\mathrm{rep}}\) increases. In particular, evaluating the
full local client dataset gives \(\gamma_{k,t}=0\) and therefore
removes the within-client report-sampling contribution
\(\tau_{i,T}^2\); client subsampling can still leave a nonzero
between-client contribution through \(\omega_{i,T}^2\).
The conditional zero-mean and variance of \(\nu_{i,k}^{\,t}\) are
derived in Lemma~\ref{lem:report-moments}, rather than postulated.

For distinct clients, we use the optional cross-client orthogonality
condition
\[
\mathbb E[\nu_{i,k}^{\,t}\nu_{i,\ell}^{\,t}
\mid\mathcal F_t,S_t]=0,
\qquad k\neq\ell,
\]
only when deriving the usual \(1/M\) reduction for averaged report
noise. The main objective-report bound remains valid without this
condition by Jensen's inequality. Report randomness is generated
before local SGD and is included in \(\mathcal U_t^{\mathrm{rep}}\)
within the local-step filtration. The report and the current client
update may share the same client data; no independence between them is
used in the report-error proof.
The server estimate is
\[
\widehat f_i^{\,t}
:=
\frac1M\sum_{k\in S_t}\widehat f_{i,k}^{\,t},
\qquad
\widehat{\bm F}^{\,t}
:=
(\widehat f_1^{\,t},\ldots,\widehat f_m^{\,t})^\top .
\]
The estimate is used to construct \(\wv^{t+1}\), rather than the weight
vector used by the same cohort \(S_t\).
\end{assumption}

{
\begin{lemma}[Finite-population report moments]
\label{lem:report-moments}
Under Assumption~\ref{as:objective-reports}, for every objective \(i\)
and participating client \(k\),
\[
\mathbb E\left[
\left.
\widehat f_{i,k}^{\,t}
\right|\mathcal F_t,S_t
\right]
=
f_i^{(k)}(\theta^t),
\]
and, for \(n_k>1\),
\[
\mathbb E\left[
\left.
(\nu_{i,k}^{\,t})^2
\right|\mathcal F_t,S_t
\right]
=
\gamma_{k,t}v_{i,k}(\theta^t).
\]
For \(n_k=1\), the report is exact and both sides of the second
identity are zero.
\end{lemma}
}

\begin{lemma}[Finite-population objective-report error]
\label{lem:objective-report-error}
Under Assumptions~\ref{as:uniform-aggregation} and
\ref{as:objective-reports}, define
\[
\Xi_{\mathrm{loss},T}
:=
\max_{0\le t<T}
\mathbb E\left[
\left\|
\widehat{\bm F}^{\,t}-\bm F(\theta^t)
\right\|_\infty^2
\right].
\]
Without any cross-client orthogonality assumption,
\[
\Xi_{\mathrm{loss},T}
\leq
\sum_{i=1}^m
\left[
\chi_N(M)\omega_{i,T}^2
+\tau_{i,T}^2
\right].
\]
If, in addition, the cross-client report condition
\[
\mathbb E\left[
\left.\nu_{i,k}^{\,t}\nu_{i,\ell}^{\,t}
\right|\mathcal F_t,S_t\right]=0,
\qquad k\neq\ell,
\tag{CCO-report}
\]
holds, the final term improves from \(\tau_{i,T}^2\) to
\(\tau_{i,T}^2/M\). The theorem uses the first, assumption-free bound.
Thus participation affects report quality through
\(\chi_N(M)\omega_{i,T}^2\), while the report budget affects
\(\tau_{i,T}^2\). Full participation and full-client reports jointly
make the displayed upper bound zero.
\end{lemma}
}
{
\subsection{One-Round-Stale Weight Error}\label{app:theory-stale}
\begin{lemma}[One-round-stale weight discrepancy]
\label{lem:stale-weight-discrepancy}
For every round \(1\leq t<T\) under the fixed-horizon condition in
Assumption~\ref{as:slack}, the discrepancy between the stale
FedHV weight \(\wv^t\) used by the algorithm and the ideal current
hypervolume weight \(\wv^{\star,t}\) satisfies
\begin{align}
\left\|
\wv^t-\wv^{\star,t}
\right\|_1
\leq\;&
C(\rho_{\mathrm{clip}},\bm r)
\left\|
\widehat{\bm F}^{\,t-1}
-
\bm F(\theta^{t-1})
\right\|_\infty
\nonumber\\
&+
C(\rho_{\mathrm{clip}},\bm r)
\left\|
\bm F(\theta^{t-1})
-
\bm F(\theta^t)
\right\|_\infty .
\label{eq:stale_weight_decomposition}
\end{align}
Consequently,
\begin{align}
\mathbb{E}
\left[
\left\|
\wv^t-\wv^{\star,t}
\right\|_1^2
\right]
\leq\;&
2C(\rho_{\mathrm{clip}},\bm r)^2\Xi_{\mathrm{loss},T}
\nonumber\\
&+
2C(\rho_{\mathrm{clip}},\bm r)^2
\mathbb{E}
\left[
\left\|
\bm F(\theta^t)-\bm F(\theta^{t-1})
\right\|_\infty^2
\right].
\label{eq:stale_weight_squared}
\end{align}
For the initialization \(t=0\), both
\(\wv^0\) and \(\wv^{\star,0}\) belong to the simplex
\(\Delta_m\), and therefore
\[
\left\|
\wv^0-\wv^{\star,0}
\right\|_1
\leq2.
\]
\end{lemma}
}
{
\subsection{Explicit Convergence Theorem and Rate}\label{app:theory-convergence}
\begin{theorem}[Explicit form of Theorem~\ref{thm:fedhv}]
\label{thm:fedhv-explicit}
Under Assumptions~\ref{as:uniform-aggregation},
\ref{as:smooth},
\ref{as:hetero}, \ref{as:slack},
\ref{as:steps}, and~\ref{as:objective-reports}, with client
participation rate \(p\in(0,1]\) and the deterministic target horizon
\(T\geq1\) fixed above. The fixed-horizon domain condition in
Assumption~\ref{as:slack} ensures that every point used in the proof
lies in \(\mathcal D_\rho\). The main theorem uses the Jensen-valid choice
\(\kappa_M=1\), so it does not require a cross-client orthogonality
assumption. For the fixed-horizon process,
\begin{align}
\frac1T\sum_{t=0}^{T-1}
\mathbb E\left[
\|\nabla\Phi(\theta^t)\|^2
\right]
\leq\;&
\frac{A_1\bigl(\Phi(\theta^0)-\Phi_{\rm lb}\bigr)}
{\eta_{\mathrm{eff}}T}
+A_2\eta_{\mathrm{eff}}\Sigma^2
\nonumber\\
&+
A_3\eta_{\mathrm{eff}}\mathcal H_N(p;\zeta)
+
A_4\eta_\ell^2(K-1)(2K-1)
\nonumber\\
&+
A_5\Xi_{\mathrm{loss},T}
+
A_6\eta_{\mathrm{eff}}^2\mathcal V
+
\frac{A_7}{T}.
\label{eq:fedhv_convergence_with_drift}
\end{align}
Here
\[
\eta_{\mathrm{eff}}:=\eta_g\eta_\ell K,
\qquad
\Sigma^2:=\sum_{i=1}^m\sigma_i^2,
\qquad
\zeta^2:=\sum_{i=1}^m\zeta_i^2,
\qquad
Q:=\zeta^2+\Sigma^2,
\]
\[
G_{\max}:=\max_{i\in[m]}G_i,
\qquad
C_\delta:=4,
\qquad
\bar G^2:=2\bigl(G_{\max}^2+C_\delta Q\bigr).
\]
The quantity
\begin{equation}
\label{eq:stale_variance_factor}
\mathcal V
:=
6G_{\max}^2
+
3\frac{\Sigma^2}{K}
+
6\mathcal H_N(p;\zeta)
+
\frac{L_w^2\bar G^2}{2}
\eta_\ell^2(K-1)(2K-1)
\end{equation}
controls the inter-round model displacement that generates the
one-round weight-staleness error. The lower-bound term is explicit:
\[
\Phi_{\rm lb}:=-\sum_{i=1}^m\log r_i.
\]
The fourth term
captures the error caused by multiple local steps and vanishes when
\(K=1\). The fifth term captures error in the client-reported
objective estimates. The sixth term is the additional error caused by
using the one-round-stale weight vector instead of the ideal current
hypervolume weights. The final \(A_7/T\) term is the one-time
initialization contribution.
The constants in \eqref{eq:fedhv_convergence_with_drift} can be chosen
explicitly. Let
\[
\bar\eta_{\mathrm{eff}}
:=\frac{c_2c_{\min}^{\,2}}{c_{\max}L_\Phi},
\qquad
E_{\mathrm{wt}}
:=c_{\max}+6L_\Phi\bar\eta_{\mathrm{eff}},
\]
where \(c_1=1/4\) and \(c_2=1/48\) are fixed in
Assumption~\ref{as:steps}. With the sharp global constant
\(C(\rho_{\mathrm{clip}},\bm r)=2/\rho_{\mathrm{clip}}\) from
Lemma~\ref{lem:weight-stab}, one admissible choice is
\begin{align*}
A_1&:=4c_{\max},
&A_2&:=6c_{\max}L_\Phi,
&A_3&:=12c_{\max}L_\Phi,\\
A_4&:=c_{\max}L_w^2\bar G^2
       \bigl(c_{\max}+L_\Phi\bar\eta_{\mathrm{eff}}\bigr),
&A_5&:=8c_{\max}E_{\mathrm{wt}}G_{\max}^2C(\rho_{\mathrm{clip}},\bm r)^2,\\
A_6&:=8c_{\max}E_{\mathrm{wt}}G_{\max}^4C(\rho_{\mathrm{clip}},\bm r)^2,
&A_7&:=16c_{\max}E_{\mathrm{wt}}G_{\max}^2.
\end{align*}
Thus \(A_1,A_2,A_3\) depend only on the displayed smoothness and
slack constants; \(A_4\) additionally depends on \(L_w,\bar G\); and
\(A_5,A_6,A_7\) additionally depend on \(G_{\max}\) and the explicit
weight-map Lipschitz constant \(C(\rho_{\mathrm{clip}},\bm r)\). None of the \(A_j\)
depends on \(T,N,p,\eta_\ell,\eta_g\), or \(K\); dependence on
\(\Sigma,\zeta\) occurs only through the displayed quantities
\(\mathcal H_N,\mathcal V\), and \(\bar G\).
The bound is a fixed-horizon ergodic
stationarity-to-neighborhood guarantee. For a family of target
horizons, the report contribution is governed by the corresponding
sequence \(\Xi_{\mathrm{loss},T}\); no asymptotic conclusion is claimed
without controlling that sequence. Partial participation and finite
report minibatches are two mechanisms that can keep
\(\Xi_{\mathrm{loss},T}\) positive.
The proof is provided in Appendix~\ref{app:proof-convergence}; it
telescopes the ordinary fixed-horizon increments. Consequently, the
displayed average also bounds
\[
\min_{0\leq t<T}\mathbb E\|\nabla\Phi(\theta^t)\|^2.
\]
If the fixed-horizon condition in Assumption~\ref{as:slack} cannot be
verified for a chosen $T$,
this theorem is not asserted for that horizon.
\end{theorem}
}

{
\begin{remark}[Optional cross-client variance specializations]
\label{rem:optional-cross-client}
The main theorem uses only Jensen's inequality. A sharper \(1/M\)
factor is available only under an additional, formally stated
cross-client condition. Namely, if for every local step \(s\) and every
distinct \(k,\ell\in S_t\),
\[
\mathbb E\left[
\left.\left\langle\xi_{k,s}^t,\xi_{\ell,s}^t\right\rangle
\right|\mathcal G_{t,s}\right]=0,
\tag{CCO-grad}
\]
then the client-averaged stochastic-gradient term may be sharpened by
replacing \(\Sigma^2\) with \(\Sigma^2/M\) in the stochastic terms of
Theorem~\ref{thm:fedhv-explicit} and in \(\mathcal V\). Likewise, if the
condition (CCO-report) stated in
Lemma~\ref{lem:objective-report-error} holds, the report-noise
contribution in the bound for \(\Xi_{\mathrm{loss},T}\) improves from
\(\tau_{i,T}^2\) to \(\tau_{i,T}^2/M\). These are optional refinements and are
not part of the main guarantee.
\end{remark}
}

{
\begin{corollary}[Rate and stationarity neighborhood]
\label{cor:fedhv-rate-explicit}
Under the conditions of Theorem~\ref{thm:fedhv-explicit}, fix a deterministic
target horizon \(T\geq1\), set
\(\eta_g=\bar\eta_g>0\), and choose
\[
\eta_\ell=\frac{\eta_0}{\bar\eta_g K\sqrt{T}},
\]
where
\[
0<\eta_0
\leq
\min\left\{
\frac{\bar\eta_g c_1}{L_w},
\bar\eta_{\mathrm{eff}}
\right\}.
\]
Then
\(\eta_{\mathrm{eff}}=\eta_0/\sqrt{T}\), and
\begin{align}
\frac1T\sum_{t=0}^{T-1}
\mathbb E\left[
\|\nabla\Phi(\theta^t)\|^2
\right]
\leq\;&
\frac{A_1(\Phi(\theta^0)-\Phi_{\rm lb})}
{\eta_0\sqrt{T}}
+\frac{A_2\eta_0\Sigma^2}{\sqrt{T}}
\nonumber\\
&+\frac{A_3\eta_0\mathcal H_N(p;\zeta)}{\sqrt{T}}
+\frac{A_4\eta_0^2(K-1)(2K-1)}
{\bar\eta_g^{\,2}K^2T}
\nonumber\\
&+A_5\Xi_{\mathrm{loss},T}
+\frac{A_6\eta_0^2}{T}\mathcal V_T
+\frac{A_7}{T},
\label{eq:fedhv_rate_with_drift}
\end{align}
where
\[
\mathcal V_T
:=
6G_{\max}^2
+3\frac{\Sigma^2}{K}
+6\mathcal H_N(p;\zeta)
+\frac{L_w^2\bar G^2}{2}
\frac{\eta_0^2}{\bar\eta_g^{\,2}K^2T}
(K-1)(2K-1).
 \]
Consequently, the displayed ergodic bound implies
\[
\min_{0\leq t<T}\mathbb E\|\nabla\Phi(\theta^t)\|^2
\leq
\mathcal O(T^{-1/2})+A_5\Xi_{\mathrm{loss},T}.
\]
This statement should be read as a family of horizon-\(T\) guarantees,
because the chosen \(\eta_\ell\) depends on the target horizon. If
\(\Xi_{\mathrm{loss},T}=0\) for a given horizon, the stochastic
optimization terms scale as \(\mathcal O(T^{-1/2})\); more generally,
the report term determines the stationarity neighborhood. The
experiments use fixed configured stepsizes rather than this theoretical
horizon-dependent schedule.
\end{corollary}
}

\subsection{Pareto Stationarity and Special Cases}\label{app:theory-pareto}
\begin{corollary}[Quantitative approximate Pareto stationarity]
\label{cor:approx-pareto}
Define the Pareto-stationarity residual
\begin{equation}
\label{eq:pareto_stationarity_residual}
\mathcal{P}(\theta)
:=
\min_{\lambda\in\Delta_m}
\left\|
\sum_{i=1}^{m}
\lambda_i\nabla f_i(\theta)
\right\|.
\end{equation}
Then, for every \(\theta\),
\begin{equation}
\label{eq:pareto_residual_hv_bound}
\mathcal{P}(\theta)
\le
\frac{\|\nabla\Phi(\theta)\|}{c(\theta)}
\le
\frac{\|\nabla\Phi(\theta)\|}{c_{\min}},
\end{equation}
where
\[
c(\theta)
=
\sum_{i=1}^{m}
\frac{1}{r_i-f_i(\theta)}.
\]
Consequently, on the fixed horizon \(t=0,\ldots,T-1\),
\begin{equation}
\label{eq:expected_pareto_residual}
\frac1T\sum_{t=0}^{T-1}
\mathbb{E}\left[
\mathcal{P}(\theta^t)^2
\right]
\leq
\frac{1}{c_{\min}^2T}
\sum_{t=0}^{T-1}
\mathbb{E}\left[
\|\nabla\Phi(\theta^t)\|^2
\right].
\end{equation}
Therefore,
\[
\min_{0\leq t<T}\mathbb{E}\mathcal{P}(\theta^t)^2
\leq
\frac{1}{c_{\min}^2}
\min_{0\leq t<T}\mathbb{E}\|\nabla\Phi(\theta^t)\|^2.
\]
Thus, Theorem~\ref{thm:fedhv-explicit} also provides a quantitative bound on
the Pareto-stationarity residual. The proof is given in
Appendix~\ref{app:proof-approx-pareto}.
\end{corollary}

\begin{corollary}[Single-objective specialization]
\label{cor:single-objective}
When \(m=1\), the FedHV weight satisfies
\[
w_1^t=1
\]
at every communication round, independently of the
objective-value estimate. Consequently, FedHV coincides
with single-objective local SGD applied to \(f_1\).

Moreover,
\[
\nabla\Phi(\theta)
=
\frac{1}{r_1-f_1(\theta)}
\nabla f_1(\theta),
\]
and, under Assumption~\ref{as:slack} and nonnegative
losses,
\[
\|\nabla f_1(\theta)\|
\le
r_1\|\nabla\Phi(\theta)\|.
\]
A direct specialization of the proof of
Theorem~\ref{thm:fedhv-explicit} to \(m=1\) eliminates the
weight-estimation term, since \(w_1^t\equiv1\), and yields
a standard local-SGD stationarity bound.
\end{corollary}


\subsection{Reference-Dependent Nash-Welfare Surrogate and Local Proportional-Fairness Interpretation}
\label{sec:fairness-analysis}

The log-hypervolume objective also provides a task-balancing
interpretation. Recall that
\[
s_i(\theta)=r_i-f_i(\theta)>0
\]
is the reference-relative slack of objective \(i\). Minimizing
the negative log-hypervolume is equivalent to maximizing
\begin{equation}
\label{eq:nash_welfare_objective}
\Psi(\theta)
:=
-\Phi(\theta)
=
\sum_{i=1}^{m}\log s_i(\theta)
=
\log\!\left(
\prod_{i=1}^{m}s_i(\theta)
\right).
\end{equation}
Equation~\eqref{eq:nash_welfare_objective} is algebraically a log-product
criterion over the reference-relative task slacks
\citep{nash1950bargaining,navon2022bargaining}. We therefore regard
\(\Psi\) as a \emph{reference-dependent Nash-welfare surrogate in slack
space}, rather than as a claim that FedHV solves a classical global
Nash bargaining problem. In particular, the attainable slack set
induced by a neural network is generally nonconvex, so global
Nash-bargaining or global proportional-fairness conclusions do not
follow from our stationarity analysis.

The marginal utility of objective \(i\) under this surrogate is
\(1/s_i(\theta)\), so an objective with a smaller reference-relative
slack receives greater local importance. This motivates a \emph{local
first-order} connection to proportional fairness
\citep{kelly1997charging}. The interpretation is reference dependent and
concerns balance among task objectives; it should not be interpreted as
a guarantee of client-level, demographic, or societal fairness. For
context, the classical global proportional-fairness definition and the
local first-order implication used here are given in
Appendix~\ref{app:fairness-proofs}.

To connect this interpretation to the convergence analysis,
let \(\Theta\) denote the feasible parameter set and let
\(T_\Theta(\theta)\) be its tangent cone at \(\theta\). In our
unconstrained setting, \(\Theta=\mathbb{R}^d\).
We use
\[
\Theta_+
:=
\left\{\theta\in\Theta:
r_i-f_i(\theta)>0,\ \forall i\in[m]\right\}
\]
for the domain on which the reference-dependent Nash-welfare surrogate is defined.

\begin{definition}[Local proportional-fairness violation]
\label{def:pfgap}
For \(\theta\in\Theta_+\), define
\begin{equation}
\label{eq:pfgap_definition}
\mathrm{PFgap}(\theta)
:=
\sup_{\substack{
d\in T_\Theta(\theta)\\
\|d\|\le 1
}}
\left\langle
\nabla\Psi(\theta),d
\right\rangle.
\end{equation}
This quantity measures the largest first-order increase in
aggregate relative slack over feasible unit-norm directions.
\end{definition}

\begin{lemma}[Stationarity controls local proportional-fairness violation]
\label{lem:pfgap_bound}
For every \(\theta\in\Theta_+\),
\begin{equation}
\label{eq:pfgap_stationarity_bound}
\mathrm{PFgap}(\theta)
\le
\|\nabla\Psi(\theta)\|
=
\|\nabla\Phi(\theta)\|.
\end{equation}
If \(\Theta=\mathbb{R}^d\), then equality holds.
The proof is given in
Appendix~\ref{app:proof-pfgap}.
\end{lemma}

\begin{corollary}[FedHV controls local proportional-fairness violation]
\label{cor:fedhv_pf}
Under the fixed-horizon condition in Assumption~\ref{as:slack}, let
\[
t^\star
\in
\arg\min_{0\le t<T}
\mathbb{E}
\left[
\|\nabla\Phi(\theta^t)\|^2
\right].
\]
Then
\begin{equation}
\label{eq:fedhv_pfgap_bound}
\mathbb{E}
\left[
\mathrm{PFgap}(\theta^{t^\star})
\right]
\le
\left(
\min_{0\le t<T}
\mathbb{E}
\left[
\|\nabla\Phi(\theta^t)\|^2
\right]
\right)^{1/2}.
\end{equation}
Consequently, Theorem~\ref{thm:fedhv-explicit} directly yields a
quantitative bound on the largest feasible local first-order
improvement in aggregate relative slack; it is not a global
proportional-fairness guarantee. The proof is given
in Appendix~\ref{app:proof-fedhv-pf}.
\end{corollary}

Therefore, near-stationary FedHV iterates have a small local
proportional-fairness violation in task-slack space, in the first-order
sense quantified by \(\mathrm{PFgap}\): no feasible unit-norm direction can
substantially increase the reference-dependent logarithmic
Nash-welfare surrogate. This is a local stationarity-based
interpretation, not a claim of a classical global Nash bargaining
solution. The inverse-slack weights provide the corresponding
optimization mechanism, since objectives closer to the reference
boundary have larger coefficients \(1/s_i(\theta)\).

\subsection{Proofs and Auxiliary Lemmas}
\label{app:theory-proofs}

\subsubsection{Proof of Lemma~\ref{lem:hv-pareto}}
\label{app:proof-hv-pareto}
\begin{proof}
Let \(s_i(\theta)=r_i-f_i(\theta)\) and define
\(u_i(\theta)=s_i(\theta)^{-1}\) wherever the slacks are positive.
In particular, assume \(s_i(\theta^\star)>0\) for every \(i\), so that
\(u_i(\theta^\star)>0\).
By direct differentiation of $\Phi(\theta)=-\sum_{i=1}^m \log s_i(\theta)$,
\begin{equation}
\label{eq:grad-phi}
\begin{aligned}
\nabla \Phi(\theta)
&= \sum_{i=1}^m \frac{1}{s_i(\theta)}\,\nabla f_i(\theta) \\[-0.4em]
&= \sum_{i=1}^m u_i(\theta)\,\nabla f_i(\theta).
\end{aligned}
\end{equation}
Evaluate at $\theta^\star$ and suppose $\nabla\Phi(\theta^\star)=\mathbf{0}$. Set
\begin{align}
\lambda_i
= \frac{u_i(\theta^\star)}{\sum_{j=1}^m u_j(\theta^\star)}
\quad\text{so that}\quad
\lambda_i>0,\;\; \sum_{i=1}^m \lambda_i=1.
\end{align}
Then
\begin{equation}
\begin{aligned}
\sum_{i=1}^m \lambda_i \,\nabla f_i(\theta^\star)
&= \frac{1}{\sum_{j=1}^m u_j(\theta^\star)}
   \sum_{i=1}^m u_i(\theta^\star)\,\nabla f_i(\theta^\star)
\\[-0.25em]
&= \frac{1}{\sum_{j=1}^m u_j(\theta^\star)}\,\nabla\Phi(\theta^\star)
= \mathbf{0}.
\end{aligned}
\end{equation}
Hence $\mathbf{0}$ lies in the convex hull of $\{\nabla f_i(\theta^\star)\}_{i=1}^m$, i.e.,
$\theta^\star$ satisfies the first-order Pareto stationarity condition.
\end{proof}

\subsubsection{Proof of Lemma~\ref{lem:self-bounded-gradients}}
\label{app:proof-self-bounded-gradients}
\begin{proof}
Because the local objectives are nonnegative, their uniform average
\(f_i\) is nonnegative. For \(L_i>0\), the descent lemma for an
\(L_i\)-smooth function, applied at
\(\theta-(1/L_i)\nabla f_i(\theta)\), gives
\[
f_i\left(\theta-\frac1{L_i}\nabla f_i(\theta)\right)
\leq
f_i(\theta)-\frac{1}{2L_i}\|\nabla f_i(\theta)\|^2.
\]
The left-hand side is nonnegative, so
\(\|\nabla f_i(\theta)\|^2\leq2L_i f_i(\theta)\). If \(L_i=0\), smoothness
and nonnegativity imply that \(f_i\) is constant and the same inequality
holds trivially. On
\(\mathcal D_\rho\), the definition of the safe subset and
nonnegativity imply
\(f_i(\theta)\leq r_i-\rho\), which proves the stated bound.
Finally, Jensen's inequality gives
\[
\left\|\sum_iw_i\nabla f_i(\theta)\right\|
\leq
\sum_iw_iG_i
\leq G_{\max}.
\]
\end{proof}

\subsubsection{Proof of Lemma~\ref{lem:smooth-phi}}
\label{app:proof-smooth-phi}
\begin{proof}
Let \(\theta,\theta'\in\mathcal D\) and assume that
\([\theta,\theta']\subseteq\mathcal D_\rho\). Write
\(u_i(\theta)=1/(r_i-f_i(\theta))\). Since the entire segment is in
\(\mathcal D_\rho\), the definition of the safe subset gives
\(u_i(\theta)\leq1/\rho\) and \(u_i(\theta')\leq1/\rho\). Therefore,
\begin{align}
\|\nabla\Phi(\theta)-\nabla\Phi(\theta')\|
&\leq
\sum_{i=1}^m
u_i(\theta)
\|\nabla f_i(\theta)-\nabla f_i(\theta')\|
\nonumber\\
&\quad+
\sum_{i=1}^m
|u_i(\theta)-u_i(\theta')|
\|\nabla f_i(\theta')\|.
\end{align}
The first term is bounded by
\(\sum_i L_i\|\theta-\theta'\|/\rho\). Moreover, because
\(x\mapsto1/x\) is \(1/\rho^2\)-Lipschitz on
\([\rho,\infty)\),
\[
|u_i(\theta)-u_i(\theta')|
\leq
\frac{|f_i(\theta)-f_i(\theta')|}{\rho^2}
\leq
\frac{G_i}{\rho^2}\|\theta-\theta'\|,
\]
where the last inequality follows from
Lemma~\ref{lem:self-bounded-gradients}. Applying the same lemma to
\(\|\nabla f_i(\theta')\|\) yields
\[
\|\nabla\Phi(\theta)-\nabla\Phi(\theta')\|
\leq
\left(
\sum_{i=1}^m\frac{L_i}{\rho}
\;+\sum_{i=1}^m\frac{G_i^2}{\rho^2}
\right)\|\theta-\theta'\|.
\]
Thus \(\Phi\) is \(L_\Phi\)-smooth without requiring twice
differentiability.
\end{proof}

\subsubsection{Proof of Lemma~\ref{lem:weight-stab}}
\label{app:proof-weight-stab}
\begin{proof}
Write \(a_i=a_i(z):=\max\{r_i-z_i,\rho_{\mathrm{clip}}\}\),
\(u_i:=a_i^{-1}\), and \(U:=\sum_j u_j\). For the smooth map
\(V_i(a)=u_i/U\), direct differentiation gives
\[
\frac{\partial V_i}{\partial a_j}
=-\frac{\mathbf 1\{i=j\}u_i^2}{U}
+\frac{u_iu_j^2}{U^2}.
\]
Using \(\sum_i u_i=U\) and
\(\sum_i u_i^2\le U/a_{\min}\),
\[
\sum_{i,j}\left|\frac{\partial V_i}{\partial a_j}\right|
\le\frac{\sum_i u_i^2}{U}
+\frac{(\sum_i u_i)(\sum_j u_j^2)}{U^2}
\le\frac{2}{a_{\min}}.
\]
Differentiating once more yields
\begin{align*}
\frac{\partial^2V_i}{\partial a_j\partial a_k}
={}&\frac{2\mathbf 1\{i=j=k\}u_i^3}{U}
-\frac{\mathbf 1\{i=j\}u_i^2u_k^2}{U^2}
-\frac{\mathbf 1\{i=k\}u_i^2u_j^2}{U^2}\\
&-\frac{2\mathbf 1\{j=k\}u_iu_j^3}{U^2}
+\frac{2u_iu_j^2u_k^2}{U^3}.
\end{align*}
Since \(\sum_i u_i^3\le U/a_{\min}^2\), summing the
absolute values of these five groups gives respectively
\(2,1,1,2,2\) times \(a_{\min}^{-2}\). This proves
\eqref{eq:weight-derivative-bounds}.

The clipped-slack map \(z\mapsto a(z)\) is 1-Lipschitz in
\(\ell_\infty\). Every point on the segment between \(a(z)\) and
\(a(z')\) has minimum coordinate at least
\(\rho_{\mathrm{clip}}\). Applying the mean-value inequality to
\(V\) on this segment and using the first bound above proves
\eqref{eq:weight-map-global-sharp}. If clipping is inactive on the
segment \([z,z']\) and all its slacks are at least
\(\widetilde\rho\), then \(a(z)=\bm r-z\) is affine there and the
same argument gives the local constant \(2/\widetilde\rho\).

Under the strict inequality
\(\widetilde\rho>\rho_{\mathrm{clip}}\), the map is twice
differentiable near that segment. Applying scalar Taylor's theorem to
each coordinate along \(z'+\lambda(z-z')\), summing absolute values,
and using the second derivative bound in
\eqref{eq:weight-derivative-bounds} gives
\[
\frac12\frac{8}{\widetilde\rho^2}
\|z-z'\|_\infty^2,
\]
which is \eqref{eq:weight-map-second-order}; its Jacobian bound follows
from the first derivative estimate.
\end{proof}

{
\subsubsection{Proof of Lemma~\ref{lem:report-moments}}
\label{app:proof-report-moments}
\begin{proof}
Fix \(i,k,t\), condition on \(\mathcal F_t,S_t\), and write
\[
y_j:=\ell_i(\theta^t;x_{k,j}),
\qquad
\bar y:=f_i^{(k)}(\theta^t),
\qquad
a_j:=y_j-\bar y.
\]
Then \(\sum_{j=1}^{n_k}a_j=0\). Let
\(I_j=\mathbf 1\{j\in\mathcal B_{k,t}^{\mathrm{rep}}\}\). For a
uniform subset of size \(b_{k,t}^{\mathrm{rep}}\), when \(n_k>1\),
\[
\mathbb E[I_j\mid\mathcal F_t,S_t]
=\frac{b_{k,t}^{\mathrm{rep}}}{n_k},
\qquad
\mathbb E[I_jI_\ell\mid\mathcal F_t,S_t]
=\frac{b_{k,t}^{\mathrm{rep}}
(b_{k,t}^{\mathrm{rep}}-1)}
{n_k(n_k-1)}
\quad (j\neq\ell).
\]
Therefore,
\[
\widehat f_{i,k}^{\,t}-\bar y
=
\frac{1}{b_{k,t}^{\mathrm{rep}}}
\sum_{j=1}^{n_k}I_ja_j,
\]
which gives
\[
\mathbb E\left[
\left.
\widehat f_{i,k}^{\,t}-\bar y
\right|\mathcal F_t,S_t
\right]
=0.
\]
Using
\(\sum_{j\neq\ell}a_ja_\ell=-\sum_j a_j^2\), its conditional
second moment is
\begin{align}
\mathbb E\left[
\left.
(\widehat f_{i,k}^{\,t}-\bar y)^2
\right|\mathcal F_t,S_t
\right]
&=
\frac{1}{(b_{k,t}^{\mathrm{rep}})^2}
\left[
\frac{b_{k,t}^{\mathrm{rep}}}{n_k}
-\frac{b_{k,t}^{\mathrm{rep}}
(b_{k,t}^{\mathrm{rep}}-1)}
{n_k(n_k-1)}
\right]
\sum_{j=1}^{n_k}a_j^2
\nonumber\\
&=
\frac{n_k-b_{k,t}^{\mathrm{rep}}}
{b_{k,t}^{\mathrm{rep}}(n_k-1)}
v_{i,k}(\theta^t).
\end{align}
Since \(\nu_{i,k}^{\,t}=\widehat f_{i,k}^{\,t}-\bar y\), this proves
the lemma. If \(n_k=1\), necessarily \(b_{k,t}^{\mathrm{rep}}=1\)
and the report is exact.
\end{proof}
}
{
\subsubsection{Proof of Lemma~\ref{lem:objective-report-error}}
\label{app:proof-objective-report-error}
\begin{proof}
Fix an objective \(i\) and condition first on \(\mathcal F_t\). Write
\[
a_{i,k}^t:=f_i^{(k)}(\theta^t)-f_i(\theta^t).
\]
Then \(\sum_{k=1}^N a_{i,k}^t=0\), and
\[
\widehat f_i^{\,t}-f_i(\theta^t)
=
\frac1M\sum_{k\in S_t}a_{i,k}^t
\;+\frac1M\sum_{k\in S_t}\nu_{i,k}^{\,t}.
\]
The first term is the mean of a finite population sampled without
replacement. Hence,
\[
\mathbb E\left[
\left.
\left(\frac1M\sum_{k\in S_t}a_{i,k}^t\right)^2
\right|\mathcal F_t
\right]
=
\frac{N-M}{M(N-1)}
\frac1N\sum_{k=1}^N(a_{i,k}^t)^2.
\]
Define
\[
R_i^t:=\frac1M\sum_{k\in S_t}\nu_{i,k}^{\,t}.
\]
By Lemma~\ref{lem:report-moments}, the conditional report means are
zero. Therefore the cross term between the client-sampling term and
\(R_i^t\) vanishes after conditioning on \(\mathcal F_t,S_t\).
Under the optional cross-client orthogonality condition,
\[
\mathbb E\left[
\left.
(R_i^t)^2
\right|\mathcal F_t,S_t
\right]
=
\frac{1}{M^2}
\sum_{k\in S_t}
\gamma_{k,t}v_{i,k}(\theta^t).
\]
Taking expectations over the uniformly sampled \(S_t\) and using the
definition of \(\tau_{i,T}^2\) gives
\[
\mathbb E[(R_i^t)^2]\leq\frac{\tau_{i,T}^2}{M}.
\]
Without cross-client orthogonality, Jensen's inequality instead gives
\[
\mathbb E[(R_i^t)^2]
\leq
\mathbb E\left[
\frac1M\sum_{k\in S_t}
\gamma_{k,t}v_{i,k}(\theta^t)
\right]
\leq\tau_{i,T}^2.
\]
Combining these bounds with the finite-population client-sampling
identity, summing over \(i\), and using
\(\|\bm z\|_\infty^2\leq\|\bm z\|_2^2\) proves the claim. The local
model update does not enter this calculation; hence no independence
between the report and that update is used.
\end{proof}
}
{
\subsubsection{Proof of Lemma~\ref{lem:stale-weight-discrepancy}}
\label{app:proof-stale-weight}

\begin{proof}
For \(1\leq t<T_\rho\),
\[
\wv^t
=
W_{\rho_{\mathrm{clip}}}
\left(
\widehat{\bm F}^{\,t-1};
\bm r
\right),
\qquad
\wv^{\star,t}
=
W_{\rho_{\mathrm{clip}}}
\left(
\bm F(\theta^t);
\bm r
\right).
\]
Adding and subtracting
\(W_{\rho_{\mathrm{clip}}}(\bm F(\theta^{t-1});\bm r)\)
and applying the triangle inequality gives
\begin{align}
\left\|
\wv^t-\wv^{\star,t}
\right\|_1
\leq\;&
\left\|
W_{\rho_{\mathrm{clip}}}
\left(
\widehat{\bm F}^{\,t-1};
\bm r
\right)
-
W_{\rho_{\mathrm{clip}}}
\left(
 \bm F(\theta^{t-1});
\bm r
\right)
\right\|_1
\nonumber\\
&+
\left\|
W_{\rho_{\mathrm{clip}}}
\left(
 \bm F(\theta^{t-1});
\bm r
\right)
-
W_{\rho_{\mathrm{clip}}}
\left(
 \bm F(\theta^t);
\bm r
\right)
\right\|_1 .
\end{align}

Applying Lemma~\ref{lem:weight-stab} to both terms yields
\begin{align}
\left\|
\wv^t-\wv^{\star,t}
\right\|_1
\leq\;&
C(\rho_{\mathrm{clip}},\bm r)
\left\|
\widehat{\bm F}^{\,t-1}
-
 \bm F(\theta^{t-1})
\right\|_\infty
\nonumber\\
&+
C(\rho_{\mathrm{clip}},\bm r)
\left\|
 \bm F(\theta^{t-1})
-
 \bm F(\theta^t)
\right\|_\infty.
\end{align}

Using
\[
(a+b)^2
\le
2a^2+2b^2
\]
and Lemma~\ref{lem:objective-report-error} gives
\begin{align}
\mathbb{E}
\left[
\left\|
\wv^t-\wv^{\star,t}
\right\|_1^2
\right]
\leq\;&
2C(\rho_{\mathrm{clip}},\bm r)^2
\Xi_{\mathrm{loss},T}
\nonumber\\
&+
2C(\rho_{\mathrm{clip}},\bm r)^2
\mathbb{E}
\left[
\left\|
 \bm F(\theta^t)-\bm F(\theta^{t-1})
\right\|_\infty^2
\right].
\end{align}

Finally, for \(t=0\),
\(\wv^0,\wv^{\star,0}\in\Delta_m\), so
\[
\left\|
\wv^0-\wv^{\star,0}
\right\|_1
\le
2.
\]
\end{proof}
}
\subsubsection{Proof of Corollary~\ref{cor:approx-pareto}}
\label{app:proof-approx-pareto}

The inverse-slack weights
\[
w_i^\star(\theta)
=
\frac{
\bigl(r_i-f_i(\theta)\bigr)^{-1}
}{
c(\theta)
},
\qquad
c(\theta)
=
\sum_{j=1}^{m}
\frac{1}{r_j-f_j(\theta)},
\]
belong to \(\Delta_m\). Hence, they form a feasible choice
in the minimization defining \(\mathcal{P}(\theta)\), and
therefore
\begin{align}
\mathcal{P}(\theta)
&\le
\left\|
\sum_{i=1}^{m}
w_i^\star(\theta)\nabla f_i(\theta)
\right\|
\nonumber\\
&=
\frac{1}{c(\theta)}
\left\|
\nabla\Phi(\theta)
\right\|.
\end{align}

By Assumption~\ref{as:slack} and the nonnegativity of the
objectives,
\[
r_i-f_i(\theta)
\le
r_i,
\]
and consequently
\[
c(\theta)
=
\sum_{i=1}^{m}
\frac{1}{r_i-f_i(\theta)}
\ge
\sum_{i=1}^{m}
\frac{1}{r_i}
=
c_{\min}.
\]
Thus,
\[
\mathcal{P}(\theta)
\le
\frac{\|\nabla\Phi(\theta)\|}{c(\theta)}
\le
\frac{\|\nabla\Phi(\theta)\|}{c_{\min}}.
\]
Squaring this inequality, taking expectations, and summing over
\(t=0,\ldots,T-1\) proves
\eqref{eq:expected_pareto_residual}. Averaging followed by the minimum
over \(0\leq t<T\) gives its second displayed consequence.
\hfill\(\square\)

\subsubsection{Proof of Theorem~\ref{thm:fedhv-explicit}}
\label{app:proof-convergence}

We prove Theorem~\ref{thm:fedhv-explicit} by following the standard
bias--variance--drift analysis of federated local SGD, adapted to the
log--hypervolume objective $\Phi$ and hypervolume weights.
Because the theorem uses the Jensen-valid choice \(\kappa_M=1\), we
instantiate the client-averaged noise bound with \(\kappa_M=1\)
throughout this proof. The symbol \(\kappa_M\) is retained only in the
separate auxiliary statements that record the optional orthogonality
refinement.
By Assumption~\ref{as:slack}, every point used in the following
argument lies in \(\mathcal D_\rho\) for each
\(t=0,\ldots,T-1\). Thus all one-round inequalities below are applied
directly to the algorithmic iterates.

{
\paragraph{Per-round notation and ideal colinearity.}
Fix a communication round \(t\). Recall the true slack,
inverse slack, and normalizing factor
\[
s_i^t
=
r_i-f_i(\theta^t),
\qquad
u_i^t
=
(s_i^t)^{-1},
\qquad
c^t
=
\sum_{j=1}^m u_j^t.
\]
The ideal current hypervolume weights are
\[
w_i^{\star,t}
=
\frac{u_i^t}{c^t},
\qquad
\wv^{\star,t}\in\Delta_m,
\]
and the corresponding global scalarized objective is
\[
L_{\wv^{\star,t}}(\theta)
=
\sum_{i=1}^m
w_i^{\star,t}f_i(\theta).
\]

In contrast, the actual weight vector used by the algorithm in round
\(t\) is the one-round-stale vector \(\wv^t\), which is
\(\mathcal F_t\)-measurable and fixed before the participating client
set \(S_t\) is sampled.

By direct differentiation of
\[
\Phi(\theta)
=
\sum_{i=1}^m
-\log\!\bigl(r_i-f_i(\theta)\bigr),
\]
we have
\[
\nabla\Phi(\theta)
=
\sum_{i=1}^m
u_i(\theta)\nabla f_i(\theta).
\]
Hence, at \(\theta^t\),
\begin{equation}
\label{eq:phi-lw-colinear-app}
\nabla\Phi(\theta^t)
=
c^t
\nabla L_{\wv^{\star,t}}(\theta^t).
\end{equation}
Thus, the log-hypervolume gradient is colinear with the gradient of
the scalarized objective formed using the \emph{ideal current}
hypervolume weights \(\wv^{\star,t}\), not necessarily with the stale
weights \(\wv^t\) used by the algorithm.
}
\paragraph{Bounds on the scaling $c^t$.}
Under the safe--reference assumption (Assumption~\ref{as:slack}),
\[
s_i^t \;=\; r_i - f_i(\theta^t) \;\ge\; \rho > 0
\quad\Rightarrow\quad
u_i^t \;\le\; \frac{1}{\rho}.
\]
Assuming nonnegative losses $f_i(\theta) \ge 0$ along the trajectory, we also have $s_i^t \le r_i$ and hence
\[
u_i^t \;=\; \frac{1}{s_i^t} \;\ge\; \frac{1}{r_i}.
\]
Summing over $i$,
\begin{equation}
\label{eq:c-bounds}
c_{\min}
\;\le\;
c^t
\;\le\;
c_{\max},
\qquad
c_{\min}
\;:=\;
\sum_{i=1}^m \frac{1}{r_i}
\;\ge\;
\frac{m}{\max_i r_i},
\qquad
c_{\max}
\;:=\;
\frac{m}{\rho}.
\end{equation}
{
Combining \eqref{eq:phi-lw-colinear-app} and
\eqref{eq:c-bounds} gives
\begin{equation}
\label{eq:norm-phi-lw}
\left\|
\nabla L_{\wv^{\star,t}}(\theta^t)
\right\|^2
=
\frac{1}{(c^t)^2}
\left\|
\nabla\Phi(\theta^t)
\right\|^2
\ge
\frac{1}{c_{\max}^2}
\left\|
\nabla\Phi(\theta^t)
\right\|^2,
\end{equation}
and similarly
\[
\left\|
\nabla L_{\wv^{\star,t}}(\theta^t)
\right\|^2
\le
\frac{1}{c_{\min}^2}
\left\|
\nabla\Phi(\theta^t)
\right\|^2.
\]
}
{
\paragraph{Local objectives and stochastic gradients.}
For each client \(k\), define the actual stale-weight local objective
used in round \(t\) by
\[
L_{\wv^t}^{(k)}(\theta)
=
\sum_{i=1}^m
w_i^t f_i^{(k)}(\theta).
\]
The corresponding global scalarized objective is
\[
L_{\wv^t}(\theta)
=
\frac{1}{N}
\sum_{k=1}^N
L_{\wv^t}^{(k)}(\theta)
=
\sum_{i=1}^m
w_i^t f_i(\theta).
\]
All local trajectories in the analysis below are the actual trajectories
generated with the stale weight vector \(\wv^t\). The ideal weight vector
\(\wv^{\star,t}\) is used only to compare the actual direction with the
current log-hypervolume direction through the error \(e_t\).

Because \(\wv^t\in\Delta_m\), Assumption~\ref{as:smooth} implies
that every \(L_{\wv^t}^{(k)}\) is \(L_w\)-smooth, where we may use
the uniform constant
\[
L_w
:=
\sum_{i=1}^m L_i.
\]

For fixed round \(t\), write
\(\theta_k^s:=\theta_{k,s}^t\) for the actual local iterate and let
\[
g_{i,k}^s
:=
g_i^{(k)}
\bigl(\theta_k^s;B_{k,s}^t\bigr),
\]
and define the weighted local stochastic gradient
\[
g_k^s
:=
\sum_{i=1}^m
w_i^t g_{i,k}^s.
\]
Each participating client \(k\in S_t\) performs
\[
\theta_k^{s+1}
=
\theta_k^s-\eta_\ell g_k^s,
\qquad
\theta_k^0=\theta^t,
\qquad
s=0,\ldots,K-1.
\]
Since \(\wv^t\) is \(\mathcal F_t\)-measurable, the component-wise
oracle conditions in Assumption~\ref{as:smooth} and the definition of
\(\mathcal G_{t,s}\) give
\[
\mathbb{E}
\left[
g_k^s
\mid
\mathcal G_{t,s}
\right]
=
\nabla
L_{\wv^t}^{(k)}
(\theta_k^s).
\]
Moreover, using \(w_i^t\ge0\) and
\(\sum_iw_i^t=1\), we may use the uniform bound
\[
\mathbb{E}
\left[
\left\|
g_k^s
-
\nabla L_{\wv^t}^{(k)}(\theta_k^s)
\right\|^2
\mid
\mathcal G_{t,s}
\right]
\le
\Sigma^2,
\qquad
\Sigma^2
:=
\sum_{i=1}^m\sigma_i^2.
\]

The local model update is
\[
\Delta\theta_k^t
=
\theta_k^K-\theta^t
=
-\eta_\ell
\sum_{s=0}^{K-1}g_k^s,
\]
and the server forms the averaged update
\[
\Delta^t
:=
\frac{1}{|S_t|}
\sum_{k\in S_t}
\Delta\theta_k^t,
\qquad
\theta^{t+1}
=
\theta^t+\eta_g\Delta^t.
\]

\paragraph{Bias--variance--drift decomposition.}
Condition on the round history \(\mathcal F_t\). By construction,
\(\theta^t\) and \(\wv^t\) are then fixed, while \(S_t\) is sampled
uniformly without replacement from the \(N\) clients.

Using the tower rule and the conditional unbiasedness of the
minibatch gradients,
\begin{align}
\mathbb{E}
\left[
\Delta^t
\mid
\mathcal F_t
\right]
=
-\eta_\ell
\sum_{s=0}^{K-1}
\mathbb{E}
\left[
\left.
\frac{1}{|S_t|}
\sum_{k\in S_t}
\nabla
L_{\wv^t}^{(k)}
(\theta_k^s)
\right|
\mathcal F_t
\right].
\end{align}
Add and subtract
\(\nabla L_{\wv^t}^{(k)}(\theta^t)\) inside the sum to obtain
\begin{align}
\mathbb{E}
\left[
\Delta^t
\mid
\mathcal F_t
\right]
=
-\eta_\ell K
\mathbb{E}
\left[
\left.
\frac{1}{|S_t|}
\sum_{k\in S_t}
\nabla
L_{\wv^t}^{(k)}
(\theta^t)
\right|
\mathcal F_t
\right]
+b_t,
\label{eq:delta-decomp-fixed}
\end{align}
where
\begin{equation}
\label{eq:bt_definition}
b_t
:=
-\eta_\ell
\sum_{s=0}^{K-1}
\mathbb{E}
\left[
\left.
\frac{1}{|S_t|}
\sum_{k\in S_t}
\Big(
\nabla L_{\wv^t}^{(k)}(\theta_k^s)
-
\nabla L_{\wv^t}^{(k)}(\theta^t)
\Big)
\right|
\mathcal F_t
\right].
\end{equation}

Because \(\wv^t\) and \(\theta^t\) are fixed conditional on
\(\mathcal F_t\), uniform sampling without replacement gives
\begin{equation}
\label{eq:conditional_mean_update}
\mathbb{E}
\left[
\left.
\frac{1}{|S_t|}
\sum_{k\in S_t}
\nabla L_{\wv^t}^{(k)}(\theta^t)
\right|
\mathcal F_t
\right]
=
\frac{1}{N}
\sum_{k=1}^{N}
\nabla L_{\wv^t}^{(k)}(\theta^t)
=
\nabla L_{\wv^t}(\theta^t).
\end{equation}
}
\begin{lemma}[Virtual-trajectory deviation]
\label{lem:grad2-derived}
Under Assumptions~\ref{as:uniform-aggregation}, \ref{as:smooth},
\ref{as:hetero}, \ref{as:slack}, and~\ref{as:steps}, fix a
communication round \(t\in\{0,\ldots,T-1\}\) and condition on
\(\mathcal F_t\). Define the virtual trajectory
\begin{equation}
\label{eq:virtual-trajectory}
v_{t,0}:=\theta^t,
\qquad
v_{t,s+1}
:=
v_{t,s}
-
\eta_\ell
\nabla L_{\wv^t}(v_{t,s}),
\qquad
s=0,\ldots,K-1.
\end{equation}
The virtual trajectory is \(\mathcal F_t\)-measurable and uses the
same frozen weight vector as the actual local trajectories. Define
\[
\zeta^2:=\sum_{i=1}^m\zeta_i^2,
\qquad
Q:=\zeta^2+\Sigma^2,
\qquad
C_\delta:=4,
\qquad
\bar G^2:=2\bigl(G_{\max}^2+C_\delta Q\bigr).
\]

By Assumption~\ref{as:slack}, for every
\(k\in S_t\) and \(0\leq s\leq K\),
\[
\theta_{k,s}^t\in\mathcal D_\rho\subseteq\mathcal D,
\qquad
v_{t,s}\in\mathcal D_\rho\subseteq\mathcal D,
\]
and the connecting segments required by the proof are contained in
\(\mathcal D_\rho\). Consequently, the heterogeneity condition is
available at every virtual iterate, and the actual local trajectories
remain inside the analysis domain. The local stochastic
gradients are those generated by the fresh with-replacement minibatch
oracle in Assumption~\ref{as:smooth}; this is the oracle used at every
actual local iterate in the statement.

If \(\eta_\ell L_w K\leq1/4\), then, for every \(0\leq s\leq K\),
\begin{equation}
\label{eq:virtual-path-deviation}
D_{t,s}
:=
\mathbb{E}\left[
\left.
\frac{1}{M}
\sum_{k\in S_t}
\left\|\theta_{k,s}^t-v_{t,s}\right\|^2
\right|
\mathcal F_t
\right]
\leq
C_\delta\eta_\ell^2s^2Q.
\end{equation}
\end{lemma}

\begin{proof}
For brevity write
\[
h_k(\theta):=\nabla L_{\wv^t}^{(k)}(\theta),
\qquad
h(\theta):=\nabla L_{\wv^t}(\theta),
\qquad
\delta_{k,s}:=\theta_{k,s}^t-v_{t,s}.
\]
Also set
\[
a_{k,s}:=h_k(v_{t,s})-h(v_{t,s}),
\qquad
q_{k,s}:=h_k(\theta_{k,s}^t)-h_k(v_{t,s}).
\]
Then
\[
\delta_{k,s+1}
=
\delta_{k,s}
-
\eta_\ell
\bigl(q_{k,s}+a_{k,s}+\xi_{k,s}^t\bigr),
\qquad
\delta_{k,0}=0.
\]
By the fresh-minibatch oracle and the filtration
\(\mathcal G_{t,s}\),
\[
\mathbb E[\xi_{k,s}^t\mid\mathcal G_{t,s}]=0,
\qquad
\mathbb E[\|\xi_{k,s}^t\|^2\mid\mathcal G_{t,s}]\leq\Sigma^2.
\]
Thus the local stochastic errors form a martingale-difference
sequence with respect to the within-round filtration. The argument
below uses only this conditional first-moment property and the
conditional second-moment bound, together with Cauchy--Schwarz; it
never factors products of noises from different local steps and does
not assume independence across local steps. The filtration already
contains the objective-report randomness and all earlier minibatches.

Using the recursion, Cauchy--Schwarz, and
\(\|x+y+z\|^2\leq3(\|x\|^2+\|y\|^2+\|z\|^2)\), we obtain
\[
D_{t,s}
\leq
3\eta_\ell^2s
\sum_{\tau=0}^{s-1}
\left(
L_w^2D_{t,\tau}
+
\zeta^2
+
\Sigma^2
\right).
\]
Here we used \(L_w\)-smoothness to bound
\(\|q_{k,\tau}\|\leq L_w\|\delta_{k,\tau}\|\), the uniform sampling
identity, and the fact that \(v_{t,\tau}\in\mathcal D\) on the
fixed-horizon trajectory. More explicitly, Jensen's inequality on
the simplex and Assumption~\ref{as:hetero} give
\[
\begin{aligned}
\frac1N\sum_{k=1}^N\|a_{k,\tau}\|^2
&=
\frac1N\sum_{k=1}^N
\left\|
\sum_{i=1}^m w_i^t
\left(
\nabla f_i^{(k)}(v_{t,\tau})
 -\nabla f_i(v_{t,\tau})
\right)
\right\|^2\\
&\leq
\sum_{i=1}^m w_i^t
\frac1N\sum_{k=1}^N
\left\|
\nabla f_i^{(k)}(v_{t,\tau})
 -\nabla f_i(v_{t,\tau})
\right\|^2\\
&\leq
\sum_{i=1}^m w_i^t\zeta_i^2
\leq\zeta^2.
\end{aligned}
\]
Therefore, uniform client sampling gives
\[
\mathbb{E}\left[
\left.
\frac1M\sum_{k\in S_t}\|a_{k,\tau}\|^2
\right|\mathcal F_t
\right]
 =\frac1N\sum_{k=1}^N\|a_{k,\tau}\|^2
 \leq\zeta^2,
\]
and the tower property applied to the martingale-difference oracle
gives the corresponding \(\Sigma^2\) noise bound.

We prove \eqref{eq:virtual-path-deviation} by induction. If
\(D_{t,\tau}\leq4\eta_\ell^2\tau^2Q\) for \(\tau<s\), then
\[
D_{t,s}
\leq
4\eta_\ell^4L_w^2s^4Q
+
3\eta_\ell^2s^2Q
\leq
4\eta_\ell^2s^2Q,
\]
where the last inequality follows from
\(\eta_\ell L_ws\leq1/4\). The case \(s=0\) is immediate.
\end{proof}

\begin{lemma}[Local-drift bias bound]
\label{lem:local-drift-bias}
Under Assumptions~\ref{as:uniform-aggregation}, \ref{as:smooth},
\ref{as:slack}, \ref{as:hetero}, and~\ref{as:steps}, and the
virtual-trajectory bound in
Lemma~\ref{lem:grad2-derived}, the
conditional-mean bias \(b_t\) defined in~\eqref{eq:bt_definition}
for every round \(t=0,\ldots,T-1\)
satisfies
\begin{equation}
\label{eq:bt_norm_bound}
\|b_t\|
\leq
\frac{
L_w\eta_\ell^2\bar G
}{
2
}
K(K-1),
\end{equation}
where
\[
\bar G:=\sqrt{\bar G^2}.
\]
Consequently,
\begin{equation}
\label{eq:bt_squared_bound}
\|b_t\|^2
\leq
\frac{
L_w^2\eta_\ell^4\bar G^2
}{
4
}
K^2(K-1)^2.
\end{equation}
\end{lemma}

\begin{proof}
Let
\[
d_{k,s}
:=
\nabla L_{\wv^t}^{(k)}(\theta_{k,s}^t)
-
\nabla L_{\wv^t}^{(k)}(\theta^t),
\qquad
D_s^{\rm drift}
:=
\frac1M\sum_{k\in S_t}d_{k,s}.
\]
Using the virtual trajectory and \(L_w\)-smoothness,
\[
\|d_{k,s}\|
\leq
L_w\|\theta_{k,s}^t-v_{t,s}\|
+
L_w\|v_{t,s}-\theta^t\|.
\]
Since \(\|\nabla L_{\wv^t}(v_{t,u})\|\leq G_{\max}\),
\[
\|v_{t,s}-\theta^t\|
\leq
\eta_\ell sG_{\max}.
\]
Jensen's inequality and
\eqref{eq:virtual-path-deviation} therefore give
\[
\mathbb E\left[
\left.
\|D_s^{\rm drift}\|
\right|\mathcal F_t
\right]
\leq
L_w\eta_\ell s
\left(
\sqrt{C_\delta Q}+G_{\max}
\right)
\leq
L_w\eta_\ell s\bar G.
\]
Substituting this estimate into the definition of \(b_t\) and using
\(\sum_{s=0}^{K-1}s=K(K-1)/2\) proves
\eqref{eq:bt_norm_bound}. Squaring gives
\eqref{eq:bt_squared_bound}.
\end{proof}
{
\begin{lemma}[Second moment of the averaged update]
\label{lem:update-second-moment}
Assume Assumptions~\ref{as:uniform-aggregation},
\ref{as:smooth},
\ref{as:slack},
\ref{as:hetero}, and~\ref{as:steps}.
For every round \(t=0,\ldots,T-1\), condition on the round history
\(\mathcal F_t\). The weight vector \(\wv^t\) and the global model
\(\theta^t\) are then fixed, and \(S_t\) is sampled uniformly without
replacement with
\[
|S_t|=M=pN.
\]

Each selected client \(k\in S_t\) performs \(K\) local SGD steps
\[
\theta_k^{s+1}
=
\theta_k^s-\eta_\ell g_k^s,
\qquad
\theta_k^0=\theta^t.
\]
Define
\[
\Delta\theta_k^t
:=
\theta_k^K-\theta^t,
\qquad
\Delta^t
:=
\frac{1}{M}
\sum_{k\in S_t}
\Delta\theta_k^t.
\]
Then, conditionally on \(\mathcal F_t\),
\begin{align}
\mathbb{E}
\left[
\|\Delta^t\|^2
\mid
\mathcal F_t
\right]
\leq\;&
C_0\eta_\ell^2K^2
\left\|
\nabla L_{\wv^t}(\theta^t)
\right\|^2
+
C_1\eta_\ell^2K
\kappa_M\Sigma^2
\nonumber\\
&+
C_2\eta_\ell^2K^2
\mathcal H_N(p;\zeta)
\nonumber\\
&+
C_B\eta_\ell^4
K^2(K-1)(2K-1),
\label{eq:update_second_moment}
\end{align}
where
\[
\Sigma^2
:=
\sum_{i=1}^m\sigma_i^2,
\qquad
C_B
:=
\frac{L_w^2\bar G^2}{2},
\]
and \(C_0,C_1,C_2>0\) are numerical constants independent of
\(t,T,N,p,\eta_\ell,\eta_g\), and \(K\). As before, one may take
\[
C_0=6,
\qquad
C_1=3,
\qquad
C_2=6.
\]
\end{lemma}
}

{
\noindent\textit{Proof.}
Throughout the proof, we condition on the history of the round 
\(\mathcal F_t\). Hence, \(\theta^t\) and the stale weight vector
\(\wv^t\) are fixed, while the remaining randomness is due to the
current client sample \(S_t\) and the current-round minibatches.

Recall
\[
\Delta^t
=
-\eta_\ell
\frac{1}{M}
\sum_{k\in S_t}
\sum_{s=0}^{K-1}
g_k^s.
\]
Decompose each stochastic gradient as
\[
g_k^s
=
\underbrace{
\nabla L_{\wv^t}^{(k)}(\theta^t)
}_{=:a_k}
+
\underbrace{
\left(
\nabla L_{\wv^t}^{(k)}(\theta_k^s)
-
\nabla L_{\wv^t}^{(k)}(\theta^t)
\right)
}_{=:d_k^s}
+
\underbrace{
\left(
g_k^s
-
\nabla L_{\wv^t}^{(k)}(\theta_k^s)
\right)
}_{=:\xi_{k,s}^t}.
\]
By Lemma~\ref{lem:martingale-noise},
\[
\mathbb{E}
\left[
\xi_{k,s}^t
\mid
 \mathcal G_{t,s}
\right]
=
0,
\qquad
\mathbb{E}
\left[
\|\xi_{k,s}^t\|^2
\mid
 \mathcal G_{t,s}
\right]
\le
\Sigma^2.
\]

Define
\[
A
:=
\frac{1}{M}
\sum_{k\in S_t}
\sum_{s=0}^{K-1}
a_k
=
K
\frac{1}{M}
\sum_{k\in S_t}
a_k,
\]
\[
B
:=
\frac{1}{M}
\sum_{k\in S_t}
\sum_{s=0}^{K-1}
d_k^s,
\qquad
C
:=
\frac{1}{M}
\sum_{k\in S_t}
\sum_{s=0}^{K-1}
\xi_{k,s}^t.
\]
Thus
\[
\Delta^t
=
-\eta_\ell(A+B+C).
\]

\paragraph{(i) Bounding the main client-sampling term \(A\).}
Because \(\theta^t\) and \(\wv^t\) are
\(\mathcal F_t\)-measurable, the vectors
\[
a_k
=
\nabla L_{\wv^t}^{(k)}(\theta^t),
\qquad
k=1,\ldots,N,
\]
form a fixed finite population conditional on \(\mathcal F_t\).
Define its population mean by
\[
\mu_t
:=
\frac{1}{N}
\sum_{k=1}^N a_k
=
\nabla L_{\wv^t}(\theta^t).
\]

Therefore,
\[
\frac{1}{M}
\sum_{k\in S_t}a_k
=
\mu_t
+
\frac{1}{M}
\sum_{k\in S_t}
(a_k-\mu_t).
\]
Using
\(\|x+y\|^2\leq2\|x\|^2+2\|y\|^2\),
\begin{align}
\mathbb{E}
\left[
\|A\|^2
\mid
\mathcal F_t
\right]
\leq\;&
2K^2
\|\mu_t\|^2
\nonumber\\
&+
2K^2
\mathbb{E}
\left[
\left.
\left\|
\frac{1}{M}
\sum_{k\in S_t}
(a_k-\mu_t)
\right\|^2
\right|
\mathcal F_t
\right].
\label{eq:A-decomp-stale}
\end{align}

To evaluate the second term, let
\[
I_k
:=
\mathbf 1\{k\in S_t\}.
\]
Since \(S_t\) is sampled uniformly without replacement and is
independent of \(\mathcal F_t\),
\[
\mathbb{E}[I_k\mid\mathcal F_t]
=
\frac{M}{N},
\]
and, for \(k\neq\ell\),
\[
\mathbb{E}
[I_kI_\ell\mid\mathcal F_t]
=
\frac{M(M-1)}{N(N-1)}.
\]
Since
\[
\sum_{k=1}^N(a_k-\mu_t)=0,
\]
the standard finite-population variance identity gives
\begin{align}
&
\mathbb{E}
\left[
\left.
\left\|
\frac{1}{M}
\sum_{k\in S_t}
(a_k-\mu_t)
\right\|^2
\right|
\mathcal F_t
\right]
\nonumber\\
&\qquad=
\frac{N-M}{M(N-1)}
\frac{1}{N}
\sum_{k=1}^N
\|a_k-\mu_t\|^2.
\label{eq:finite-pop-stale}
\end{align}

Moreover,
\[
a_k-\mu_t
=
\sum_{i=1}^m
w_i^t
\left(
\nabla f_i^{(k)}(\theta^t)
-
\nabla f_i(\theta^t)
\right).
\]
Because \(\wv^t\in\Delta_m\) and the squared norm is convex,
\[
\|a_k-\mu_t\|^2
\le
\sum_{i=1}^m
w_i^t
\left\|
\nabla f_i^{(k)}(\theta^t)
-
\nabla f_i(\theta^t)
\right\|^2.
\]
Averaging over clients and using the global-iterate clause of the
trajectory-wise condition in Assumption~\ref{as:hetero},
\[
\frac{1}{N}
\sum_{k=1}^N
\|a_k-\mu_t\|^2
\le
\sum_{i=1}^m
w_i^t\zeta_i^2
\le
\sum_{i=1}^m
\zeta_i^2.
\]
Consequently,
\begin{align}
\mathbb{E}
\left[
\|A\|^2
\mid
\mathcal F_t
\right]
\leq\;&
2K^2
\left\|
\nabla L_{\wv^t}(\theta^t)
\right\|^2
\nonumber\\
&+
2K^2
\mathcal H_N(p;\zeta).
\label{eq:A-bound-stale}
\end{align}
}

{
\paragraph{(ii) Stochastic-gradient noise term \(C\).}
Define
\[
H_s:=\sum_{k\in S_t}\xi_{k,s}^t,
\qquad
C=\frac1M\sum_{s=0}^{K-1}H_s.
\]
For a fixed local step \(s\), the definition of \(\kappa_M\) gives
\[
\mathbb E\left[
\left.
\|H_s\|^2
\right|
\mathcal G_{t,s}
\right]
\leq M^2\kappa_M\Sigma^2.
\]
For \(s<u\), \(H_s\) is
\(\mathcal G_{t,u}\)-measurable, while
\(\mathbb E[H_u\mid\mathcal G_{t,u}]=0\). Therefore all cross terms
between different local steps vanish by the tower property:
\[
\mathbb E\langle H_s,H_u\rangle
=
\mathbb E\left[
\left\langle H_s,
\mathbb E[H_u\mid\mathcal G_{t,u}]
\right\rangle
\right]
=0.
\]
Consequently,
\[
\mathbb E\left[
\left.
\|C\|^2
\right|
\mathcal F_t
\right]
\leq
\frac{1}{M^2}
\sum_{s=0}^{K-1}
\mathbb E\left[
\left.
\|H_s\|^2
\right|
\mathcal F_t
\right]
\leq
\kappa_M K\Sigma^2.
\]
}

{
\paragraph{(iii) Client-drift term \(B\).}
For
\[
d_k^s
:=
\nabla L_{\wv^t}^{(k)}(\theta_k^s)
-
\nabla L_{\wv^t}^{(k)}(\theta^t),
\qquad
D_s
:=
\frac1M\sum_{k\in S_t}d_k^s,
\]
use the virtual trajectory from Lemma~\ref{lem:grad2-derived} to
write
\[
d_k^s
=
\left[
\nabla L_{\wv^t}^{(k)}(\theta_k^s)
-
\nabla L_{\wv^t}^{(k)}(v_{t,s})
\right]
+
\left[
\nabla L_{\wv^t}^{(k)}(v_{t,s})
-
\nabla L_{\wv^t}^{(k)}(\theta^t)
\right].
\]
The \(L_w\)-smoothness of the local scalarized objectives implies
\[
\|d_k^s\|
\leq
L_w\|\theta_k^s-v_{t,s}\|
+
L_w\|v_{t,s}-\theta^t\|.
\]
Moreover,
\[
\|v_{t,s}-\theta^t\|
\leq
\eta_\ell sG_{\max},
\]
because \(\|\nabla L_{\wv^t}(v_{t,u})\|\leq G_{\max}\). Thus,
by Jensen's inequality and
\eqref{eq:virtual-path-deviation},
\begin{align}
\mathbb E\left[
\left.
\|D_s\|^2
\right|\mathcal F_t
\right]
&\leq
\mathbb E\left[
\left.
\frac1M\sum_{k\in S_t}\|d_k^s\|^2
\right|\mathcal F_t
\right]
\nonumber\\
&\leq
2L_w^2
\left(
D_{t,s}+\eta_\ell^2s^2G_{\max}^2
\right)
\nonumber\\
&\leq
L_w^2\eta_\ell^2s^2\bar G^2.
\label{eq:Ds_second_moment}
\end{align}
Here \(\bar G^2=2(G_{\max}^2+C_\delta Q)\) is the
constant introduced in Lemma~\ref{lem:grad2-derived}.

Since \(B=\sum_{s=0}^{K-1}D_s\), Cauchy--Schwarz gives
\[
\|B\|^2
\leq
K\sum_{s=0}^{K-1}\|D_s\|^2.
\]
Consequently,
\begin{align}
\mathbb E\left[
\left.
\|B\|^2
\right|\mathcal F_t
\right]
&\leq
K L_w^2\eta_\ell^2\bar G^2
\sum_{s=0}^{K-1}s^2
\nonumber\\
&=
\frac{
L_w^2\eta_\ell^2\bar G^2
}{6}
K^2(K-1)(2K-1),
\label{eq:B_second_moment_bound}
\end{align}
where
\[
\sum_{s=0}^{K-1}s^2
=
\frac{K(K-1)(2K-1)}{6}.
\]
Therefore,
\begin{align}
3\eta_\ell^2
\mathbb E\left[
\left.
\|B\|^2
\right|\mathcal F_t
\right]
\leq
\frac{L_w^2\bar G^2}{2}
\eta_\ell^4K^2(K-1)(2K-1).
\label{eq:update_second_moment_drift}
\end{align}
}
{
\paragraph{Combining the three terms.}
Using
\[
\Delta^t
=
-\eta_\ell(A+B+C)
\]
and
\[
\|A+B+C\|^2
\le
3\|A\|^2
+
3\|B\|^2
+
3\|C\|^2,
\]
we obtain
\begin{align}
\mathbb{E}
\left[
\|\Delta^t\|^2
\mid
\mathcal F_t
\right]
\le\;&
3\eta_\ell^2
\mathbb{E}
\left[
\|A\|^2
\mid
\mathcal F_t
\right]
\nonumber\\
&+
3\eta_\ell^2
\mathbb{E}
\left[
\|B\|^2
\mid
\mathcal F_t
\right]
\nonumber\\
&+
3\eta_\ell^2
\mathbb{E}
\left[
\|C\|^2
\mid
\mathcal F_t
\right].
\label{eq:ABC_combination}
\end{align}
}
Substituting the bounds for \(A\), \(B\), and \(C\) gives
{
\begin{align}
\mathbb{E}
\left[
\|\Delta^t\|^2
\mid
\mathcal F_t
\right]
\le\;&
6\eta_\ell^2K^2
\left\|
\nabla L_{\wv^t}(\theta^t)
\right\|^2
+
3\eta_\ell^2K
\kappa_M\Sigma^2
\nonumber\\
&+
6\eta_\ell^2K^2
\mathcal H_N(p;\zeta)
\nonumber\\
&+
\frac{L_w^2\bar G^2}{2}
\eta_\ell^4
K^2(K-1)(2K-1).
\label{eq:update_second_moment_explicit}
\end{align}
}
Therefore,
\eqref{eq:update_second_moment} follows with
\[
C_0=6,
\qquad
C_1=3,
\qquad
C_2=6,
\qquad
C_B=\frac{L_w^2\bar G^2}{2}.
\]
The final term is retained explicitly because it depends on
\(\eta_\ell\) and \(K\). It vanishes when \(K=1\).

{
\paragraph{Stale-weight gradient error.}
Define the discrepancy between the gradient of the scalarized
objective actually used in round \(t\) and the scalarized objective
associated with the ideal current hypervolume weights by
\[
e_t
:=
\nabla L_{\wv^t}(\theta^t)
-
\nabla L_{\wv^{\star,t}}(\theta^t).
\]
Equivalently,
\[
e_t
=
\sum_{i=1}^m
\left(
w_i^t-w_i^{\star,t}
\right)
\nabla f_i(\theta^t).
\]
Let
\[
G_{\max}
:=
\max_{i\in[m]}G_i.
\]
Then, by Lemma~\ref{lem:self-bounded-gradients},
\begin{equation}
\label{eq:grad-weight-error}
\|e_t\|
\le
G_{\max}
\|\wv^t-\wv^{\star,t}\|_1.
\end{equation}

Using Lemma~\ref{lem:stale-weight-discrepancy}, for every
\(t\geq1\),
\begin{align}
\mathbb{E}\|e_t\|^2
\leq\;&
2G_{\max}^2
C(\rho_{\mathrm{clip}},\bm r)^2
\Xi_{\mathrm{loss},T}
\nonumber\\
&+
2G_{\max}^2
C(\rho_{\mathrm{clip}},\bm r)^2
\mathbb{E}
\left[
\left\|
 \bm F(\theta^t)-\bm F(\theta^{t-1})
\right\|_\infty^2
\right].
\label{eq:stale-gradient-error-pre}
\end{align}

We next control the second term through the inter-round model
displacement. Since \(\wv^t\in\Delta_m\),
\[
\left\|
\nabla L_{\wv^t}(\theta^t)
\right\|
\le
\sum_{i=1}^m
w_i^t
\|\nabla f_i(\theta^t)\|
\le
G_{\max}.
\]
Therefore, Lemma~\ref{lem:update-second-moment} implies
\begin{equation}
\label{eq:V-definition}
\mathbb{E}
\left[
\|\Delta^t\|^2
\mid
\mathcal F_t
\right]
\le
\eta_\ell^2K^2\mathcal V,
\end{equation}
where
\[
\mathcal V
:=
C_0G_{\max}^2
+C_1\frac{\Sigma^2}{K}
+
C_2\mathcal H_N(p;\zeta)
+
C_B\eta_\ell^2
(K-1)(2K-1).
\]
Since
\[
\theta^{t+1}-\theta^t
=
\eta_g\Delta^t,
\]
we obtain
\begin{equation}
\label{eq:inter-round-model-movement}
\mathbb{E}
\left[
\|\theta^{t+1}-\theta^t\|^2
\right]
\le
\eta_{\mathrm{eff}}^2
\mathcal V,
\qquad
\eta_{\mathrm{eff}}
=
\eta_g\eta_\ell K.
\end{equation}

By Lemma~\ref{lem:self-bounded-gradients}, each global objective is
\(G_i\)-Lipschitz on the region containing consecutive global
iterates. Hence
\[
\left\|
 \bm F(\theta^t)-\bm F(\theta^{t-1})
\right\|_\infty
\le
G_{\max}
\|\theta^t-\theta^{t-1}\|.
\]
Combining this relation with
\eqref{eq:inter-round-model-movement} and
\eqref{eq:stale-gradient-error-pre} gives, for \(t\geq1\),
\begin{equation}
\label{eq:stale-gradient-error}
\mathbb{E}\|e_t\|^2
\le
2G_{\max}^2C(\rho_{\mathrm{clip}},\bm r)^2
\Xi_{\mathrm{loss},T}
+
2G_{\max}^4C(\rho_{\mathrm{clip}},\bm r)^2
\eta_{\mathrm{eff}}^2
\mathcal V.
\end{equation}

For the initialization round, both
\(\wv^0\) and \(\wv^{\star,0}\) belong to \(\Delta_m\). Therefore,
\[
\|\wv^0-\wv^{\star,0}\|_1
\le
2,
\qquad
\|e_0\|
\le
2G_{\max}.
\]
The resulting initialization contribution occurs only once and will
therefore produce an \(O(T^{-1})\) term after telescoping.
}
{
\paragraph{Relating the actual weighted gradient to the
log-hypervolume gradient.}
By the definition of \(e_t\) and the ideal colinearity relation
\eqref{eq:phi-lw-colinear-app},
\begin{equation}
\label{eq:actual-gradient-decomposition}
\nabla L_{\wv^t}(\theta^t)
=
\frac{1}{c^t}
\nabla\Phi(\theta^t)
+
e_t.
\end{equation}
Consequently,
\begin{align}
\left\|
\nabla L_{\wv^t}(\theta^t)
\right\|^2
&\le
2
\left\|
\nabla L_{\wv^{\star,t}}(\theta^t)
\right\|^2
+
2\|e_t\|^2
\nonumber\\
&\le
\frac{2}{c_{\min}^2}
\left\|
\nabla\Phi(\theta^t)
\right\|^2
+
2\|e_t\|^2.
\label{eq:actual-gradient-norm-bound}
\end{align}

Substituting \eqref{eq:actual-gradient-norm-bound} into
Lemma~\ref{lem:update-second-moment} yields
\begin{align}
\mathbb{E}
\left[
\|\Delta^t\|^2
\mid
\mathcal F_t
\right]
\leq\;&
\frac{
2C_0\eta_\ell^2K^2
}{
c_{\min}^2
}
\left\|
\nabla\Phi(\theta^t)
\right\|^2
+
2C_0\eta_\ell^2K^2
\|e_t\|^2
\nonumber\\
&+
C_1\eta_\ell^2K
\kappa_M\Sigma^2
+
C_2\eta_\ell^2K^2
\mathcal H_N(p;\zeta)
\nonumber\\
&+
C_B\eta_\ell^4K^2
(K-1)(2K-1).
\label{eq:delta-second-moment-conv}
\end{align}
}
{
\paragraph{One-round descent with stale estimated weights.}
By Lemma~\ref{lem:smooth-phi} and
\[
\theta^{t+1}
=
\theta^t+\eta_g\Delta^t,
\]
conditioning on \(\mathcal F_t\) gives
\begin{equation}
\label{eq:phi-descent-1}
\begin{aligned}
\mathbb{E}
\left[
\Phi(\theta^{t+1})
\mid
\mathcal F_t
\right]
\le\;&
\Phi(\theta^t)
+
\eta_g
\left\langle
\nabla\Phi(\theta^t),
\mathbb{E}
\left[
\Delta^t
\mid
\mathcal F_t
\right]
\right\rangle
\\
&+
\frac{L_\Phi\eta_g^2}{2}
\mathbb{E}
\left[
\|\Delta^t\|^2
\mid
\mathcal F_t
\right].
\end{aligned}
\end{equation}

Using \eqref{eq:conditional_mean_update},
\begin{align}
\mathbb{E}
\left[
\Phi(\theta^{t+1})
\mid
\mathcal F_t
\right]
\leq\;&
\Phi(\theta^t)
-
\eta_{\mathrm{eff}}
\left\langle
\nabla\Phi(\theta^t),
\nabla L_{\wv^t}(\theta^t)
\right\rangle
\nonumber\\
&+
\eta_g
\left\langle
\nabla\Phi(\theta^t),
b_t
\right\rangle
+
\frac{L_\Phi\eta_g^2}{2}
\mathbb{E}
\left[
\|\Delta^t\|^2
\mid
\mathcal F_t
\right].
\label{eq:phi-descent-2}
\end{align}

From \eqref{eq:actual-gradient-decomposition},
\begin{align}
\left\langle
\nabla\Phi(\theta^t),
\nabla L_{\wv^t}(\theta^t)
\right\rangle
=
\frac{1}{c^t}
\left\|
\nabla\Phi(\theta^t)
\right\|^2
+
\left\langle
\nabla\Phi(\theta^t),
e_t
\right\rangle .
\end{align}
Using \(c^t\le c_{\max}\) and Young's inequality,
\[
\left|
\left\langle
\nabla\Phi(\theta^t),
e_t
\right\rangle
\right|
\le
\frac{1}{4c_{\max}}
\left\|
\nabla\Phi(\theta^t)
\right\|^2
+
c_{\max}
\|e_t\|^2.
\]
Therefore,
\begin{align}
-\eta_{\mathrm{eff}}
\left\langle
\nabla\Phi(\theta^t),
\nabla L_{\wv^t}(\theta^t)
\right\rangle
\le\;&
-\frac{
3\eta_{\mathrm{eff}}
}{
4c_{\max}
}
\left\|
\nabla\Phi(\theta^t)
\right\|^2
\nonumber\\
&+
c_{\max}
\eta_{\mathrm{eff}}
\|e_t\|^2.
\label{eq:stale-direction-descent}
\end{align}

For the local-drift bias, Young's inequality similarly gives
\begin{align}
\eta_g
\left|
\left\langle
\nabla\Phi(\theta^t),
b_t
\right\rangle
\right|
\le\;&
\frac{
\eta_{\mathrm{eff}}
}{
4c_{\max}
}
\left\|
\nabla\Phi(\theta^t)
\right\|^2
\nonumber\\
&+
\eta_g
\frac{c_{\max}}{\eta_\ell K}
\|b_t\|^2.
\label{eq:drift-young-stale}
\end{align}

Finally, substituting
\eqref{eq:delta-second-moment-conv},
\eqref{eq:stale-direction-descent}, and
\eqref{eq:drift-young-stale} into
\eqref{eq:phi-descent-2}, and choosing the constant in
Assumption~\ref{as:steps} as specified there, the
smoothness-generated gradient term is bounded by
\[
\frac{L_\Phi C_0\eta_{\mathrm{eff}}^2}{c_{\min}^2}
\leq
\frac{\eta_{\mathrm{eff}}}{8c_{\max}}.
\]
The two Young inequalities already leave
\(-\eta_{\mathrm{eff}}\|\nabla\Phi(\theta^t)\|^2/(2c_{\max})\);
therefore the following one-round inequality holds with the displayed
\(-\eta_{\mathrm{eff}}/(4c_{\max})\) coefficient:
\begin{align}
\mathbb{E}
\left[
\Phi(\theta^{t+1})
\mid
\mathcal F_t
\right]
\leq\;&
\Phi(\theta^t)
-
\frac{
\eta_{\mathrm{eff}}
}{
4c_{\max}
}
\left\|
\nabla\Phi(\theta^t)
\right\|^2
\nonumber\\
&+
C_e\eta_{\mathrm{eff}}
\|e_t\|^2
+
\eta_g
\frac{c_{\max}}{\eta_\ell K}
\|b_t\|^2
\nonumber\\
&+
\frac{L_\Phi C_1}{2}
\eta_g^2\eta_\ell^2K
\Sigma^2
\nonumber\\
&+
\frac{L_\Phi C_2}{2}
\eta_g^2\eta_\ell^2K^2
\mathcal H_N(p;\zeta)
\nonumber\\
&+
\frac{L_\Phi C_B}{2}
\eta_g^2\eta_\ell^4K^2
(K-1)(2K-1),
\label{eq:phi-one-step-final}
\end{align}
where
\[
\begin{aligned}
C_e&:=c_{\max}+L_\Phi C_0\bar\eta_{\mathrm{eff}}
=c_{\max}+6L_\Phi\bar\eta_{\mathrm{eff}},
& C_\Sigma&:=\frac{L_\Phi C_1}{2},\\
C_H&:=\frac{L_\Phi C_2}{2},
& C_B'&:=\frac{L_\Phi C_B}{2}.
\end{aligned}
\]
}

{
\paragraph{Telescoping and stationarity rate.}
For convenience, define
\[
C_{\mathrm{est}}
:=
2G_{\max}^2 C(\rho_{\mathrm{clip}},\bm r)^2,
\qquad
C_{\mathrm{st}}
:=
2G_{\max}^4 C(\rho_{\mathrm{clip}},\bm r)^2.
\]
By \eqref{eq:stale-gradient-error}, for every \(t\geq1\),
\[
\mathbb{E}\|e_t\|^2
\le
C_{\mathrm{est}}\Xi_{\mathrm{loss},T}
+
C_{\mathrm{st}}
\eta_{\mathrm{eff}}^2
\mathcal V.
\]
For \(t=0\), we have
\[
\|e_0\|^2
\le
4G_{\max}^2.
\]
Hence, for every \(t\geq0\), the unified bound
\begin{equation}
\label{eq:unified-stale-error}
\mathbb{E}\|e_t\|^2
\le
C_{\mathrm{est}}\Xi_{\mathrm{loss},T}
+
C_{\mathrm{st}}
\eta_{\mathrm{eff}}^2
\mathcal V
+
4G_{\max}^2
\mathbf{1}_{\{t=0\}}
\end{equation}
holds.

Taking total expectation in
\eqref{eq:phi-one-step-final} and using the tower property
\[
\mathbb{E}
\left[
\Phi(\theta^{t+1})
\right]
=
\mathbb{E}
\left[
\mathbb{E}
\left[
\Phi(\theta^{t+1})
\mid
\mathcal F_t
\right]
\right],
\]
Using the explicit definitions
\[
C_\Sigma:=\frac{L_\Phi C_1}{2},
\qquad
C_H:=\frac{L_\Phi C_2}{2},
\qquad
C_B':=\frac{L_\Phi C_B}{2},
\]
and the value of \(C_e\) specified in
\eqref{eq:phi-one-step-final}, we obtain
\begin{align}
\mathbb{E}
\left[
\Phi(\theta^{t+1})
\right]
\leq\;&
\mathbb{E}
\left[
\Phi(\theta^t)
\right]
-
\alpha
\mathbb{E}
\left[
\left\|
\nabla\Phi(\theta^t)
\right\|^2
\right]
\nonumber\\
&+
C_\Sigma
\eta_g^2\eta_\ell^2K
\Sigma^2
+
C_H
\eta_g^2\eta_\ell^2K^2
\mathcal H_N(p;\zeta)
\nonumber\\
&+
\eta_g
\frac{c_{\max}}{\eta_\ell K}
\mathbb{E}\|b_t\|^2
\nonumber\\
&+
C_B'
\eta_g^2\eta_\ell^4K^2
(K-1)(2K-1)
\nonumber\\
&+
C_e\eta_{\mathrm{eff}}
\mathbb{E}\|e_t\|^2,
\label{eq:one_round_with_drift}
\end{align}
where
\[
\alpha
=
\frac{\eta_{\mathrm{eff}}}{4c_{\max}},
\qquad
\eta_{\mathrm{eff}}
:=
\eta_g\eta_\ell K.
\]

Substituting \eqref{eq:unified-stale-error} and using the explicit
choice of \(C_e\) above gives
\[
C_{\mathrm{loss}}
:=C_eC_{\mathrm{est}},
\qquad
C_{\mathrm{stale}}
:=C_eC_{\mathrm{st}},
\qquad
C_{\mathrm{init}}
:=4C_eG_{\max}^2.
\]
Thus
\begin{align}
\mathbb{E}
\left[
\Phi(\theta^{t+1})
\right]
\leq\;&
\mathbb{E}
\left[
\Phi(\theta^t)
\right]
-
\alpha
\mathbb{E}
\left[
\left\|
\nabla\Phi(\theta^t)
\right\|^2
\right]
\nonumber\\
&+
C_\Sigma
\eta_g^2\eta_\ell^2K
\Sigma^2
+
C_H
\eta_g^2\eta_\ell^2K^2
\mathcal H_N(p;\zeta)
\nonumber\\
&+
\eta_g
\frac{c_{\max}}{\eta_\ell K}
\mathbb{E}\|b_t\|^2
\nonumber\\
&+
C_B'
\eta_g^2\eta_\ell^4K^2
(K-1)(2K-1)
\nonumber\\
&+
C_{\mathrm{loss}}
\eta_{\mathrm{eff}}
\Xi_{\mathrm{loss},T}
+
C_{\mathrm{stale}}
\eta_{\mathrm{eff}}^3
\mathcal V
\nonumber\\
&+
C_{\mathrm{init}}
\eta_{\mathrm{eff}}
\mathbf{1}_{\{t=0\}}.
\label{eq:one_round_stale_expanded}
\end{align}

Using Lemma~\ref{lem:local-drift-bias}, the conditional-mean
local-drift contribution satisfies
\begin{align}
\eta_g
\frac{c_{\max}}{\eta_\ell K}
\mathbb{E}\|b_t\|^2
&\le
\frac{
c_{\max}L_w^2\bar G^2
}{
4
}
\eta_g\eta_\ell^3
K(K-1)^2
\nonumber\\
&=
\widetilde C_b
\eta_{\mathrm{eff}}
\eta_\ell^2
(K-1)^2,
\label{eq:mean_drift_contribution}
\end{align}
where
\[
\widetilde C_b
:=
\frac{
c_{\max}L_w^2\bar G^2
}{
4
}.
\]

Therefore, the one-round inequality can be written as
\begin{align}
\mathbb{E}
\left[
\Phi(\theta^{t+1})
\right]
\leq\;&
\mathbb{E}
\left[
\Phi(\theta^t)
\right]
-
\alpha
\mathbb{E}
\left[
\left\|
\nabla\Phi(\theta^t)
\right\|^2
\right]
\nonumber\\
&+
C_\Sigma
\eta_g^2\eta_\ell^2K
\Sigma^2
+
C_H
\eta_g^2\eta_\ell^2K^2
\mathcal H_N(p;\zeta)
\nonumber\\
&+
\widetilde C_b
\eta_{\mathrm{eff}}
\eta_\ell^2(K-1)^2
\nonumber\\
&+
C_B'
\eta_g^2\eta_\ell^4K^2
(K-1)(2K-1)
\nonumber\\
&+
C_{\mathrm{loss}}
\eta_{\mathrm{eff}}
\Xi_{\mathrm{loss},T}
+
C_{\mathrm{stale}}
\eta_{\mathrm{eff}}^3
\mathcal V
\nonumber\\
&+
C_{\mathrm{init}}
\eta_{\mathrm{eff}}
\mathbf{1}_{\{t=0\}}.
\label{eq:one_round_after_drift_bounds}
\end{align}

\paragraph{Telescoping and final stationarity bound.}
The objective is bounded below along the optimization trajectory.
Indeed, Assumption~\ref{as:slack} and the nonnegativity of the
losses imply
\[
0
<
s_i(\theta)
=
r_i-f_i(\theta)
\le
r_i.
\]
Consequently,
\begin{equation}
\label{eq:phi_lower_bound}
\Phi(\theta)
=
-\sum_{i=1}^m
\log s_i(\theta)
\ge
-\sum_{i=1}^m
\log r_i
=:
\Phi_{\rm lb}.
\end{equation}

Taking total expectations in the unconditional one-round inequality
\eqref{eq:one_round_after_drift_bounds}, summing over the deterministic
range \(t=0,\ldots,T-1\), and using the lower bound
\(\Phi(\theta)\geq\Phi_{\rm lb}\) gives
\begin{align}
\alpha
\sum_{t=0}^{T-1}
\mathbb{E}
\left[
\left\|
\nabla\Phi(\theta^t)
\right\|^2
\right]
\leq\;&
\Phi(\theta^0)-\Phi_{\rm lb}
\nonumber\\
&+
T C_\Sigma
\eta_g^2\eta_\ell^2K
\Sigma^2
\nonumber\\
&+
T C_H
\eta_g^2\eta_\ell^2K^2
\mathcal H_N(p;\zeta)
\nonumber\\
&+
T\widetilde C_b
\eta_{\mathrm{eff}}
\eta_\ell^2(K-1)^2
\nonumber\\
&+
T C_B'
\eta_g^2\eta_\ell^4K^2
(K-1)(2K-1)
\nonumber\\
&+
T C_{\mathrm{loss}}
\eta_{\mathrm{eff}}
\Xi_{\mathrm{loss},T}
\nonumber\\
&+
T C_{\mathrm{stale}}
\eta_{\mathrm{eff}}^3
\mathcal V
\nonumber\\
&+
C_{\mathrm{init}}
\eta_{\mathrm{eff}}.
\label{eq:telescope_before_division}
\end{align}

Dividing this inequality by \(\alpha T\) gives the fixed-horizon
average bound.

\begin{align}
\frac1T
\sum_{t=0}^{T-1}
\mathbb{E}
\left[
\left\|
\nabla\Phi(\theta^t)
\right\|^2
\right]
\leq\;&
\frac{4c_{\max}
(\Phi(\theta^0)-\Phi_{\rm lb})}
{\eta_{\mathrm{eff}}T}
+4c_{\max}C_\Sigma
\frac{\eta_{\mathrm{eff}}}{K}\Sigma^2
\nonumber\\
&+
4c_{\max}C_H\eta_{\mathrm{eff}}\mathcal H_N(p;\zeta)
\nonumber\\
&+
4c_{\max}\widetilde C_b\eta_\ell^2(K-1)^2
\nonumber\\
&+
4c_{\max}C_B'\eta_{\mathrm{eff}}\eta_\ell^2(K-1)(2K-1)
\nonumber\\
&+
4c_{\max}C_{\mathrm{loss}}\Xi_{\mathrm{loss},T}
\nonumber\\
&+
4c_{\max}C_{\mathrm{stale}}\eta_{\mathrm{eff}}^2\mathcal V
\nonumber\\
&+
\frac{4c_{\max}C_{\mathrm{init}}}{T}.
\label{eq:precombined_final_bound}
\end{align}

For the stochastic-gradient term, \(K\geq1\) gives
\(\eta_{\mathrm{eff}}/K\leq\eta_{\mathrm{eff}}\). By
Assumption~\ref{as:steps},
\[
\eta_{\mathrm{eff}}\leq\bar\eta_{\mathrm{eff}},
\qquad
(K-1)^2\leq(K-1)(2K-1).
\]
The right-hand side of
\eqref{eq:precombined_final_bound} is therefore bounded by the
theorem's right-hand side after substituting the explicit constants
listed in the theorem.

This proves the claimed fixed-horizon bound with the explicit
coefficients \(A_1,\ldots,A_7\) stated in
Theorem~\ref{thm:fedhv-explicit}.
\hfill\(\square\)
}


\subsubsection{Proportional-Fairness Background and Local First-Order Proofs}
\label{app:fairness-proofs}

\paragraph{Classical proportional fairness in slack space.}
Let
\[
\mathcal S
:=
\{
\bm s(\theta):\theta\in\Theta
\}
\subset\mathbb{R}_{++}^{m}
\]
denote the achievable slack set.

\begin{definition}[Proportional fairness in slack space]
\label{def:pf}
A point \(\bm s^\star\in\mathcal S\) is proportionally fair
if, for every \(\bm s\in\mathcal S\),
\begin{equation}
\label{eq:global_pf_definition}
\sum_{i=1}^{m}
\frac{s_i-s_i^\star}{s_i^\star}
\le
0.
\end{equation}
\end{definition}

If \(\mathcal S\) is convex and compact, any maximizer of
\[
\sum_{i=1}^{m}\log s_i
\]
over \(\mathcal S\) is proportionally fair in the sense of
Definition~\ref{def:pf}
\citep{kelly1997charging}. This classical statement is included only
for background. The achievable slack set induced by a neural network is
generally nonconvex, and our convergence result establishes stationarity
rather than global optimality. Accordingly, we do not claim that FedHV
obtains a classical global Nash bargaining solution or a globally
proportionally fair slack vector; the paper uses only the local
first-order implication below.

\begin{lemma}[First-order proportional fairness]
\label{lem:pf_local}
Let \(\theta^\star\in\Theta\) satisfy \(s_i(\theta^\star)>0\) for every
\(i\), and suppose that \(\theta^\star\) satisfies the first-order
necessary condition
for maximizing
\[
\Psi(\theta)
=
\sum_{i=1}^{m}\log s_i(\theta)
\]
over \(\Theta\); that is,
\begin{equation}
\label{eq:psi_first_order_condition}
\left\langle
\nabla\Psi(\theta^\star),d
\right\rangle
\le
0,
\qquad
\forall d\in T_\Theta(\theta^\star).
\end{equation}
Then, for every \(d\in T_\Theta(\theta^\star)\),
\begin{equation}
\label{eq:first_order_pf_slack}
\sum_{i=1}^{m}
\frac{
\left\langle
\nabla s_i(\theta^\star),d
\right\rangle
}{
s_i(\theta^\star)
}
\le
0.
\end{equation}
Equivalently,
\begin{equation}
\label{eq:first_order_pf_loss}
\sum_{i=1}^{m}
\frac{
\left\langle
\nabla f_i(\theta^\star),d
\right\rangle
}{
r_i-f_i(\theta^\star)
}
\ge
0.
\end{equation}
\end{lemma}

\begin{proof}
Differentiating the reference-dependent Nash-welfare surrogate gives
\[
\nabla\Psi(\theta)
=
\sum_{i=1}^{m}
\frac{1}{s_i(\theta)}
\nabla s_i(\theta).
\]
Therefore,
\[
\left\langle
\nabla\Psi(\theta^\star),d
\right\rangle
=
\sum_{i=1}^{m}
\frac{
\left\langle
\nabla s_i(\theta^\star),d
\right\rangle
}{
s_i(\theta^\star)
}.
\]
The first inequality follows from
\eqref{eq:psi_first_order_condition}. Since
\[
\nabla s_i(\theta)
=
-\nabla f_i(\theta),
\]
the equivalent loss-space expression follows immediately.
\end{proof}

\paragraph{Proof of Lemma~\ref{lem:pfgap_bound}.}
\label{app:proof-pfgap}

By Definition~\ref{def:pfgap},
\begin{align}
\mathrm{PFgap}(\theta)
&=
\sup_{\substack{
d\in T_\Theta(\theta)\\
\|d\|\le1
}}
\left\langle
\nabla\Psi(\theta),d
\right\rangle
\nonumber\\
&\le
\sup_{\|d\|\le1}
\left\langle
\nabla\Psi(\theta),d
\right\rangle
=
\|\nabla\Psi(\theta)\|.
\end{align}
The last equality is the dual characterization of the
Euclidean norm. Since \(\Psi=-\Phi\),
\[
\|\nabla\Psi(\theta)\|
=
\|\nabla\Phi(\theta)\|.
\]

If \(\Theta=\mathbb{R}^d\), then
\(T_\Theta(\theta)=\mathbb{R}^d\). When
\(\nabla\Psi(\theta)\neq0\), choosing
\[
d^\star
=
\frac{
\nabla\Psi(\theta)
}{
\|\nabla\Psi(\theta)\|
}
\]
attains the upper bound. When \(\nabla\Psi(\theta)=0\), both
sides are zero. Hence equality holds in the unconstrained
case.
\hfill\(\square\)

\paragraph{Proof of Corollary~\ref{cor:fedhv_pf}.}
\label{app:proof-fedhv-pf}

Lemma~\ref{lem:pfgap_bound} implies, for every \(t\),
\[
\mathrm{PFgap}(\theta^t)^2
\le
\|\nabla\Phi(\theta^t)\|^2.
\]
For the index \(t^\star\) defined in
Corollary~\ref{cor:fedhv_pf}, the Cauchy--Schwarz inequality
therefore gives
\begin{align}
\mathbb{E}
\left[
\mathrm{PFgap}(\theta^{t^\star})
\right]
&\le
\left(
\mathbb{E}
\left[
\mathrm{PFgap}(\theta^{t^\star})^2
\right]
\right)^{1/2}
\nonumber\\
&\le
\left(
\mathbb{E}
\left[
\|\nabla\Phi(\theta^{t^\star})\|^2
\right]
\right)^{1/2}
\nonumber\\
&=
\left(
 \min_{0\le t<T}
\mathbb{E}
\left[
\|\nabla\Phi(\theta^t)\|^2
\right]
\right)^{1/2}.
\end{align}
This proves~\eqref{eq:fedhv_pfgap_bound}.
\hfill\(\square\)

\section{Experimental Details and Reproducibility}
\label{app:experimental-details}

\subsection{Benchmark Construction}
\label{app:benchmark-construction}
\paragraph{\textsc{MultiMNIST}.}
Each example contains two spatially shifted MNIST digits in one image;
overlapping pixels are combined by the maximum intensity. The two tasks
predict the left and right digit labels~\citep{lecun1998gradient}.

\paragraph{\textsc{MNIST--FMNIST}.}
Each example combines one MNIST digit and one Fashion-MNIST item at
offset locations in a common grayscale image. The two tasks are digit
classification and Fashion-MNIST category
classification~\citep{lecun1998gradient,xiao2017fashion}.

\paragraph{\textsc{CIFAR10--MNIST}.}
A semi-transparent MNIST digit is superimposed on a CIFAR-10 image. The
two tasks are CIFAR-10 object classification and MNIST digit
classification. This benchmark creates a pronounced difference in task
difficulty and is therefore useful for examining worst-task
behavior~\citep{lecun1998gradient,krizhevsky2009learning}.

\paragraph{\textsc{CelebA5}.}
CelebA contains 40 binary facial attributes. Following the FedCMOO
benchmark construction, we partition them into five fixed objectives of
eight attributes each. The exact grouping is given in
Table~\ref{tab:celeba5-attribute-groups}~\citep{liu2015deep,askin2025fedcmoo}.

Figure~\ref{fig:dataset-examples} visually summarizes the four
benchmark constructions. For \textsc{CelebA5}, each displayed cue is a
representative positive attribute from one eight-attribute objective;
the visualization therefore complements, rather than replaces, the
complete grouping in Table~\ref{tab:celeba5-attribute-groups}.

\begin{figure}[t]
    \centering
    \includegraphics[width=0.96\textwidth]{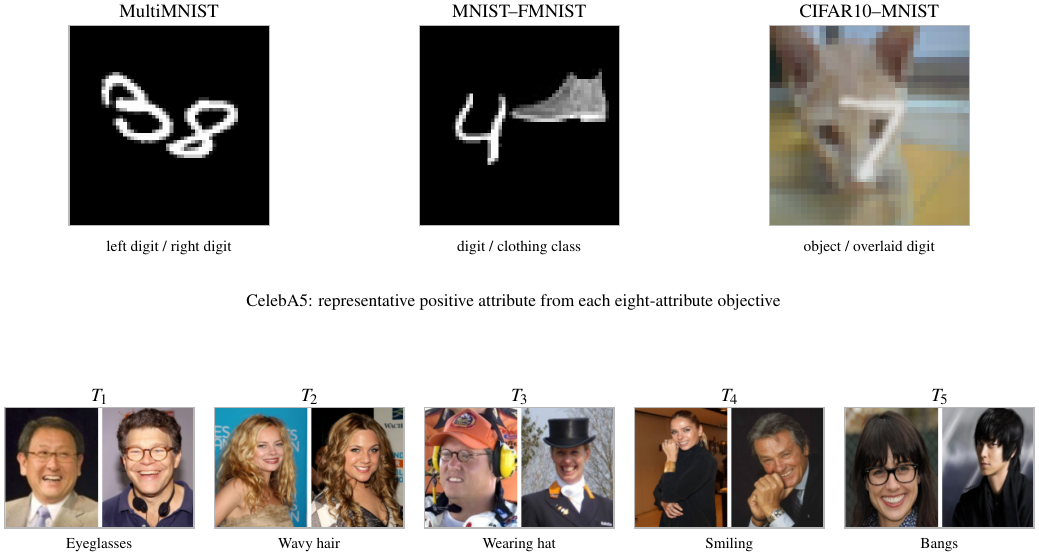}
    \caption{Dataset and benchmark examples. The first row illustrates
    the two shifted digits in \textsc{MultiMNIST}, the digit/clothing
    pair in \textsc{MNIST--FMNIST}, and the overlaid digit/object pair
    in \textsc{CIFAR10--MNIST}. The second row shows two positive
    \textsc{CelebA} examples for one representative attribute from each
    objective group: Eyeglasses ($T_1$), Wavy Hair ($T_2$), Wearing Hat
    ($T_3$), Smiling ($T_4$), and Bangs ($T_5$). Source datasets are
    described by \citet{lecun1998gradient,xiao2017fashion,krizhevsky2009learning,liu2015deep}.}
    \label{fig:dataset-examples}
\end{figure}

\begin{table}[H]
\centering
\caption{The five multi-label objectives used in \textsc{CelebA5}.}
\label{tab:celeba5-attribute-groups}
\small
\setlength{\tabcolsep}{5pt}
\begin{tabularx}{\textwidth}{c>{\raggedright\arraybackslash}X}
\toprule
Objective & Attributes \\
\midrule
$T_1$ & 5-o'Clock Shadow, Blond Hair, Bags Under Eyes, Eyeglasses, Mustache, Wavy Hair, Oval Face, Rosy Cheeks \\
$T_2$ & Arched Eyebrows, Black Hair, Big Nose, Blurry, Goatee, Straight Hair, Pale Skin, Wearing Earrings \\
$T_3$ & Attractive, Brown Hair, Bushy Eyebrows, Double Chin, Mouth Slightly Open, Receding Hairline, Wearing Hat, Young \\
$T_4$ & Bald, Big Lips, Chubby, Heavy Makeup, High Cheekbones, Male, Sideburns, Smiling \\
$T_5$ & Bangs, Gray Hair, Narrow Eyes, No Beard, Pointy Nose, Wearing Lipstick, Wearing Necklace, Wearing Necktie \\
\bottomrule
\end{tabularx}
\end{table}

\subsection{Objective and Metric Definitions}
\label{app:objective-definitions}
The reduction convention is part of the definition of FedHV because
rescaling an objective changes its slack and hence its adaptive weight.
For each two-objective benchmark and client $k$, task $i$ uses mean
multiclass cross-entropy,
\begin{equation}
f_i^{(k)}(\theta)
=\frac{1}{n_k}\sum_{j=1}^{n_k}
\operatorname{CE}\!\left(z_i(x_{k,j};\theta),y_{i,k,j}\right),
\label{eq:two-task-objectives}
\end{equation}
where $i$ denotes the left/right digit, digit/clothing, or
CIFAR-10/MNIST head, respectively. Task accuracy is the fraction of
examples whose top-1 prediction from that head is correct.

For \textsc{CelebA5}, let $G_i$ be the eight binary attributes assigned
to objective $i$ in Table~\ref{tab:celeba5-attribute-groups}. The client
objective is the mean binary cross-entropy over both examples and the
eight attributes,
\begin{equation}
f_i^{(k)}(\theta)
=\frac{1}{8n_k}\sum_{j=1}^{n_k}\sum_{a\in G_i}
\operatorname{BCE}\!\left(z_a(x_{k,j};\theta),y_{a,k,j}\right).
\label{eq:celeba-objectives}
\end{equation}
The corresponding task accuracy averages the thresholded binary
attribute accuracy over the same $8n_k$ predictions. Objective reports,
training references, MeanLoss, and evaluation HV all use these mean
reductions; no objective is summed over its classes or attributes.

\subsection{Model Architectures}
\label{app:model-architectures}
The model architecture is shared by all compared methods within a benchmark.
\textsc{MultiMNIST} and \textsc{MNIST--FMNIST} use a LeNet-style shared
encoder with $3\times3$ convolutions $1\!\to\!10\!\to\!15$, max pooling,
a $540\!\to\!50$ fully connected layer, and dropout $0.5$. Each task
head uses $50\!\to\!50\!\to\!10$ fully connected layers with dropout
$0.5$ and log-softmax output. \textsc{CIFAR10--MNIST} uses two
$3\times3$ $64$-channel convolutions with max pooling, followed by
$4096\!\to\!384\!\to\!192$ fully connected layers with dropout $0.5$;
each task decoder is $192\!\to\!100\!\to\!10$. \textsc{CelebA5} uses a
ResNet-18-style shared backbone with BasicBlock depths $[2,2,2,2]$ and
GroupNorm with two groups, producing a 512-dimensional representation;
each of the five decoders is a linear $512\!\to\!8$ sigmoid head. We
use the implementation's standard PyTorch initialization and no
pretrained weights.

\subsection{Federated Partitioning}
\label{app:partitioning}
Table~\ref{tab:partition-settings} lists the six main benchmark
settings. The partition seed is fixed to $10$ across methods within each
setting.

\begin{table}[H]
\centering
\caption{Federated partitioning and default participation in the six
main settings. The all-label strategy forms Dirichlet allocation strata
from all available task-label coordinates, whereas the first-label
strategy forms strata using only the first label coordinate. Smaller
$\alpha$ yields more concentrated client label distributions
\citep{hsu2019measuring}.}
\label{tab:partition-settings}
\small
\setlength{\tabcolsep}{5pt}
\begin{tabular}{lccccc}
\toprule
Setting & $N$ & $M$ & Samples/client & $\alpha$ & Partition rule \\
\midrule
\textsc{MultiMNIST} & 30 & 20 & 600 & 0.15 & all-label Dirichlet \\
\textsc{MNIST--FMNIST} & 30 & 10 & 500 & 0.30 & all-label Dirichlet \\
\textsc{CelebA5}-A & 30 & 10 & 500 & 0.15 & all-label Dirichlet \\
\textsc{CelebA5}-B & 30 & 10 & 500 & 0.30 & first-label Dirichlet \\
\textsc{CIFAR10--MNIST}-A & 30 & 10 & 500 & 0.10 & all-label Dirichlet \\
\textsc{CIFAR10--MNIST}-B & 30 & 10 & 500 & 0.30 & all-label Dirichlet \\
\bottomrule
\end{tabular}
\end{table}
The two \textsc{CelebA5} settings jointly vary partition strategy and
concentration: setting A uses all-label allocation with $\alpha=0.15$,
whereas setting B uses first-label allocation with $\alpha=0.30$. The
two \textsc{CIFAR10--MNIST} settings retain the same all-label strategy
at $\alpha\in\{0.10,0.30\}$.

\subsection{Training and Evaluation Protocol}
\label{app:training-protocol}
All methods use SGD and the same round budget and local-training
structure within a setting. Table~\ref{tab:shared-training-settings}
summarizes the shared quantities. No local learning-rate scheduler is
used.

\begin{table}[H]
\centering
\caption{Shared training and evaluation settings. $K$ is the number of local SGD steps per selected client.}
\label{tab:shared-training-settings}
\small
\setlength{\tabcolsep}{5pt}
\begin{tabular}{lccccc}
\toprule
Setting & Adaptive rounds & Batch size & $K$ & Momentum & Test period \\
\midrule
\textsc{MultiMNIST} $\alpha=0.15$ & 200 & 32 & 10 & 0 & 3 \\
\textsc{MNIST--FMNIST} $\alpha=0.30$ & 500 & 128 & 10 & 0 & 3 \\
\textsc{CelebA5} $\alpha=0.15$ & 300 & 64 & 5 & 0.9 & 5 \\
\textsc{CelebA5} $\alpha=0.30$ & 300 & 64 & 5 & 0.9 & 5 \\
\textsc{CIFAR10--MNIST} $\alpha=0.10$ & 500 & 128 & 10 & 0 & 3 \\
\textsc{CIFAR10--MNIST} $\alpha=0.30$ & 500 & 128 & 10 & 0 & 3 \\
\bottomrule
\end{tabular}
\end{table}
Final tables use the final scheduled test evaluation. Curves report the
mean across seeds with a $\pm1$ standard-deviation band. The common
evaluation hypervolume reference is
$\bm r^{\rm eval}=3\,\mathbf{1}_m$ for all methods and is distinct from
FedHV's training reference.

\subsection{Baseline Implementations and Learning Rates}
\label{app:baseline-configs}
We evaluate FSMGDA and FedCMOO using their defining algorithm-specific procedures under the common experimental protocol described above. FedCMOO uses its randomized-SVD estimator with additional upload dimension $1$, followed by $1000$ projected server-side weight-update iterations with step size $10^{-3}$ and zero momentum. FSMGDA performs objective-specific local updates and computes the server-side combination weights using the MinNorm/MGDA procedure.

All methods use the same model architecture, client partition, participation schedule, number of communication rounds, and local-update budget within each benchmark, as specified in Tables~\ref{tab:partition-settings} and \ref{tab:shared-training-settings}. Thus, the methods differ principally in how they estimate and update the objective weights and in the resulting computation and communication required by those procedures.

Uniform uses fixed weights $w_i=1/m$ throughout training and otherwise follows the same single-trajectory local optimization and model-aggregation structure as FedHV. It does not perform reference calibration or communicate objective reports for adaptive weighting.

We use method-specific optimization rates rather than imposing a single learning-rate pair across algorithms with different update mechanisms. Table~\ref{tab:method-learning-rates} reports the server and local learning rates, $\eta_g/\eta_\ell$, used for every method and benchmark. Each reported pair is held fixed across the three training seeds.

\begin{table}[H]
\centering
\caption{Server/local learning rates $\eta_g/\eta_\ell$ used in the
reported runs. Each pair is fixed across training seeds.}
\label{tab:method-learning-rates}
\small
\setlength{\tabcolsep}{5pt}
\begin{tabular}{lcccc}
\toprule
Setting & Uniform & FSMGDA & FedCMOO & FedHV \\
\midrule
\textsc{MultiMNIST} $\alpha=0.15$ & $1.0/.3$ & $2.0/.1$ & $1.2/.5$ & $1.2/.3$ \\
\textsc{MNIST--FMNIST} $\alpha=0.30$ & $1.2/.4$ & $2.0/.1$ & $1.2/.5$ & $1.6/.3$ \\
\textsc{CelebA5} $\alpha=0.15$ & $.8/.3$ & $1.6/.3$ & $1.6/.3$ & $1.0/.2$ \\
\textsc{CelebA5} $\alpha=0.30$ & $.8/.3$ & $1.6/.3$ & $1.6/.3$ & $1.0/.2$ \\
\textsc{CIFAR10--MNIST} $\alpha=0.10$ & $1.2/.4$ & $2.0/.1$ & $1.6/.3$ & $1.6/.3$ \\
\textsc{CIFAR10--MNIST} $\alpha=0.30$ & $1.2/.4$ & $1.6/.3$ & $1.6/.3$ & $1.6/.3$ \\
\bottomrule
\end{tabular}
\end{table}

\subsection{FedHV Hyperparameters}
\label{app:fedhv-hyperparameters}
Unless an ablation changes it, FedHV uses full-client objective reports
($B_{\rm rep}=\mathrm{all}$), no update normalization, and numerical
estimated-slack floor $\rho_{\rm clip}=10^{-6}$. The weight vector is
held fixed over all $K$ local steps. The default $M$ and $K$ values are
those in Tables~\ref{tab:partition-settings} and
\ref{tab:shared-training-settings}. The objective-report budget
ablation replaces the full-client report by $B_{\rm rep}\in\{1,2,5\}$
minibatches without changing the local training minibatches.

\subsection{Reference Calibration}
\label{app:reference-selection}
The fixed FedHV training reference is constructed before adaptive
training. For \textsc{MultiMNIST} and \textsc{MNIST--FMNIST}, a single
$20$-round uniform-weight calibration run is used. For the
\textsc{CIFAR10--MNIST} family, one $20$-round calibration run with the
all-label $\alpha=0.30$ partition is used and the resulting reference is
reused for both $\alpha=0.10$ and $\alpha=0.30$; hence the reference is not
retuned to the lower-$\alpha$ setting. In each case we take the
componentwise maximum reported calibration loss and add margin
$\delta_r=.15$.

For \textsc{CelebA5}, the robust reference used in the reported runs is
constructed from three independent $20$-round calibration runs with
seeds $\{0,42,2026\}$ under the first-label $\alpha=0.30$ calibration
configuration. We take the componentwise maximum over all three runs and
add $\delta_r=.30$. This single reference is then reused for both
\textsc{CelebA5} benchmark settings and all adaptive-training seeds.

\begin{table}[H]
\centering
\caption{Frozen training references used by FedHV. These are optimization references $\bm r^{\rm train}$ and are distinct from the common evaluation reference $\bm r^{\rm eval}=3\,\mathbf{1}_m$.}
\label{tab:training-references}
\small
\begin{tabular}{ll}
\toprule
Benchmark family & $\bm r^{\rm train}$ \\
\midrule
\textsc{MultiMNIST} & $(2.458940,\;2.453542)$ \\
\textsc{MNIST--FMNIST} & $(2.454883,\;2.464302)$ \\
\textsc{CIFAR10--MNIST} & $(2.452658,\;2.466771)$ \\
\textsc{CelebA5} & $(1.297312,\;1.023714,\;1.060341,\;1.108676,\;1.085493)$ \\
\bottomrule
\end{tabular}
\end{table}
Only the resulting reference is reused: adaptive models are restarted
from their normal initialization rather than continued from calibration,
and all reported adaptive-round curves exclude calibration. Reference
construction is therefore a one-time preprocessing cost that baselines
do not require. It costs $20$ calibration rounds for each two-task
benchmark family; the robust \textsc{CelebA5} reference used here costs
three $20$-round calibration runs ($60$ sequential-round equivalents),
shared across both CelebA5 settings and all adaptive seeds. This setup
cost is distinct from our within-round communication claim: every
adaptive FedHV round still uses one downlink and one uplink with no
additional synchronization stage.

The calibration rule is empirical and does not by itself certify the
true safe-reference condition in Assumption~\ref{as:slack}. The
numerical floor $\rho_{\rm clip}$ stabilizes estimated slacks but
likewise does not establish that theoretical condition. Table
\ref{tab:reference-diagnostics} reports the closest diagnostic available
from the retained logs: the minimum raw slack formed from the aggregated
server objective report. All recorded standard-run report slacks remain
positive and clipping is inactive. These values are not true full-client
training losses at every local proof point, so they cannot certify
Assumption~\ref{as:slack}. The current runs also use only
$\delta_r\in\{0.15,0.30\}$ as described above; they do not constitute a
controlled reference-margin sensitivity sweep. Theorem
\ref{thm:reference-safety-certificate} provides a separate conditional
route from a verified initial margin to trajectory safety, but its
bounded-update and budget inequalities were not evaluated for these
runs.

\begin{table}[H]
\centering
\caption{Reference diagnostics over all adaptive rounds and three
training seeds. ``Minimum report slack'' is
$\min_{t,i}(r_i-\widehat f_i^{,t})$ for the server-aggregated report;
it is a practical diagnostic rather than the true-slack condition in
Assumption~\ref{as:slack}.}
\label{tab:reference-diagnostics}
\small
\begin{tabular}{lcc}
\toprule
Setting & Minimum report slack & Floor activation \\
\midrule
\textsc{MultiMNIST} & $0.1463$ & $0/1200$ \\
\textsc{MNIST--FMNIST} & $0.1413$ & $0/3000$ \\
\textsc{CelebA5}-A & $0.2981$ & $0/4500$ \\
\textsc{CelebA5}-B & $0.3000$ & $0/4500$ \\
\textsc{CIFAR10--MNIST}-A & $0.0282$ & $0/3000$ \\
\textsc{CIFAR10--MNIST}-B & $0.1412$ & $0/3000$ \\
\bottomrule
\end{tabular}
\end{table}

\subsection{Random Seeds and Reproducibility}
\label{app:reproducibility}
The client partition is generated once with partition seed $10$ and is
held fixed across methods within a setting. Adaptive training is
repeated with run seeds $0$, $42$, and $2026$. We report the arithmetic
mean and sample standard deviation across these three runs. GPU and
framework operations are not assumed fully deterministic, so small
sub-standard-deviation differences are not interpreted as statistically
significant. The three ablation sweeps were launched as independent
experiment suites; consequently, nominally identical default variants
(e.g., $K=10$, $M=10$, or $B_{\rm rep}=\mathrm{all}$) can differ by a
few tenths of a percentage point across sweeps. We therefore interpret
comparisons within each controlled sweep and do not treat such small
cross-sweep differences as effects. All compared methods use the same
evaluation schedule and test data within each setting.

\section{Complete Experimental Results}
\label{app:complete-results}

\subsection{Loss and Hypervolume Results}
\label{app:loss-hv-results}
\begin{table}[H]
\centering
\caption{Final loss-based results for Uniform and the three adaptive
FedMOO methods,
mean $\pm$ standard deviation over three seeds. Lower MeanLoss is
better; higher HV is better. HV uses the common evaluation reference
$\bm r^{\rm eval}=3\,\mathbf{1}_m$, not FedHV's training reference. Bold
marks the best result and underline the second-best result within each
setting.}
\label{tab:loss-hv-results}
\small
\setlength{\tabcolsep}{6pt}
\begin{tabular}{llcc}
\toprule
Setting & Method & MeanLoss $\downarrow$ & HV $\uparrow$ \\
\midrule
\textsc{MultiMNIST} & Uniform & \underline{$0.328\pm0.008$} & \underline{$7.14\pm0.04$} \\
 & FSMGDA & $0.359\pm0.032$ & $6.97\pm0.17$ \\
 & FedCMOO & $0.354\pm0.009$ & $7.00\pm0.05$ \\
 & FedHV & \textbf{\boldmath$0.312\pm0.019$} & \textbf{\boldmath$7.23\pm0.10$} \\
\addlinespace[1.5pt]
\textsc{MNIST--FMNIST} & Uniform & \underline{$0.379\pm0.006$} & \underline{$6.82\pm0.03$} \\
 & FSMGDA & $0.427\pm0.020$ & $6.57\pm0.10$ \\
 & FedCMOO & \textbf{\boldmath$0.372\pm0.007$} & \textbf{\boldmath$6.87\pm0.04$} \\
 & FedHV & $0.382\pm0.014$ & $6.81\pm0.07$ \\
\addlinespace[1.5pt]
\textsc{CelebA5}-A & Uniform & \textbf{\boldmath$0.267\pm0.003$} & \textbf{\boldmath$152.32\pm0.77$} \\
 & FSMGDA & $0.295\pm0.003$ & $144.94\pm0.85$ \\
 & FedCMOO & $0.302\pm0.005$ & $142.96\pm1.32$ \\
 & FedHV & \underline{$0.269\pm0.004$} & \underline{$151.78\pm1.05$} \\
\addlinespace[1.5pt]
\textsc{CelebA5}-B & Uniform & \textbf{\boldmath$0.271\pm0.002$} & \textbf{\boldmath$151.40\pm0.48$} \\
 & FSMGDA & $0.318\pm0.004$ & $138.74\pm0.94$ \\
 & FedCMOO & $0.342\pm0.038$ & $132.92\pm9.31$ \\
 & FedHV & \underline{$0.273\pm0.001$} & \underline{$150.92\pm0.25$} \\
\addlinespace[1.5pt]
\textsc{CIFAR10--MNIST}-A & Uniform & \underline{$0.821\pm0.013$} & \underline{$4.37\pm0.07$} \\
 & FSMGDA & \textbf{\boldmath$0.812\pm0.007$} & \textbf{\boldmath$4.44\pm0.05$} \\
 & FedCMOO & $0.857\pm0.027$ & $4.14\pm0.17$ \\
 & FedHV & $0.897\pm0.019$ & $3.98\pm0.11$ \\
\addlinespace[1.5pt]
\textsc{CIFAR10--MNIST}-B & Uniform & \underline{$0.766\pm0.013$} & \underline{$4.67\pm0.06$} \\
 & FSMGDA & $0.798\pm0.017$ & $4.48\pm0.09$ \\
 & FedCMOO & \textbf{\boldmath$0.753\pm0.008$} & \textbf{\boldmath$4.72\pm0.03$} \\
 & FedHV & $0.826\pm0.020$ & $4.37\pm0.10$ \\
\bottomrule
\end{tabular}
\end{table}
The CIFAR rows make the metric trade-off explicit: FedHV improves
mean and worst-task accuracy (Table~\ref{tab:main-results}) without
also optimizing the common-reference loss HV.

\FloatBarrier
\subsection{Per-Task Final Metrics}
\label{app:per-task-results}
\begin{table}[H]
\centering
\caption{Final task-level test performance for the four two-objective
settings, mean $\pm$ standard deviation over three seeds. Accuracies
are percentages. Bold is best and underline is second best in each
metric among FSMGDA, FedCMOO, and FedHV within a setting.}
\label{tab:per-task-final}
\scriptsize
\setlength{\tabcolsep}{3.2pt}
\begin{tabular}{llcccc}
\toprule
Setting & Method & Task 1 Acc. & Task 2 Acc. & Task 1 Loss & Task 2 Loss \\
\midrule
\textsc{MultiMNIST} $\alpha=0.15$ & FSMGDA & \underline{$91.41\pm0.31$} & $86.15\pm2.53$ & \underline{$0.292\pm0.012$} & $0.425\pm0.056$ \\
 & FedCMOO & $90.99\pm0.41$ & \underline{$88.79\pm0.57$} & $0.339\pm0.025$ & \underline{$0.369\pm0.008$} \\
 & FedHV & \textbf{\boldmath$92.50\pm0.68$} & \textbf{\boldmath$88.94\pm0.61$} & \textbf{\boldmath$0.257\pm0.018$} & \textbf{\boldmath$0.366\pm0.021$} \\
\addlinespace[1.5pt]
\textsc{MNIST--FMNIST} $\alpha=0.30$ & FSMGDA & $94.11\pm0.81$ & $75.58\pm1.05$ & $0.204\pm0.021$ & $0.650\pm0.020$ \\
 & FedCMOO & \underline{$94.81\pm0.43$} & \textbf{\boldmath$78.58\pm0.66$} & \underline{$0.186\pm0.019$} & \textbf{\boldmath$0.558\pm0.017$} \\
 & FedHV & \textbf{\boldmath$95.67\pm0.31$} & \underline{$77.03\pm0.70$} & \textbf{\boldmath$0.161\pm0.018$} & \underline{$0.602\pm0.012$} \\
\addlinespace[1.5pt]
\textsc{CIFAR10--MNIST} $\alpha=0.10$ & FSMGDA & \underline{$51.96\pm0.55$} & \underline{$93.56\pm0.25$} & \textbf{\boldmath$1.405\pm0.022$} & \underline{$0.218\pm0.009$} \\
 & FedCMOO & $47.98\pm1.15$ & \textbf{\boldmath$95.10\pm0.34$} & \underline{$1.529\pm0.064$} & \textbf{\boldmath$0.184\pm0.011$} \\
 & FedHV & \textbf{\boldmath$53.03\pm0.35$} & $93.34\pm0.08$ & $1.561\pm0.041$ & $0.232\pm0.003$ \\
\addlinespace[1.5pt]
\textsc{CIFAR10--MNIST} $\alpha=0.30$ & FSMGDA & $51.85\pm1.00$ & \underline{$94.36\pm0.32$} & \underline{$1.408\pm0.028$} & \underline{$0.188\pm0.007$} \\
 & FedCMOO & \underline{$54.23\pm0.88$} & \textbf{\boldmath$95.53\pm0.16$} & \textbf{\boldmath$1.328\pm0.008$} & \textbf{\boldmath$0.177\pm0.011$} \\
 & FedHV & \textbf{\boldmath$57.55\pm0.20$} & $93.38\pm0.13$ & $1.423\pm0.031$ & $0.228\pm0.009$ \\
\bottomrule
\end{tabular}
\end{table}

\begin{figure}[t]
    \centering
    \includegraphics[width=0.88\textwidth]{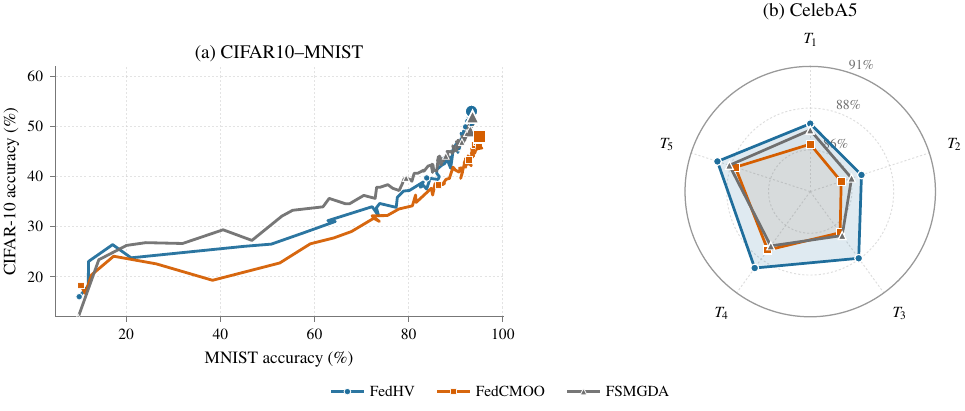}
    \caption{Alternate task-level views using the same visual encoding
    as Figure~\ref{fig:main-task-profiles}. \textbf{(a)} Round-by-round
    training trajectories for \textsc{CIFAR10--MNIST}-A (all-label
    Dirichlet, $\alpha=0.10$), with MNIST accuracy on the horizontal
    axis and CIFAR-10 accuracy on the vertical axis. Each path follows
    scheduled evaluation checkpoints from early to late training; the
    terminal cross shows $\pm1$ standard deviation on both objectives.
    \textbf{(b)} Mean five-objective radar profile for \textsc{CelebA5}-A
    (all-label Dirichlet, $\alpha=0.15$). The radial scale is truncated
    to $83$--$91\%$ for legibility; task-wise uncertainty is reported
    in Table~\ref{tab:per-task-final}. The main paper reports the
    corresponding B settings.}
    \label{fig:appendix-task-profiles}
\end{figure}

\FloatBarrier
\subsection{Aggregate Learning Curves}
\label{app:aggregate-curves}
Each curve below is the mean over three seeds; shaded bands show
$\pm1$ standard deviation. The following figures show the optimization
trajectories of FedHV, FSMGDA, and FedCMOO across communication rounds.

\begin{figure}[tbp]
    \centering
    \includegraphics[width=0.90\textwidth]{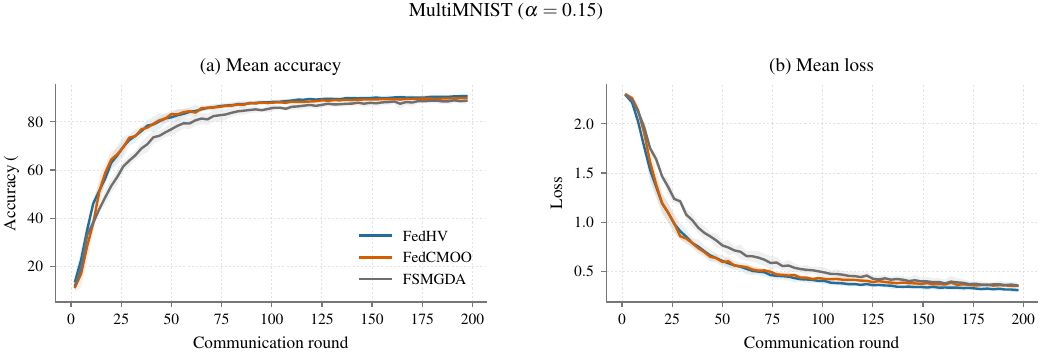}
    \caption{Aggregate test accuracy and loss for \textsc{MultiMNIST} $\alpha=0.15$ across adaptive communication rounds.}
    \label{fig:curve-multimnist-alpha015}
\end{figure}

\begin{figure}[tbp]
    \centering
    \includegraphics[width=0.90\textwidth]{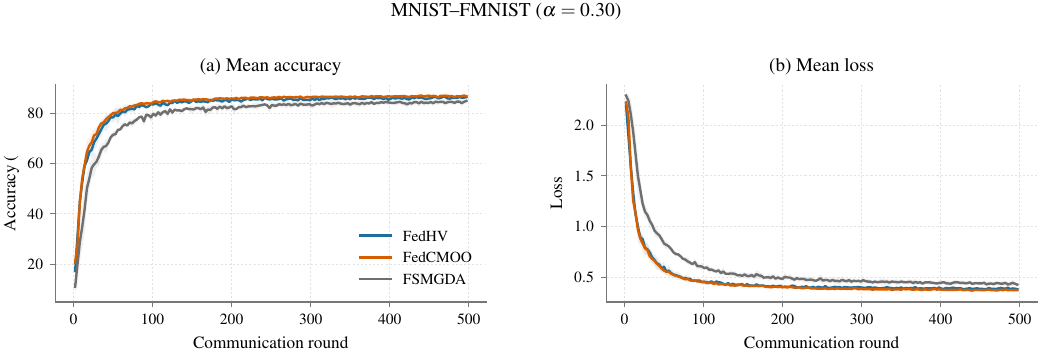}
    \caption{Aggregate test accuracy and loss for \textsc{MNIST--FMNIST} $\alpha=0.30$ across adaptive communication rounds.}
    \label{fig:curve-mnist-fmnist-alpha03}
\end{figure}

\begin{figure}[tbp]
    \centering
    \includegraphics[width=0.90\textwidth]{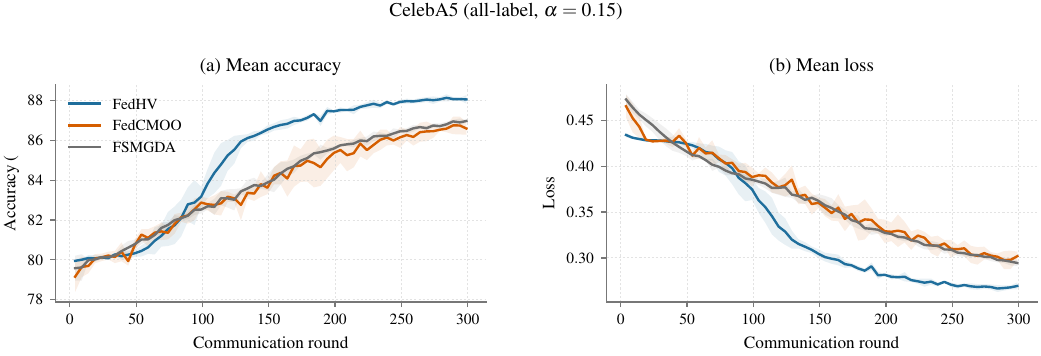}
    \caption{Aggregate test accuracy and loss for \textsc{CelebA5} $\alpha=0.15$ across adaptive communication rounds.}
    \label{fig:curve-celeba5-alpha015}
\end{figure}

\begin{figure}[tbp]
    \centering
    \includegraphics[width=0.90\textwidth]{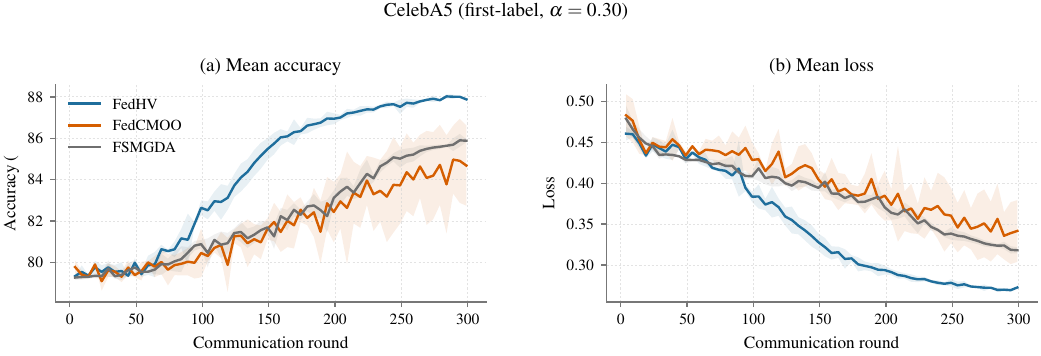}
    \caption{Aggregate test accuracy and loss for \textsc{CelebA5} $\alpha=0.30$ across adaptive communication rounds.}
    \label{fig:curve-celeba5-alpha03}
\end{figure}

\begin{figure}[tbp]
    \centering
    \includegraphics[width=0.90\textwidth]{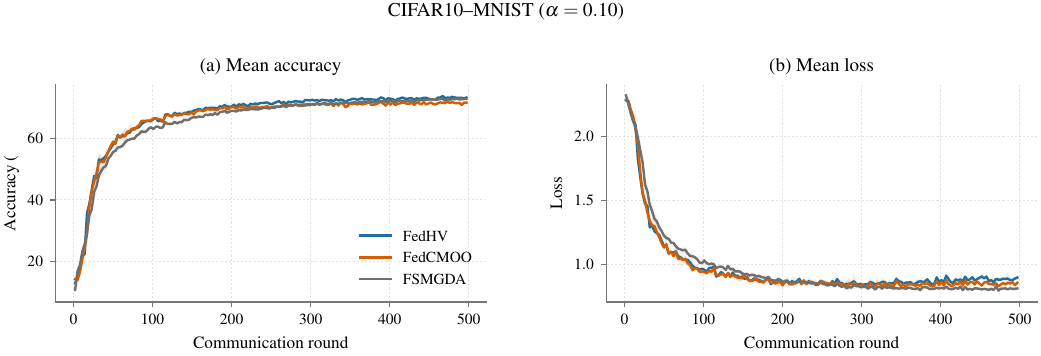}
    \caption{Aggregate test accuracy and loss for \textsc{CIFAR10--MNIST} $\alpha=0.10$ across adaptive communication rounds.}
    \label{fig:curve-cifar10-mnist-alpha01}
\end{figure}

\begin{figure}[tbp]
    \centering
    \includegraphics[width=0.90\textwidth]{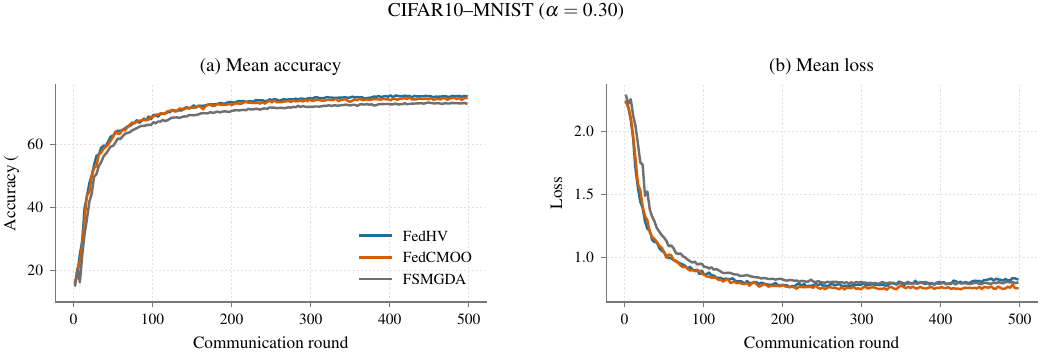}
    \caption{Aggregate test accuracy and loss for \textsc{CIFAR10--MNIST} $\alpha=0.30$ across adaptive communication rounds.}
    \label{fig:curve-cifar10-mnist-alpha03}
\end{figure}

\FloatBarrier
\subsection{Per-Task Learning Curves}
\label{app:per-task-curves}
For the two-objective benchmarks, the following figures separate the
accuracy and loss trajectory of each task. This avoids hiding a
difficult objective behind the mean.

\begin{figure}[H]
    \centering
    \includegraphics[width=0.90\textwidth]{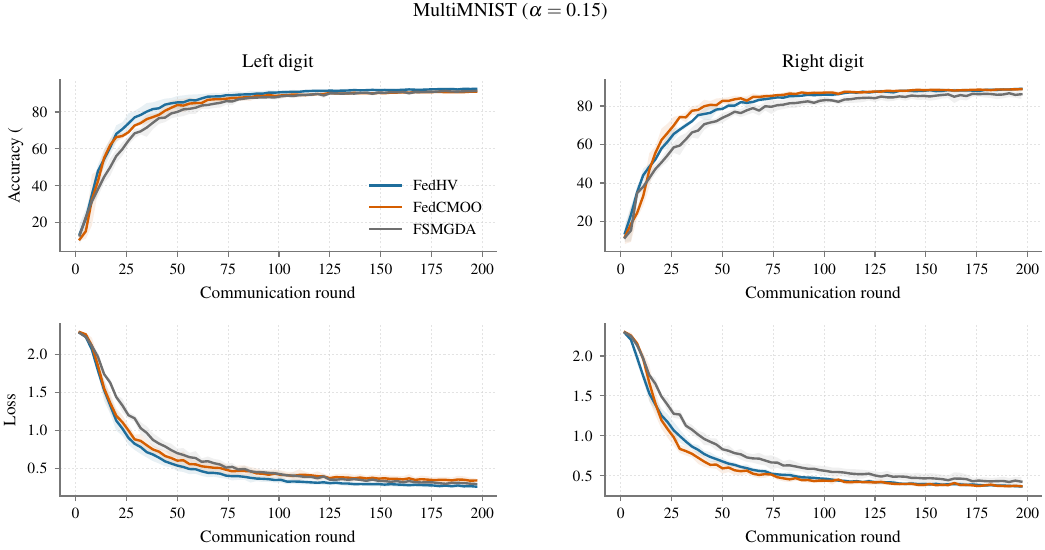}
    \caption{Per-task test accuracy and loss for \textsc{MultiMNIST} $\alpha=0.15$, mean $\pm$ standard deviation over three seeds.}
    \label{fig:per-task-multimnist-alpha015}
\end{figure}

\begin{figure}[tbp]
    \centering
    \includegraphics[width=0.90\textwidth]{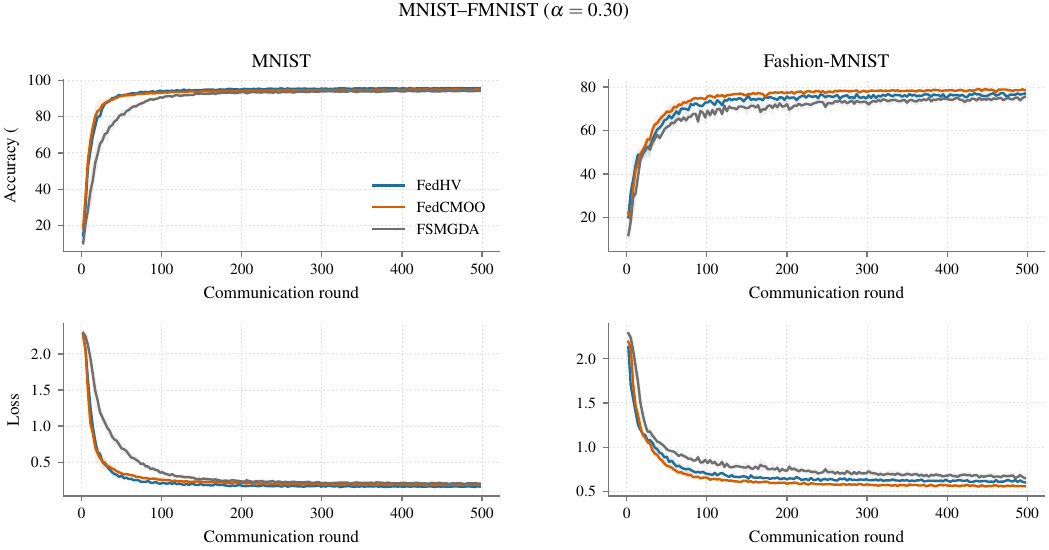}
    \caption{Per-task test accuracy and loss for \textsc{MNIST--FMNIST} $\alpha=0.30$, mean $\pm$ standard deviation over three seeds.}
    \label{fig:per-task-mnist-fmnist-alpha03}
\end{figure}

\begin{figure}[tbp]
    \centering
    \includegraphics[width=0.90\textwidth]{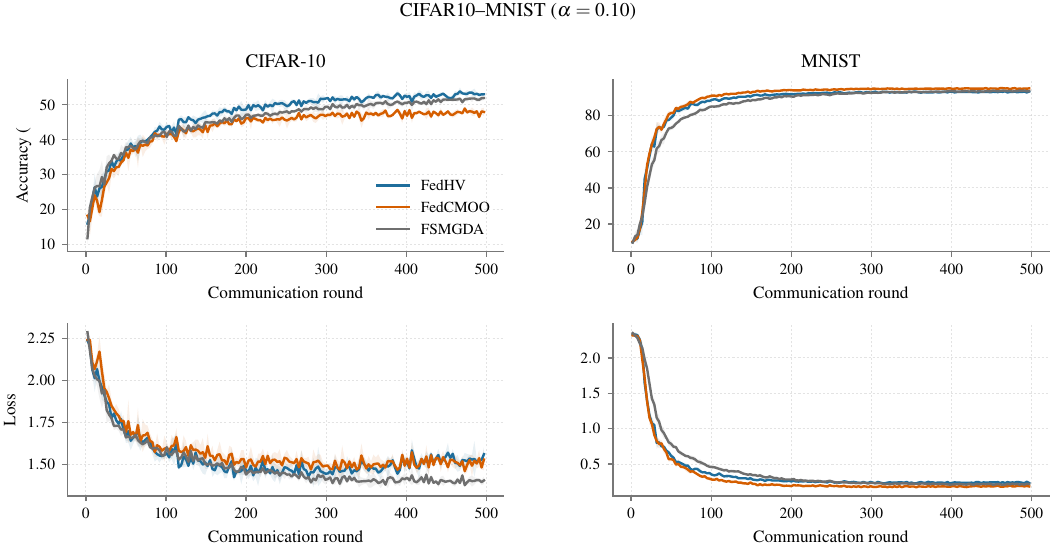}
    \caption{Per-task test accuracy and loss for \textsc{CIFAR10--MNIST} $\alpha=0.10$, mean $\pm$ standard deviation over three seeds.}
    \label{fig:per-task-cifar10-mnist-alpha01}
\end{figure}

\begin{figure}[tbp]
    \centering
    \includegraphics[width=0.90\textwidth]{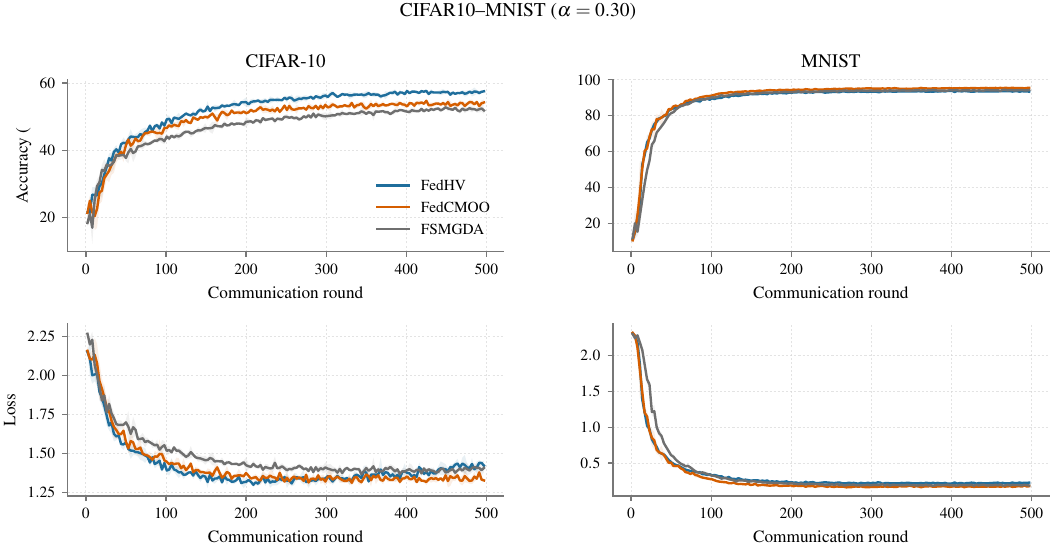}
    \caption{Per-task test accuracy and loss for \textsc{CIFAR10--MNIST} $\alpha=0.30$, mean $\pm$ standard deviation over three seeds.}
    \label{fig:per-task-cifar10-mnist-alpha03}
\end{figure}

\FloatBarrier
\subsection{FedHV Ablations}
\label{app:ablation-results}
\begin{table}[H]
\centering
\caption{FedHV ablations on \textsc{CIFAR10--MNIST} $\alpha=0.30$,
mean $\pm$ standard deviation over three seeds. Bold marks the best
value and underline the second-best value within each ablation sweep;
loss uses four decimals to resolve rounding ties.}
\label{tab:ablation-full}
\small
\begin{tabular}{llcc}
\toprule
Ablation & Variant & MeanAcc $\uparrow$ & MeanLoss $\downarrow$ \\
\midrule
Report budget & $B_{\rm rep}=1$ & $75.52\pm0.19$ & $0.8314\pm0.0078$ \\
 & $B_{\rm rep}=2$ & \underline{$75.54\pm0.45$} & \underline{$0.8279\pm0.0151$} \\
 & $B_{\rm rep}=5$ & \textbf{\boldmath$75.66\pm0.18$} & \textbf{\boldmath$0.8274\pm0.0148$} \\
 & $B_{\rm rep}=\mathrm{all}$ & $75.43\pm0.43$ & $0.8332\pm0.0203$ \\
\addlinespace[1.5pt]
Participation & $M=5/30$ & $73.96\pm0.93$ & $0.8639\pm0.0346$ \\
 & $M=10/30$ & $75.41\pm0.24$ & $0.8396\pm0.0101$ \\
 & $M=20/30$ & \underline{$76.13\pm0.13$} & \underline{$0.8166\pm0.0050$} \\
 & $M=30/30$ & \textbf{\boldmath$76.27\pm0.27$} & \textbf{\boldmath$0.8097\pm0.0112$} \\
\addlinespace[1.5pt]
Local steps & $K=1$ & $61.98\pm1.59$ & $1.0522\pm0.0373$ \\
 & $K=5$ & $74.82\pm0.26$ & \textbf{\boldmath$0.7455\pm0.0055$} \\
 & $K=10$ & \textbf{\boldmath$75.61\pm0.36$} & \underline{$0.8330\pm0.0175$} \\
 & $K=20$ & \underline{$75.18\pm0.50$} & $0.9778\pm0.0252$ \\
\bottomrule
\end{tabular}
\end{table}
The report-budget ablation demonstrates that FedHV remains effective
with inexpensive objective estimation: one to five reporting
minibatches closely track full-client reporting, supporting sparse
objective reports in the tested regime. Increasing participation
improves both accuracy and loss as broader client coverage reduces the
sampling channels represented by $\chi_N(M)$ in the analysis. The
local-step sweep exhibits the expected progress--drift trade-off:
moderate $K$ supplies enough local optimization to avoid undertraining,
whereas larger $K$ strengthens client drift after accuracy has
saturated. Together, these results motivate the default reporting,
participation, and local-computation budgets used in the main study.

\begin{figure}[tbp]
\centering
\includegraphics[width=0.90\textwidth]{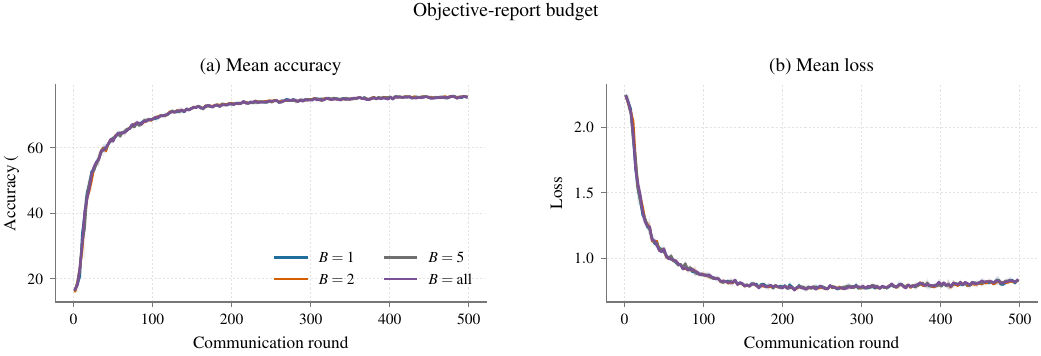}
\caption{Objective-report-budget learning curves.}
\label{fig:ablation-report-curves}
\end{figure}
\begin{figure}[tbp]
\centering
\includegraphics[width=0.90\textwidth]{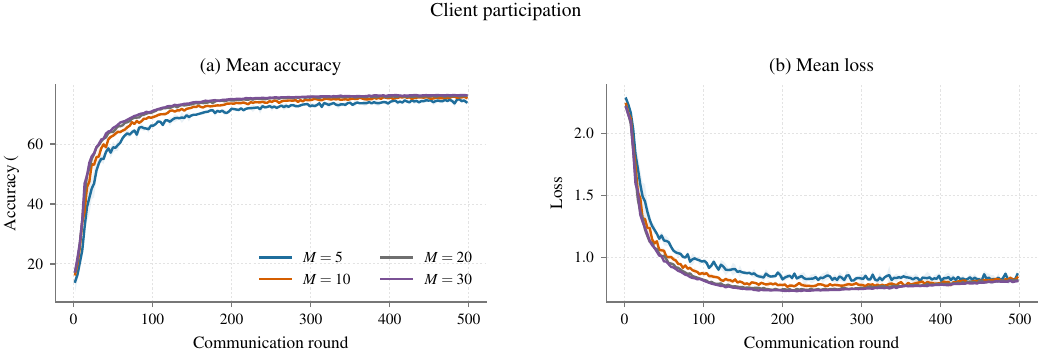}
\caption{Client-participation learning curves.}
\label{fig:ablation-participation-curves}
\end{figure}
\begin{figure}[tbp]
\centering
\includegraphics[width=0.90\textwidth]{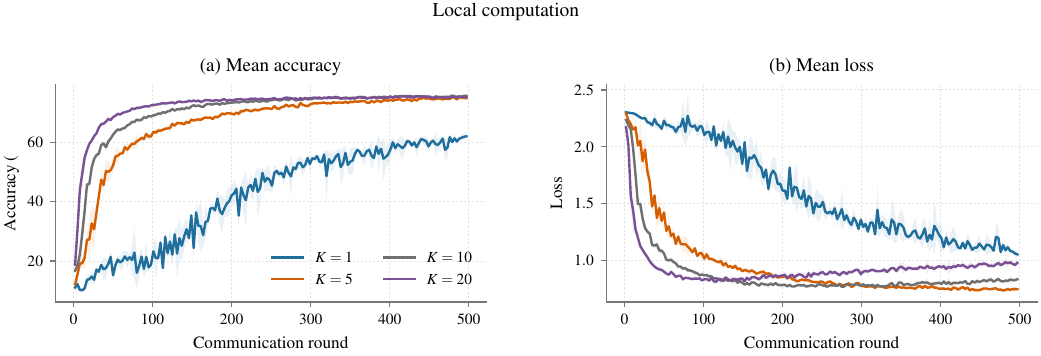}
\caption{Local-step learning curves.}
\label{fig:ablation-k-curves}
\end{figure}

\FloatBarrier
\subsection{FedHV Weight Trajectories}
\label{app:weight-trajectories}
\begin{figure}[!t]
\centering
\includegraphics[width=0.90\textwidth]{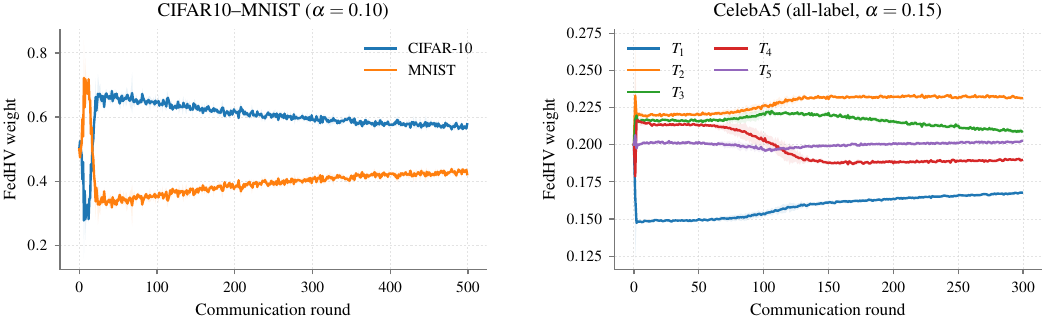}
\caption{FedHV task weights over adaptive rounds, mean $\pm$ standard
 deviation across three seeds. On \textsc{CIFAR10--MNIST}
 $\alpha=0.10$, FedHV assigns greater late-stage weight to the harder
 CIFAR-10 task; on \textsc{CelebA5}, mass remains distributed over all
 five objectives rather than collapsing to one task.}
\label{fig:fedhv-weight-trajectories}
\end{figure}

\FloatBarrier
\subsection{Additional Diagnostics}
\label{app:diagnostics}
The estimated-slack floor is inactive in the standard benchmark runs
and in the $B_{\rm rep}$ and participation ablations. The $K=1$
configuration is the exception: its diagnostic floor-activation rate is
approximately $0.2\%$, with minimum stabilized slack $10^{-6}$, and its
task weights can become nearly degenerate. This is a diagnostic of the
\emph{estimated} slack stabilization and does not establish failure of
the true safe-reference condition in Assumption~\ref{as:slack}.

\begin{table}[H]
\centering
\caption{Mean accuracy averaged over the final ten scheduled test
evaluations, then aggregated over three seeds. This diagnostic reduces
sensitivity to one final checkpoint; bold is best and underline is
second best among FSMGDA, FedCMOO, and FedHV within each setting.}
\label{tab:last10}
\small
\begin{tabular}{lccc}
\toprule
Setting & FSMGDA & FedCMOO & FedHV \\
\midrule
\textsc{MultiMNIST} $\alpha=0.15$ & $88.60\pm1.07$ & \underline{$89.66\pm0.08$} & \textbf{\boldmath$90.38\pm0.60$} \\
\textsc{MNIST--FMNIST} $\alpha=0.30$ & $84.39\pm1.01$ & \textbf{\boldmath$86.75\pm0.21$} & \underline{$86.13\pm0.17$} \\
\textsc{CelebA5} $\alpha=0.15$ & \underline{$86.75\pm0.11$} & $86.50\pm0.32$ & \textbf{\boldmath$88.05\pm0.11$} \\
\textsc{CelebA5} $\alpha=0.30$ & \underline{$85.55\pm0.18$} & $84.43\pm1.92$ & \textbf{\boldmath$87.87\pm0.06$} \\
\textsc{CIFAR10--MNIST} $\alpha=0.10$ & \underline{$72.61\pm0.30$} & $71.44\pm0.19$ & \textbf{\boldmath$73.13\pm0.22$} \\
\textsc{CIFAR10--MNIST} $\alpha=0.30$ & $73.19\pm0.35$ & \underline{$74.67\pm0.16$} & \textbf{\boldmath$75.41\pm0.29$} \\
\bottomrule
\end{tabular}
\end{table}
Averaging the last ten scheduled evaluations preserves the main
qualitative MeanAcc story: FedHV remains strongest in five of six
settings against FSMGDA and FedCMOO, while FedCMOO remains ahead on
\textsc{MNIST--FMNIST}. On \textsc{CIFAR10--MNIST}, the accuracy and
cross-entropy trajectories provide complementary views of the learned
trade-off, since the two metrics emphasize prediction quality and
surrogate optimization, respectively.

\section{Complexity Comparison}
\label{app:complexity}

Let \(d\) denote the model dimension, \(m\) the number of
objectives, \(M=|S_t|\) the number of participating clients,
and \(K\) the number of local optimization steps.

\paragraph{FedHV overhead.}
FedHV uses one server-to-client downlink and one client-to-server
uplink per communication round. Specifically, the server sends
\((\thetav^t,\wv^t)\) to each selected client in a single downlink, 
and client \(k\) returns \((\Delta\thetav_k^t,\widehat{\bm F}_k^{\,t})\)
in a single uplink. The objective estimates collected in round \(t\)
are used only to prepare \(\wv^{t+1}\), so no intermediate
loss-reporting round followed by a second weight-broadcast round is
required.

A standard model exchange communicates \(d\) model scalars in each
direction, whereas FedHV additionally communicates \(m\) task weights
in the downlink and \(m\) objective values in the uplink. Hence the total
per-participating-client payload is \(2d+2m\) scalars per round, 
compared with \(2d\) for the model exchange alone. Equivalently, the 
communication complexity is \(\Theta(d+m)=\Theta(d)\) when \(m\ll d\).

Including reference construction changes the end-to-end accounting.
Let $T_{\rm cal}$ be the number of sequential calibration-round
equivalents and $M_{\rm cal}$ its participating-client count. The total
FedHV payload over calibration and adaptive training is
\(2d\,T_{\rm cal}M_{\rm cal}+(2d+2m)TM\) scalars, before protocol-specific 
baseline auxiliaries. Here $T_{\rm cal}=20$ for each two-task benchmark family 
and $T_{\rm cal}=60$ for the robust \textsc{CelebA5} reference. Thus the
calibration cost is included in the end-to-end budget even though it is
excluded from adaptive-round learning curves.

\begin{table*}[!t]
\centering
\caption{
Dominant per-round complexity of the compared methods.
Here, \(d\) is the model dimension, \(m\) is the number of
objectives, and \(M\) is the number of participating clients.
Empirical transmission values reflect the \textsc{CIFAR10--MNIST} benchmark.
}
\label{tab:complexity_comparison}
\renewcommand{\arraystretch}{1.15}
\setlength{\tabcolsep}{5pt}
\begin{tabularx}{\textwidth}{
@{}
>{\raggedright\arraybackslash}p{1.65cm}
>{\raggedright\arraybackslash}X
>{\centering\arraybackslash}p{3.5cm}
>{\raggedright\arraybackslash}X
@{}
}
\toprule
\textbf{Method} & \textbf{Client-side procedure} & \textbf{Per-client communication} & \textbf{Additional server-side procedure} \\
\midrule
FSMGDA & Separate objective-specific local updates & \(\Theta(md)\) \newline (Empirical: $\sim 205.5$ MB) & Task-update geometry and multi-gradient optimization \\
Uniform & One uniformly weighted local trajectory & \(\Theta(d)\) \newline (Empirical: $\sim 138.1$ MB) & None beyond model aggregation \\
FedCMOO & Task-gradient estimation and weighted local training & \(\Theta(d)\) \newline (Empirical: $\sim 205.5$ MB) & Approximate Gram matrix construction and iterative weight updates \\
FedHV & Objective estimation and one weighted local trajectory & \(\Theta(d+m)=\Theta(d)\) \newline (Empirical: $\sim 138.1$ MB) & \(O(Mm)\) objective aggregation and closed-form inverse-slack weighting \\
\bottomrule
\end{tabularx}
\end{table*}

\paragraph{Comparison with FSMGDA.}
FSMGDA generates and communicates objective-specific local
updates, leading to communication that scales with \(md\)
\citep{yang2023fmol}. FedHV instead maintains one
local model trajectory and communicates only an
\(m\)-dimensional objective vector in addition to the
ordinary model update.

\paragraph{Comparison with FedCMOO.}
The communication-efficient version of FedCMOO uses
randomized low-rank estimation to obtain model-order
communication complexity \citep{askin2025fedcmoo}. However, it still requires
task-gradient information, an approximate task-gradient
Gram matrix, and iterative server-side weight optimization.
While FedCMOO successfully reduces communication from $\mathcal{O}(md)$ to $\mathcal{O}(d)$ via randomized-SVD, our empirical profiling reveals that uploading the compressed task-geometry matrix effectively doubles the actual uplink payload ($\sim 136.50$ MB compared to the $\sim 69.06$ MB base model update). FedHV avoids gradient-geometry estimation entirely, requiring only $80$ bytes of auxiliary scalar reports on the uplink, thereby achieving the theoretical $\mathcal{O}(d)$ bound with a practical multiplier of $1$.

\paragraph{Empirical Efficiency \& Synchronization.}
Protocol profiling further reveals that FedCMOO requires two distinct server-client synchronization stages per round to negotiate update geometry and weight projections. FedHV resolves weights entirely from the previous round's one-dimensional reports, operating seamlessly within a single standard synchronization stage. In high-latency federated environments, avoiding this secondary barrier is critical. 

Table~\ref{tab:empirical_efficiency} and Figure~\ref{fig:wall_clock_time} summarize the resulting end-to-end efficiency. Despite including the upfront penalty of its 20-round calibration phase, FedHV has lower measured wall-clock time and exhibits a tighter variance ($\pm 5.69$ seconds) compared to the erratic runtime fluctuations of FSMGDA ($\pm 33.31$ seconds). This confirms that calculating weights from scalar objective slacks mitigates the computational volatility inherent to projecting multi-dimensional task gradients.

\begin{table}[h]
\centering
\caption{\textbf{Empirical Efficiency Analysis (\textsc{CIFAR10--MNIST}).} Metrics are averaged over three independent seeds for 300 adaptive communication rounds. Total values for FedHV strictly incorporate the one-time 20-round uniform-weight calibration phase ($\sim 117$s, $\sim 2.76$ GB) to ensure a fair end-to-end comparison.}
\label{tab:empirical_efficiency}
\begin{tabular}{lcc}
\toprule
\textbf{Method} & \textbf{Total Wall-Clock Time (s)} & \textbf{Downlink / Uplink per Client (MB)} \\
\midrule
FSMGDA & $2824.73 \pm 33.31$ & $69.06$ / $136.50$ \\
FedCMOO & $2521.24 \pm 8.98$ & $69.06$ / $136.50$ \\
\textbf{FedHV} (Ours) & \textbf{\boldmath$1899.25 \pm 5.69$} & \textbf{\boldmath$69.06$ / $69.06$} \\
\bottomrule
\end{tabular}
\end{table}

\begin{figure}[!b]
    \centering
    \includegraphics[width=0.85\textwidth]{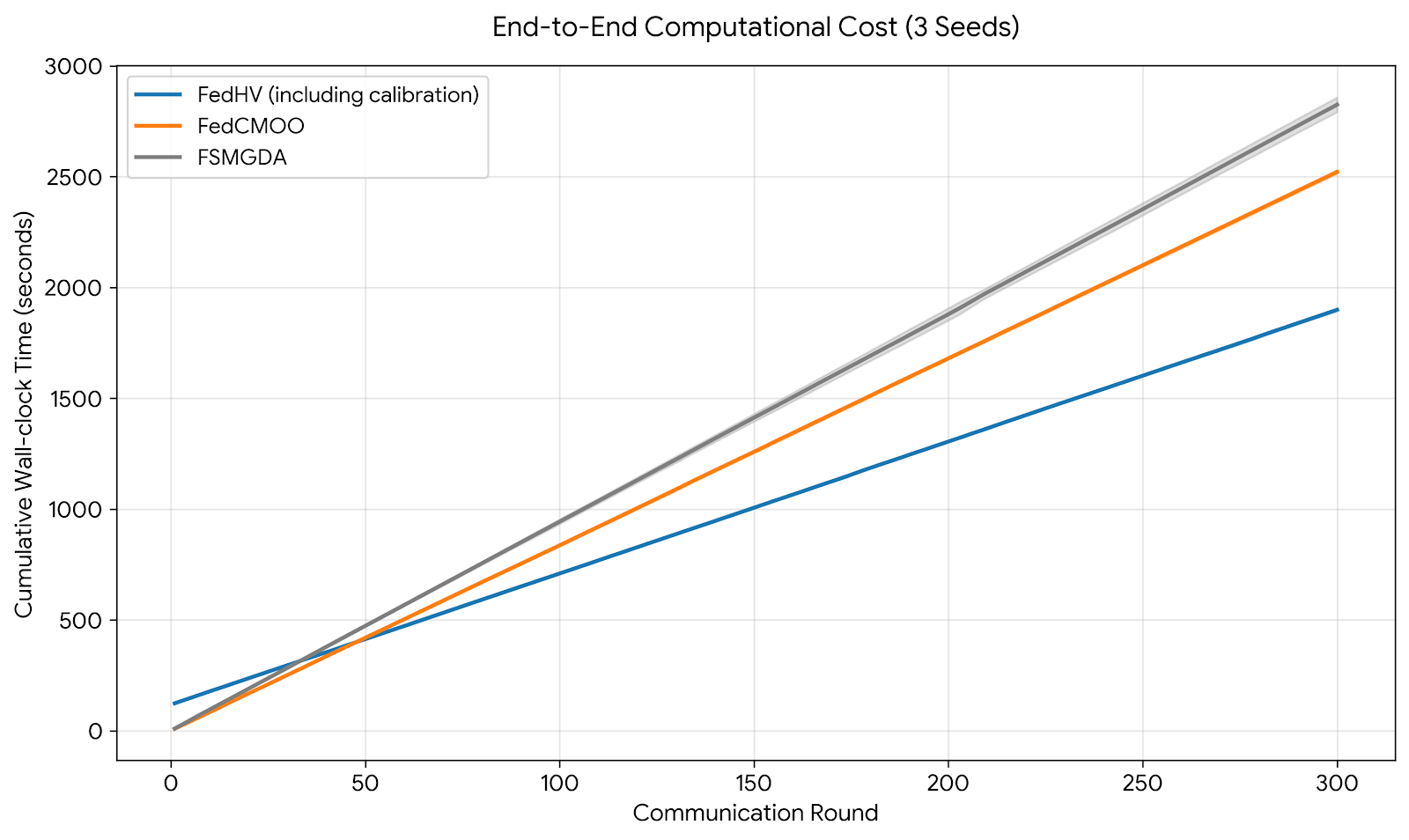}
    \caption{\textbf{End-to-End Computational Cost.} Cumulative wall-clock time over 300 communication rounds on \textsc{CIFAR10--MNIST} ($\alpha=0.30$). The FedHV trajectory strictly includes the runtime of its 20-round uniform-weight calibration phase as an initial offset. FedHV avoids iterative server-side gradient optimization, resulting in a $\sim 24.7\%$ speedup over FedCMOO and $\sim 32.9\%$ speedup over FSMGDA with significantly lower runtime variance.}
    \label{fig:wall_clock_time}
\end{figure}
\section{Complete FedHV Algorithm}
\label{app:fedhv-algorithm}

Algorithm~\ref{alg:fedhv} gives the full round-level pseudocode. The
reference vector is calibrated and frozen before this adaptive phase;
see Appendix~\ref{app:reference-selection}.

\begin{algorithm}[H]
\caption{FedHV training procedure with one-round-stale inverse-slack weights}
\label{alg:fedhv}
\footnotesize
\begin{algorithmic}[1]
\STATE \textbf{Input:} rounds $T$; local steps $K$; client fraction $p$;
report budget $B_{\rm rep}$; stepsizes $\eta_\ell,\eta_g$; fixed
reference $\bm r$; estimated-slack floor $\rho_{\mathrm{clip}}$.
\STATE \textbf{Server initialization:}
$\thetav^0\leftarrow\textsc{InitializeModel}()$ and
$\wv^0\leftarrow\frac{1}{m}\mathbf{1}_m$.
\FOR{$t=0,1,\ldots,T-1$}
  \STATE \textbf{Server selection:} sample $S_t\subseteq[N]$ with
  $|S_t|=pN$ and broadcast $(\thetav^t,\wv^t)$.
  \FORALL{$k\in S_t$ \textbf{ in parallel}}
    \STATE \textbf{Client report:} estimate $\widehat{\bm F}_k^{\,t}$
    from $B_{\rm rep}$ minibatches at the pre-update model $\thetav^t$.
    \STATE \textbf{Client initialization:}
    $\thetav_{k,0}^t\leftarrow\thetav^t$.
    \FOR{$s=0,1,\ldots,K-1$}
      \STATE Draw a fresh minibatch $B_{k,s}^t$
      \STATE $\displaystyle
      \thetav_{k,s+1}^t\leftarrow\thetav_{k,s}^t-
      \eta_\ell\nabla_\theta\!\left[
      \sum_{i=1}^m\sg{w_i^t}\,
      \ell_i(\thetav_{k,s}^t;B_{k,s}^t)\right]$
    \ENDFOR
    \STATE $\Delta\thetav_k^t\leftarrow
    \thetav_{k,K}^t-\thetav^t$
    \STATE \textbf{Client return:}
    $(\Delta\thetav_k^t,\widehat{\bm F}_k^{\,t})$.
  \ENDFOR
  \STATE \textbf{Server model aggregation:}
  $\displaystyle \thetav^{t+1}\leftarrow\thetav^t+
  \eta_g |S_t|^{-1}\sum_{k\in S_t}\Delta\thetav_k^t$
  \STATE \textbf{Server report aggregation:}
  $\displaystyle \widehat{\bm F}^{\,t}\leftarrow
  |S_t|^{-1}\sum_{k\in S_t}\widehat{\bm F}_k^{\,t}$
  \STATE \textbf{Server weight update:}
  $\widehat s_i^t\leftarrow
  \max\{r_i-\widehat f_i^{\,t},\rho_{\mathrm{clip}}\}$ for every
  $i\in[m]$.
  \STATE $\displaystyle w_i^{t+1}\leftarrow
  \frac{(\widehat s_i^t)^{-1}}
  {\sum_{j=1}^m(\widehat s_j^t)^{-1}}$
  for every $i\in[m]$.
\ENDFOR
\STATE \textbf{Output:} global model $\thetav^T$.
\end{algorithmic}
\end{algorithm}

\end{document}